\documentclass[twoside,11pt]{article}

\usepackage{jmlr2e}

\usepackage{amsmath}
\usepackage{subfig,multirow}
\usepackage{bbm} 
\usepackage{enumitem}
\usepackage{algorithm,algpseudocode}

\newcommand{\ie}{{i.e.}}

\newcommand{\bm}{{\boldmath}}
\newcommand{\etal}{{et al. }}
\newcommand{\hide}[1]{}
\newcommand{\inred}[1]{\textcolor{red}{#1}}

\newcommand*{\Scale}[2][4]{\scalebox{#1}{$#2$}}
\newcommand{\norm}[1]{\left\lVert#1\right\rVert}

\newtheorem{problem}{Problem}
\newtheorem{assumption}{Assumption}

\usepackage{hyperref}

\newcommand{\RNum}[1]{\uppercase\expandafter{\romannumeral #1\relax}}

\usepackage{mathtools}
\DeclarePairedDelimiter\ceil{\lceil}{\rceil}
\DeclarePairedDelimiter\floor{\lfloor}{\rfloor}

\jmlrheading{1}{2000}{1-48}{4/00}{10/00}{Kuang, Xiong, Cui, Athey and Li}

\ShortHeadings{Stable Prediction across Unknown Environments}{Kuang, Xiong, Cui, Athey and Li}
\firstpageno{1}

\begin{document}

\title{Stable Prediction across Unknown Environments\thanks{We are grateful for helpful comments from Vitor Hadad.}}

\author{\name Kun Kuang\thanks{Equal Contribution} \email kkun2010@gmail.com \\
       \addr Department of Computer Science\\
       Tsinghua University
       \AND
       \name Ruoxuan Xiong\footnotemark[2] \email rxiong@stanford.edu \\
       \addr Department of Management Science \& Engineering\\
       Stanford University
       \AND
       \name Peng Cui \email cuip@tsinghua.edu.cn \\
       \addr Department of Computer Science\\
       Tsinghua University
       \AND
       \name Susan Athey \email athey@stanford.edu \\
       \addr Graduate School of Business\\
       Stanford University
       \AND
       \name Bo Li \email libo@sem.tsinghua.edu.cn \\
       \addr School of Economics and Management\\
       Tsinghua University}

\hide{
\author{\name Marina Meil\u{a} \email mmp@stat.washington.edu \\
       \addr Department of Statistics\\
       University of Washington\\
       Seattle, WA 98195-4322, USA
       \AND
       \name Michael I.\ Jordan \email jordan@cs.berkeley.edu \\
       \addr Division of Computer Science and Department of Statistics\\
       University of California\\
       Berkeley, CA 94720-1776, USA}
       }

\editor{XXX}

\maketitle

\begin{abstract}
In many machine learning applications, the training distribution used to learn a probabilistic classifier differs from the testing distribution on which the classifier will be used to make predictions.
Traditional methods correct the distribution shift by reweighting the training data with the ratio of the density between test and training data. But in many applications training takes place without prior knowledge of the testing. 
Recently, methods have been proposed to address the shift by learning causal structure, but they rely on the diversity of multiple training data to a good performance, and have complexity limitations in high dimensions.
In this paper, we propose a novel Deep Global Balancing Regression algorithm to jointly optimize a deep auto-encoder model and a global balancing model for stable prediction across unknown environments.
The global balancing model constructs balancing weights that facilitate estimating of partial effects of features, a problem that is challenging in high dimensions, and thus helps to identify stable, causal relationships between features and outcomes. The deep auto-encoder model is designed to reduce the dimensionality of the feature space, thus making global balancing easier. We show, both theoretically and with empirical experiments, that our algorithm can make stable predictions across unknown environments.
\end{abstract}


\section{Introduction}

Predicting unknown outcome values based on their observed features using a model estimated on a training data set is a common statistical problem. Many machine learning and data mining methods have been proposed and shown to be successful when the test data and training data come from the same distribution.
However, the best-performing models for a given distribution of training data typically exploit subtle statistical relationships among features, making them potentially more prone to prediction error when applied to test data sets where, for example, the joint distribution of features differs from that in the training data. 
Therefore, it can be useful to develop predictive algorithms that are robust to shifts in the environment, particularly in application areas where models can not be retrained as quickly as the environment changes, \ie, online prediction.

\hide{
 \todo{Perhaps give some specific cases where that might be, from cases where model training is computationally expensive, where it is expensive to gather new outcome data, or cold start problem in a new geography, or online environment like newspaper where a viral story can bring in a different distribution of customers}
}

Recently, many methods (\cite{shimodaira2000improving,bickel2009discriminative,sugiyama2008direct,huang2007correcting,dudik2006correcting,liu2014robust}) have been proposed to address this problem.
The main idea of these methods is to reweight training data with a density ratio, so that its distribution can become more closely aligned with the distribution of test data. The methods have achieved good performance for correcting for distribution shift, but they require prior knowledge of the test distribution when estimating the density ratio. 

For the case of unknown test data, some researchers have proposed learning methods where training takes place across multiple training datasets.
By exploring the invariance across multiple datasets, Peters \etal (\cite{peters2016causal}) proposed an algorithm to identify causal features, and Rojas-Carulla \etal (\cite{rojas2015causal}) proposed a causal transform framework to learn invariant structure.
Similarly, domain generalization methods (\cite{muandet2013domain}) try to learn an invariant representation of data.
The performance of these methods relies on the diversity of their multiple training data, and they cannot address distribution shifts which do not appear in their training data.
Moreover, most of these methods are highly complex, with training complexity growing exponentially with the dimension of the feature space in the worst case, which is not acceptable in high dimensional settings.

In this paper, we focus on an environment where the expected value of the outcome conditional on all covariates is stable across enrivonments.  Further, covariates fall into one of two categories: for the first category, the conditional expectation has a non-zero dependence on the covariates; we call these ``causal'' variables, although in some applications they might better be described as variables that have a structural relationship with the outcome.  For example, ears, noses, and whiskers are structural features of cats that are stable across different environments where images of animals may be taken.  A second category of variable are termed ``noisy variables,'' which are variables that are correlated with either the causal variables, the outcome, or both, but do not themselves have a causal effect on the outcome; conditional on the full set of causal variables, they do not affect expected outcomes. Further, we consider a setting where the analyst may not know a prior which variables fall into each category.  Finally, we assume that there are no unobserved confounders, so that it is possible to estimate the causal effect of each causal variable with a very large dataset when all covariates are adequately controlled for.  We focus on settings when there are many features and perhaps limited data.

One way to improve the stability of prediction algorithms in such a setting is to isolate the impact of each individual feature.  If the expectation of the outcome conditional on covariates is stable across environments, and variability in the joint distribution of features is the source of instability, then the stable prediction problem can be solved by estimating the conditional expectation function accurately.  
With a small number of discrete features and a large enough dataset, simple estimation methods such as ordinary least squares can accomplish this goal.  If there is a larger number of features but only a few matter for the conditional expectation (that is, the true outcome model is sparse), regularized regression can be applied to consistently estimate the conditional expectation function.  
However, with a larger set of causal features relative to the number of observations, regularized regression will no longer consistently estimate partial effects.  For example, LASSO will omit many variables from the regression, while the coefficients on included variables depend on the covariance of the outcome with the omitted variables as well as on the covariance between the omitted and included variables.  This results in instability: if the covariance among features differs across environments, then prediction based on such a model will be unstable across environments. 
In such high-dimensional cases, alternative approaches are required.  

Here, we use an approach motivated by the literature on causal inference, where variable balancing strategies are used for estimating the average effect of changing a single binary covariate (the treatment). Causal inference methods optimize a different objective than prediction-based methods; they prioritize consistent estimation of treatment effects over prediction in a given training data set. The methods are designed for a scenario where the analyst has domain knowledge about which variable has a causal effect, so that the focus of the analysis is on estimating the effect of the treatment in the presence of other features which are known to be confounders (variables that affect both treatment assignment and potential outcomes). Indeed, only after controlling for confounders can the difference in the expectation of the outcome between treatment and control groups be interpreted as a treatment effect.  One approach to estimating treatment effects in the presence of confounders is to use variable balancing methods, which
attempt to construct weights that balance the distribution of covariates between a treatment and a control group. 
They either employ propensity scores (\cite{rosenbaum1983central,lunceford2004stratification,austin2011introduction,kuang2017treatment,kuang2016steering}), or optimize balancing weights directly (\cite{hainmueller2012entropy,zubizarreta2015stable,athey2016approximate,kuang2017estimating}).
These methods provide an efficient approach to estimate causal effects with a small number of treatment variables in observational studies, but most of them can not handle well settings where there may be many causal variables and the analyst does not know which ones are causal; as such, existing covariate balancing methods do not immediately extend to the general stable prediction problem.  

\hide{
, and they apply to the case where the analyst has prior knowledge of which variables have a causal effect and which are potential confounders
}
Inspired by balancing methods from the causal inference literature, we propose a Deep Global Balancing Regression (DGBR) algorithm for stable prediction. The framework is illustrated in Figure~\ref{fig:framework}, which consists of three (jointly optimized) sub-models: (i) a deep auto-encoder to reduce the dimensionality of the features, (ii) construction of balancing weights that enable the effect of each covariate to be isolated, and (iii) estimation of a predictive model using the encoded features and balancing weights.  As this algorithm explicitly prioritizes covariate balancing (at the expense of a singular focus on predictive accuracy in a given training dataset), it is able to achieve greater stability than a purely predictive model.
Using both empirical experiments and theoretical analysis, we establish that our algorithm achieves stability in prediction across unknown environments.
The experimental results on both synthetic and real world datasets demonstrate that our algorithm outperforms all the baselines for the stable prediction problem.

In summary, the contributions of this paper are listed as follows:
\begin{itemize}
\item We investigate the problem of stable prediction across unknown environments, where the distribution of agnostic test data might be very different with the training data.
\item We propose a novel DGBR algorithm to jointly optimize deep auto-encoder for dimension reduction and global balancing for estimation of causal effects, and simultaneously address the stable prediction problem.
\item We give theoretical analysis on our proposed algorithm and prove that our algorithm can make a stable prediction across unknown environments by global balancing.
\item The advantages of our DGBR algorithm are demonstrated on both synthetic and real world datasets.
\end{itemize}

The rest of the paper is organized as follows.
Section 2 reviews the related work.
In Section 3, we give problem formulation and introduce our DGBR algorithm.
Section 4 gives the optimization and discussion on our algorithm.
Section 5 gives the theoretical analysis on our algorithm.
Section 6 gives the experimental results.
Finally, Section 7 concludes.

\section{Related Work}
In this section, we investigate the previous related work, including literatures on covariate shift, variable balancing, and invariant learning.

The covariate shift literature (\cite{shimodaira2000improving}) focuses on settings where the data distribution for training is different than the data distribution for testing.
To correct for the differences, (\cite{shimodaira2000improving}) introduced the idea of reweighting samples in training data by the ratio of the density in the testing data to the density in the training data. 
A variety of techniques have been proposed to estimate the density ratio, including discriminative estimation (\cite{bickel2009discriminative}), Kullaback-Leibler importance estimation (\cite{sugiyama2008direct}), kernel mean matching (\cite{huang2007correcting} \cite{yu2012analysis}), maximum entropy methods (\cite{dudik2006correcting}), minimax optimization (\cite{wen2014robust}), and robust bias-aware approach (\cite{liu2014robust}).
These methods achieved good performance for correcting for covariate shifts, but most of them require prior knowledge of testing distribution to estimate the density ratio.
In contrast, we focus on the stable prediction across unknown environments in this paper.
\hide{
Unfortunately, it is infeasible in many real applications, for example prediction on streaming test data or agnostic test data.
In this paper, we propose a deep global balancing algorithm for stable prediction across agnostic test data.
}

Adjusting for confounders is a key challenge for estimating causal effects in observational studies.
To precisely estimate causal effects in the presence of many confounders, covariate balancing methods have been proposed (\cite{kuang2017estimating,kuang2017effective,kuang2017treatment,athey2016approximate,zubizarreta2015stable,hainmueller2012entropy,rosenbaum1983central}).
In a seminal paper, Rosenbaum and Rubin (\cite{rosenbaum1983central}) proposed to achieve variable balancing by reweighting observations 
by the inverse of propensity score.
Kuang et al. (\cite{kuang2017treatment}) proposed a data-driven variable decomposition method for variable balancing.
Li et al. (\cite{li2017matching}) bal- anced the variables by matching on their nonlinear representation.
Hainmueller (\cite{hainmueller2012entropy}) introduced entropy balancing method for variable balancing across a range of statistical tasks.
Athey \etal (\cite{athey2016approximate}) proposed approximate residual balancing algorithm, which, motivated by doubly robust approaches, combines outcome modeling using the LASSO with balancing weights constructed to approximately balance covariates between treatment and control groups.
Kuang \etal (\cite{kuang2017estimating}) proposed a differentiated variable balancing algorithm by jointly optimizing sample weights and variable weights.
These methods provide an effective way to estimate causal effects in observational studies, but they are limited to estimate causal effect of one variable, and are not designed for the case with many causal variables; further, the methods assume that the analyst has prior knowledge of which covariates have a causal effect and which do not.
\hide{
In this paper, by adapting the existing variable balancing technique, we propose a deep global balancing algorithm to explore causal relationship between all observed variables and response variable, and make a stable prediction across environments.
}

Recently, some methods have been proposed to make prediction on agnostic test data using the method of invariant learning.
Peters \etal (\cite{peters2016causal}) proposed an algorithm to identify causal predictors by exploring the invariance of the conditional distribution of the outcome across multiple training datasets.
Rojas-Carulla \etal (\cite{rojas2015causal}) proposed a causal transfer framework to identify invariant predictors across multiple datasets and then use them for prediction.
Similarly, domain generalization (\cite{muandet2013domain}) methods estimate an invariant representation of data by minimizing the dissimilarity across training domains.
Invariant learning methods can be used to estimate a model that will in principle perform well for an unknown test dataset, but the performance of these methods relies on the diversity of their multiple training data, and they cannot address the distribution shift which does not appear in their training data.
\hide{
Otherwise, most of these methods are highly complex with training complexity as $2^{p}$ in the worst case.
}

\hide{
In this paper, by adapting variable balancing technique, we propose a deep global balancing algorithm for stable prediction across environments, and our algorithm needs only one dataset for training.  
}
\hide{
However, similar to the invariant learning literature, we rely on multiple datasets constructed by the analyst in order to tune the model.  In our method, the analyst artificially creates differences across training environments.\todo{Kun: check you agree with this description  Also the paper doesn't give many details about how to do this or why it matters.}
}
\section{Problem and Our Algorithm}

In this section, we first give problem formulation, then introduce the details of our deep global balancing regression algorithm. Finally, we give theoretical analysis about our proposed algorithm.

\subsection{Problem Formulation}

Let $\mathcal{X}$ denote the space of observed features and $\mathcal{Y}$ denote the outcome space. For simplicity, we consider the case where the features have finite support, which without loss
of generality can be represented as a set of binary features: $\mathcal{X}=\{0,1\}^p$. We also focus on the case where the outcome space is binary: $\mathcal{Y}=\{0,1\}$. 
We define an \textbf{environment} to be a joint distribution $P_{XY}$ on $\mathcal{X} \times \mathcal{Y}$, and let $\mathcal{E}$ denote the set of all environments.
In each environment $e\in \mathcal{E}$, we have dataset $D^{e} = (\mathbf{X}^{e}, Y^{e})$, where $\mathbf{X}^e \in \mathcal{X}$ are predictor variables and $Y^e \in \mathcal{Y}$ is a response variable.  The joint distribution of features and outcomes on $(\mathbf{X},Y)$ can vary across environments: $P^{e}_{XY} \neq P^{e'}_{XY}$ for $e,e'\in\mathcal{E}$, and $e\neq e'$.

\begin{table}[tbp]
\centering
\caption{Symbols and definitions.}
\label{tab:symbols}
\begin{tabular}{|c|l|}
\hline
Symbols & Definitions \\
\hline
\hline
$n$ & Sample size \\
\hline
$p$ & Dimension of features \\
\hline
$\mathbf{X}=\{\mathbf{S},\mathbf{V}\}\in\{0,1\}^p$ & Features\\
\hline
$\mathbf{S}\in\{0,1\}^{p_s}$ & Stable features\\
\hline
$\mathbf{V}\in\{0,1\}^{p_v}$ & Noisy features\\
\hline
$Y\in\{0,1\}$ & Outcome \\
\hline
$W\in {\mathbb{R}^+}^{n\times 1}$ & Global sample weights\\
\hline
$\phi(\cdot)$ & Embedding function \\
\hline
\end{tabular}
\end{table}

In this paper, our goal is to learn a predictive model, which can make a stable prediction across unknown environments. Before giving problem formulation, we first define $Average\_Error$ and $Stability\_Error$ across environments of a predictive model as:
\begin{eqnarray}
\label{metrics:acc} \Scale[0.9]{Average\_Error} \!\!\!\!\!\! &=& \!\!\!\!\!\! \Scale[0.9]{\frac{1}{|\mathcal{E}|}\sum_{e \in \mathcal{E}}Error(D^e)},\\
\label{metrics:stb}  \Scale[0.9]{Stability\_Error} \!\!\!\!\!\! &=& \!\!\!\!\!\! \Scale[0.8]{\sqrt{\frac{1}{|\mathcal{E}|-1}\sum_{e \in \mathcal{E}}\left(Error(D^e)-Average\_Error\right)^{2}}},
\end{eqnarray}
where $|\mathcal{E}|$ refers to the number of environments, and $Error(D^e)$ represents the predictive error on dataset $D^e$ from environment $e$. 

\hide{
It should be clear that the definition of Stability depends on a prior notion of the set of environments that might be relevant as well as how those environments are weighted.  The algorithm we propose is less sensitive to this choice--we use it in tuning model parameters but not in the core algorithm--and we leave the study of how sensitive our tuning is to the specification of the set of environments to future work.
}

In this paper, we define Stability (\cite{yu2013stability}) by $Stability\_Error$. The smaller $Stability\_Error$, the better a model is ranked in terms of Stability.
Then, we define the stable prediction problem as follow:

\begin{problem}[Stable Prediction] 
\textbf{Given} one training environment $e\in \mathcal{E}$ with dataset $D^{e}=(\mathbf{X}^{e},Y^{e})$, the task is to \textbf{learn} a predictive model to predict across unknown environment $\mathcal{E}$ with not only small $Average\_Error$ but also small $Stability\_Error$. 
\end{problem}

Suppose $\mathbf{X} = \{\mathbf{S} ,\mathbf{V} \}$. We define $\mathbf{S}$ as \emph{stable features}, and refer to the other features $\mathbf{V} = \mathbf{X} \backslash \mathbf{S}$ as \emph{noisy features}, where the following assumption gives their defining properties:

\begin{assumption}
\label{asmp:stable}
There exists a probability mass function $P(y|s)$ such that for all environments $e\in\mathcal{E}$, $Pr(Y^e=y|\mathbf{S}^e=s,\mathbf{V}^e=v)=Pr(Y^e=y|\mathbf{S}^e=s)=P(y|s)$.
\end{assumption}

\hide{
Denote $E$ as the environment variable. Under assumption \ref{asmp:stable}, the conditional distribution of $Y$ given $\mathbf{S}=s$ in an environment $E=e$, $P(Y|\mathbf{S}=s, E=e)$, is the same $\forall e$. However, the conditional distribution $P(Y|\mathbf{V}=v, E=e)$ may vary in different environments.
}

\hide{identifying stable features and capturing structure between the stable features and the outcome.
Thus, there are two key challenges for addressing the stable prediction problem.
The first challenge is how to precisely identify the stable features $\mathbf{S}$ from whole observed variables $\mathbf{X}$, especially in high dimensional settings.
The second challenge is how to capture the underlying structure between stable features $\mathbf{S}$ and response variable $Y$, which could be highly non-linear in real applications.  Our algorithm does not separate those goals into distinct steps, but accomplishes both goals (approximately) in a single algorithm.}

With Assumption \ref{asmp:stable}, we can address the stable prediction problem by building a model that learns the stable function $P(y|s)$. To understand the content of Assumption \ref{asmp:stable}, without loss of generality we can write a generative model for the outcome unit $i$ in environment $e$ with stable features $s$, where $h(\cdot)$ is a known function to account for discreteness of $Y$:
\[Y^e_i(s)=h(g(s)+\epsilon^e_{s,i}), \text{ and } Y_i^e = Y^e_i(\mathbf{S}_i) = h(g(\mathbf{S}_i) + \epsilon^e_{{\mathbf{S}_i},i}).\]
$Y^e_i(s)$ is the outcome that would occur for unit $i$ in environment $e$ if the input is equal to $s$.  If we allow $\epsilon^e_{s,i}$ to be correlated with the unit's features $\mathbf{X}_i$ in arbitrary ways, Assumption \ref{asmp:stable} may fail, for example if $\mathbf{V}_i^e$ is positively correlated with $\epsilon^e_{s,i}$ then units with higher values of $\mathbf{V}^e_i$ would have higher than average values of $Y_i^e$, so that $\mathbf{V}_i^e$ would be a useful predictor in a given environment, but that relationship might vary across environments, leading to instability.  If we first impose the condition that for each $s$, $\epsilon^e_{s,i}$ is independent of $\mathbf{V}^e_i$ conditional on $\mathbf{S}^e_i$, then given the model specification, $\mathbf{V}^e_i$ is no longer needed as a predictor for outcomes conditional on $\mathbf{S}^e_i$.  If we
second impose the condition that for each $s$, $\epsilon^e_{s,i}$ is independent of $\mathbf{S}^e_i$ conditional on $\mathbf{V}^e_i$, then instability in the distribution of $\epsilon^e_{s,i}$ across environments will not affect $Pr(Y^e=y|\mathbf{S}^e=s,\mathbf{V}^e=v)$.
Maintaining the first condition, the second condition is sufficient not only for Assumption \ref{asmp:stable} but also to enable consistent estimation of $g(\cdot)$ using techniques from the causal inference literature in a setting with sufficient sample size and when the analyst has prior knowledge of the set of stable features; we propose a method that will estimate $g$ without prior knowledge of which features are stable.   
We also observe that a stronger but simpler condition can replace the second condition to guarantee Assumption \ref{asmp:stable}, namely that the distribution of $\epsilon^e_{s,i}$ does not vary with $\{e,s\}$.
Fig. \ref{fig:graph} illustrates three relationships between predictor variables $\mathbf{X^e} = \{\mathbf{S}^e, \mathbf{V}^e\}$ and response variable $Y^e$ consistent with the conditions, including $\mathbf{S}\perp \mathbf{V}$, $\mathbf{S}\rightarrow \mathbf{V}$, and $\mathbf{V}\rightarrow \mathbf{S}$.

\begin{figure}[tb]
\centering
\subfloat[$\mathbf{S}\perp \mathbf{V}$ \label{fig:s0v}]{
  \includegraphics[width=1.1in]{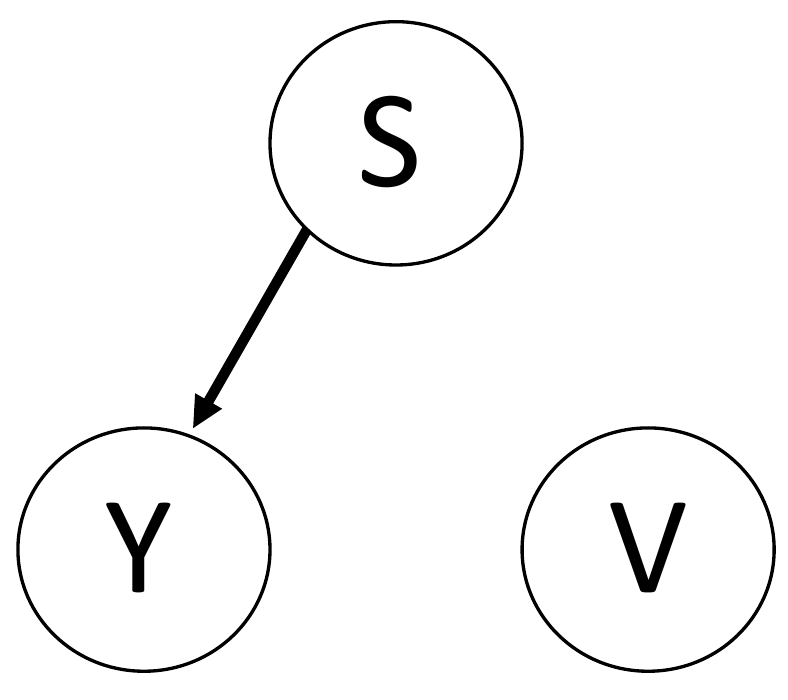}
}
\subfloat[$\mathbf{S}\rightarrow \mathbf{V}$ \label{fig:s2v}]{
  \includegraphics[width=1.1in]{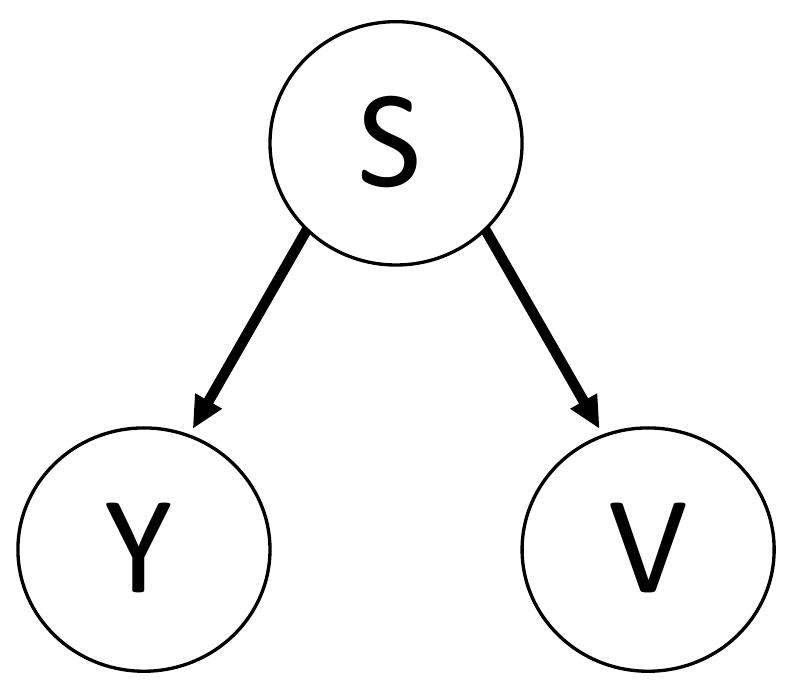}
}
\subfloat[$\mathbf{V}\rightarrow \mathbf{S}$ \label{fig:v2s}]{
  \includegraphics[width=1.1in]{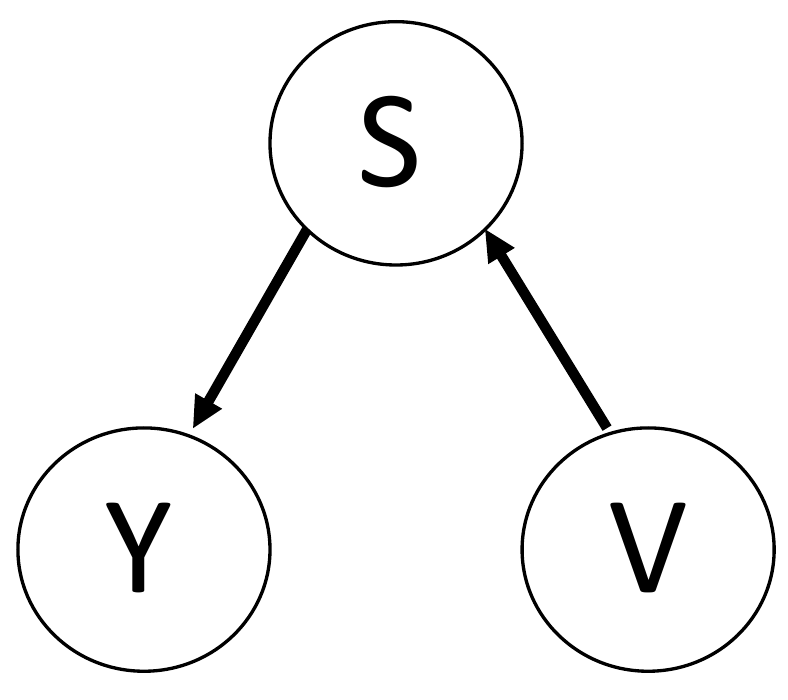}
}
\caption{Three diagrams for stable features $\mathbf{S}$, noisy features $\mathbf{V}$, and response variable $Y$.}
\label{fig:graph}
\end{figure}

\hide{
In practice, the relationship between $\mathbf{X}$ and $Y$ we observed might be very different with the unseen ground truth $Y=f(\mathbf{S})$ because of the selection bias induced by data sampling or data collection, which could generate spurious correlation between $\mathbf{X}$ and $Y$. 
Therefore, to precisely identify stable features $\mathbf{S}$ for stable prediction, one has to remove the spurious correlation between $\mathbf{X}$ and $Y$ from observed training dataset.
}

\subsection{The Model}

\subsubsection{Framework}

\begin{figure}[t]
\centering
\includegraphics[width=4.0in]{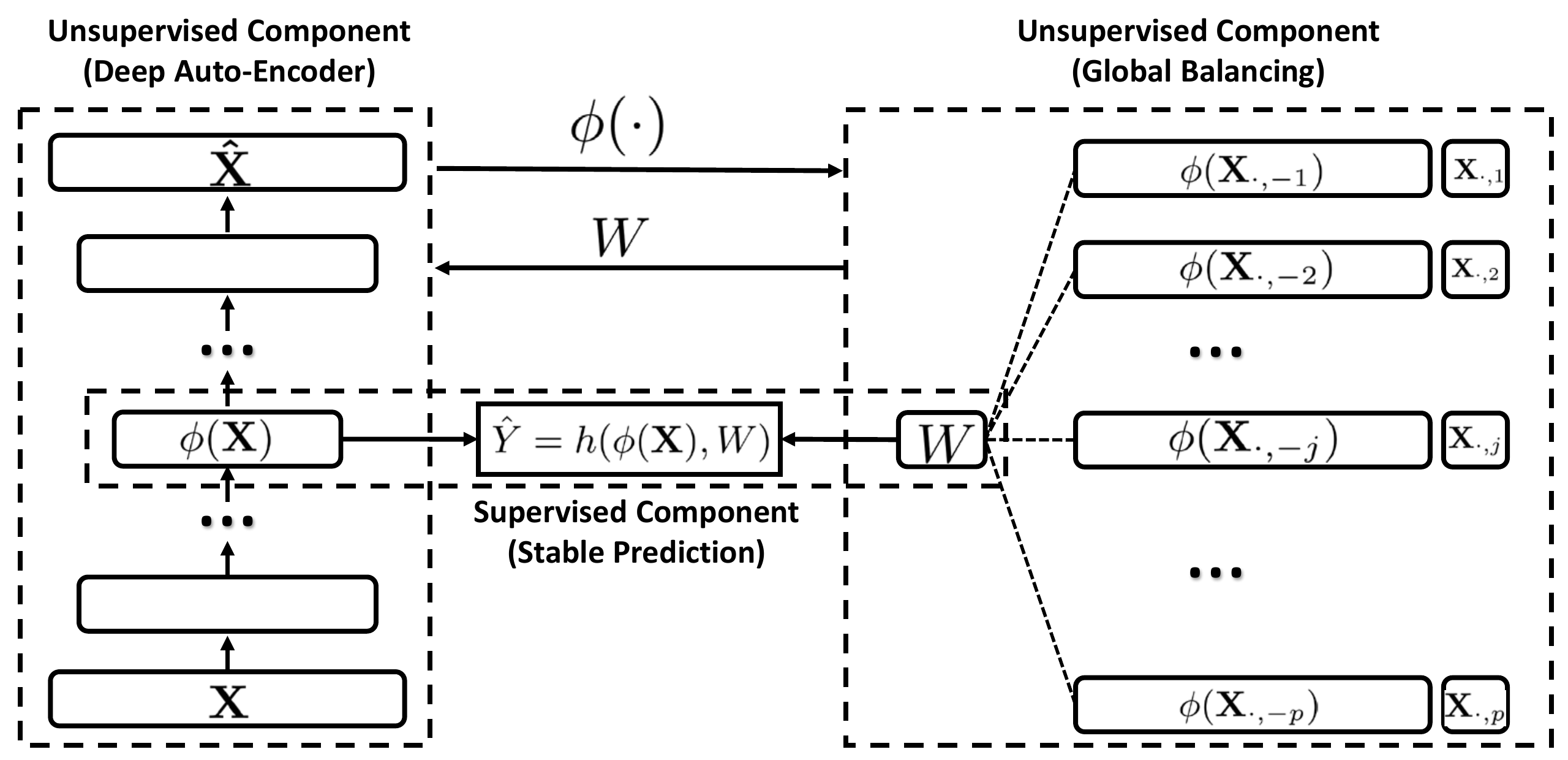}
\caption{The framework of our proposed DGBR model.}
\label{fig:framework}
\end{figure}

We propose a Deep Global Balancing Regression (DGBR) algorithm to identify stable features and capture non-linear structure for stable prediction.
Its framework is shown in Figure~\ref{fig:framework}.
To identify the stable features, we propose a global balancing model, where we learn global sample weights which can be used to estimate the effect of each feature while controlling for the other features and thus identify stable features.
To capture the non-linear structure between stable features and response variable, we employ a deep auto-encoder model, which is composed of multiple non-linear mapping functions to map the input data to a non-linear and low dimensional space. Balancing in a low dimensional space simplifies the problem of global balancing, since for each covariate $j$, the weights balance the constructed covariates from the dimension reduction $\phi(\mathbf{X}_{\cdot,-j})$ across realizations of $\mathbf{X}_{\cdot,j}$. Finally, weighting observations with the global sample weights, we learn a predictive model for outcomes as a function of the low-dimensional representation of covariates using regularized regression.  All three components of the model are jointly optimized in the algorithm.

\subsubsection{Global Balancing Regression Algorithm}

In this section, we develop the construction of global balancing weights. To be self-contained, we briefly revisit the key idea of variable balancing technique. Variable balancing techniques are often used for causal effect estimation in observational studies, where the distributions of covariates are different between treated and control groups because of non-random treatment assignment, but treatment assignment is independent of potential outcomes conditional on covariates.
To consistently estimate causal effects in such a setting, one has to balance the distribution of covariates between treatment and control.
Most variables balancing approaches exploit moments to characterize distributions, and balance them between treated and control groups by adjusting sample weights $W$ as following:
\begin{equation}
    \label{eq:variable_balancing}
    W = \arg \min_{W}\Scale[1.0]{\left\|\frac{\sum_{i:T_i=1}W_i\cdot \mathbf{X}_i}{\sum_{i:T_i=1}W_i} - \frac{\sum_{i:T_i=0}W_i\cdot \mathbf{X}_i}{\sum_{i:T_i=0}W_i}\right\|^{2}_2}.
\end{equation}
Given a treatment variable $T$, the $\frac{\sum_{i:T_i=1}W_i\cdot \mathbf{X}_i}{\sum_{i:T_i=1}W_i}$ and $\frac{\sum_{i:T_i=0}W_i\cdot \mathbf{X}_i}{\sum_{i:T_i=0}W_i}$ represent the first-order moments of variables $\mathbf{X}$ on treated ($T=1$) and control ($T=0$) groups, respectively.
By sample reweighting with $W$ learnt from Eq.~(\ref{eq:variable_balancing}), one can estimate the causal effect of treatment variable on response variable by comparing the average difference of $Y$ between treated and control groups.  In high-dimensional problems, approximate balancing can be used for consistent estimation under some additional assumptions (\cite{athey2016approximate}), where to control variance of estimates the sum of squared weights is also penalized in the minimization.

In low dimensions, the same approach could be employed to estimate $Pr(Y=y|\mathbf{X}=x)$ for different values of $x$.  However, when $p$ is large, there may not be sufficient data to do so, and so approximate balancing techniques generalized to the case where $\mathbf{X}$ is a vector of indicator variables may perform well in practice, and also help identify stable features from the larger vector $\mathbf{X}$.
We propose a global balancing regularizer, where we successively regard each variable as treatment variable and balance all of them together via learning global sample weights by minimizing:
\begin{equation}
\label{eq:L_balancing}
\Scale[1.0]{\sum_{j=1}^{p}\left\|\frac{\mathbf{X}_{\cdot,-j}^T\cdot (W\odot \mathbf{X}_{\cdot,j})}{W^T\cdot \mathbf{X}_{\cdot,j}}-\frac{\mathbf{X}_{\cdot,-j}^T\cdot (W\odot (1-\mathbf{X}_{\cdot,j}))}{W^T\cdot (1-\mathbf{X}_{\cdot,j})}\right\|_2^2},
\end{equation}
where $W$ is global sample weights, $\mathbf{X}_{\cdot,j}$ is the $j^{th}$ variable in $\mathbf{X}$, and $\mathbf{X}_{\cdot,-j} = \mathbf{X} \backslash \{\mathbf{X}_{\cdot,j}\}$ means all the remaining variables by removing the $j^{th}$ variable in $\mathbf{X}$ \footnote{We obtain $\mathbf{X}_{\cdot,-j}$ in experiment by setting the value of $j^{th}$ variable in $\mathbf{X}$ as $zero$.}. 
The summand represents the loss from covariate imbalance when setting variable $\mathbf{X}_{\cdot,j}$ as the treatment variable, and $\odot$ refers to Hadamard product.
Note that only first-order moment is considered in Eq.~(\ref{eq:L_balancing}), but higher order moments can be easily incorporated by including interaction features of $\mathbf{X}$.

By sample reweighting with $W$ learnt from Eq.~(\ref{eq:L_balancing}), we can identify stable features  $\mathbf{S}$ by checking if there is any correlation between $Y$ and $\mathbf{X}$ covariate by covariate, because, as we show below, only stable features are correlated with $Y$ after sample reweighting by $W$.

With the global balancing regularizer in Eq.~(\ref{eq:L_balancing}), we propose a Global Balancing Regression (GBR) algorithm to jointly optimize global sample weights $W$ and regression coefficients $\beta$ for stable prediction based on traditional logistical regression as:
\begin{eqnarray}
\label{eq:global_balancing_algorithm}
\!\!\!&\min& \Scale[1.0]{\sum_{i=1}^{n}W_i\cdot \log(1+\exp((1-2Y_i)\cdot (\mathbf{X}_i\beta)))},\\
\nonumber \!\!\!&s.t.& \Scale[0.85]{\sum_{j=1}^{p}\left\|\frac{\mathbf{X}_{\cdot,-j}^T\cdot (W\odot \mathbf{X}_{\cdot,j})}{W^T\cdot \mathbf{X}_{\cdot,j}}-\frac{\mathbf{X}_{\cdot,-j}^T\cdot (W\odot (1-\mathbf{X}_{\cdot,j}))}{W^T\cdot (1-\mathbf{X}_{\cdot,j})}\right\|_2^2 \leq \lambda_1,\ \ \  W\succeq 0},\\
\nonumber \!\!\!  &\quad& \Scale[0.85]{\|W\|_2^2 \leq \lambda_2, \ \  \|\beta\|_2^2 \leq \lambda_3, \ \  \|\beta\|_1 \leq \lambda_4}, \ \ \Scale[0.85]{(\sum_{k=1}^{n}W_k-1)^{2} \leq \lambda_5}
\end{eqnarray}
where $\mathbf{X}_i$ is the $i^{th}$ row / sample in $\mathbf{X}$, and $\sum_{i=1}^{n}W_i\cdot \log(1+\exp((1-2Y_i)\cdot (\mathbf{X}_i\beta)))$ is the weighted loss of logistic regression and the loss is defined as the minus log likelihood.  
The terms $W\succeq 0$ constrain each of sample weights to be non-negative.
With norm $\|W\|_2^2\leq \lambda_2$, we can reduce the variance of the sample weights.
Elastic net constraints $\|\beta\|_2^2\leq \lambda_3$ and $\|\beta\|_1\leq \lambda_4$ help to avoid overfitting.
The formula $(\sum_{k=1}^{n}W_k-1)^{2} \leq \lambda_5$ avoids all the sample weights to be $zero$.

\subsubsection{Deep Global Balancing Regression Algorithm}

The proposed GBR algorithm in Eq.~(\ref{eq:global_balancing_algorithm}) can help to identify stable features and make a stable prediction, but with many features relative to observations, it may be difficult to estimate the effects of all the features as well as their interactions, and it might also be challenging for GBR to learn global sample weights.

To address these challenges, we propose a Deep Global Balancing Regression (DGBR) algorithm by jointly optimizing Deep auto-encoder and Global Balancing Regression. Following standard approaches (\cite{bengio2007greedy}), the deep auto-encoder consists of multiple non-linear mapping functions to map the input data to a low dimensional space while capturing the underlying features interactions.
Deep auto-encoder is an unsupervised model which is composed of two parts, the encoder and decoder.
The encoder maps the input data to low-dimensional representations, while the decoder reconstructs the original input space from the representations.
Given the input $\mathbf{X}_i$, the hidden representations for each layer are shown as follows:
\begin{eqnarray}
\nonumber \phi(\mathbf{X}_i)^{(1)} &=& \sigma(\mathbf{A}^{(1)}\mathbf{X}_i+b^{(1)})\\
\nonumber \phi(\mathbf{X}_i)^{(k)} &=& \sigma(\mathbf{A}^{(k)}\phi(\mathbf{X}_i)^{(k-1)}+b^{(k)}), k = 2,\cdots,K
\end{eqnarray}
where $K$ is the number of layer. $\mathbf{A}^{(k)}$ and $b^{(k)}$ are weight matrix and bias on $k^{th}$ layer. $\sigma(\cdot)$ represents non-linear activation function.\footnote{We use sigmoid function $\sigma(x) = \frac{1}{1+\exp(-x)}$ as non-linear activation function.}

After obtaining the representation $\phi(\mathbf{X}_i)^{(K)}$, we can obtain the reconstruction $\hat{\mathbf{X}}_i$ by reversing the calculation process of encoder with parameters $\hat{\mathbf{A}}^{(k)}$ and $\hat{b}^{(k)}$.
The goal of deep auto-encoder is to minimize the reconstruction error between the input $\mathbf{X}_i$ and the reconstruction $\hat{\mathbf{X}}_i$ with the following loss function.
\begin{equation}
\label{eq:L_auto}
\Scale[1.0]{\mathcal{L} = \sum_{i=1}^{n}\|(\mathbf{X}_i-\hat{\mathbf{X}}_i)\|_2^2.}
\end{equation}

By combining the loss functions of deep auto-encoder in Eq.~(\ref{eq:L_auto}) and GBR algorithm in Eq.~(\ref{eq:global_balancing_algorithm}), we give the objective function of our Deep Global Balancing Regression algorithm as:
\begin{eqnarray}
\label{eq:deep_global_balancing_algorithm}
&\min& \Scale[1.0]{\sum_{i=1}^{n}W_i\cdot \log(1+\exp((1-2Y_i)\cdot (\phi(\mathbf{X}_i)\beta)))},\\
\nonumber &s.t.& \Scale[0.9]{\sum_{j=1}^{p}\left\|\frac{\phi(\mathbf{X}_{\cdot,-j})^T\cdot (W\odot \mathbf{X}_{\cdot,j})}{W^T\cdot \mathbf{X}_{\cdot,j}}-\frac{\phi(\mathbf{X}_{\cdot,-j})^T\cdot (W\odot (1-\mathbf{X}_{\cdot,j}))}{W^T\cdot (1-\mathbf{X}_{\cdot,j})}\right\|_2^2 \leq \lambda_1},\\
\nonumber &\quad& \|(W\cdot \bm{1})\odot(X-\hat{X})\|_F^2\leq \lambda_2, \ \ W\succeq 0, \ \ \|W\|_2^2 \leq \lambda_3,\\
\nonumber   &\quad& \|\beta\|_2^2 \leq \lambda_4, \ \  \|\beta\|_1 \leq \lambda_5,\ \ \Scale[1.0]{(\sum_{k=1}^{n}W_k-1)^{2} \leq \lambda_6}\\
\nonumber  &\quad& \Scale[1.0]{\sum_{k=1}^{K}(\|A^{(k)}\|_F^2+\|\hat{A}^{(k)}\|_F^2)} \leq \lambda_7,
\end{eqnarray}
where $\phi(\cdot) = \phi(\cdot)^{(K)}$ for brevity.
$\|(W\cdot \mathbf{1})\odot(X-\hat{\mathbf{X}})\|_F^2$ represents the reconstruction error between input $\mathbf{X}$ and reconstruction $\hat{\mathbf{X}}$ with global sample weights $W$.
The term $\Scale[1.0]{\sum_{k=1}^{K}(\|\mathbf{A}^{(k)}\|_F^2+\|\hat{\mathbf{A}}^{(k)}\|_F^2)} \leq \lambda_7$ regularizes the coefficients of the deep auto-encoder model.

\section{Theoretical Analysis}

In this section, we give theoretical analysis about our algorithm. We prove it can make a stable prediction across unknown environments with sufficient data, and analyze the upper bound about our proposed algorithm. 

\subsection{Analysis on Stable Prediction}
A key requirement for the method to work is the overlap assumption, which is a common assumption in the literature of treatment effect estimation \cite{athey2016approximate}. We suppress the notation for the enviornment $e$ in the first part of this section.

\hide{
That is, even if we do not know which variables are $\mathbf{S}$ and which are $\mathbf{V}$ in $\mathbf{X}$, $\mathbf{S}$ and $\mathbf{V}$ are independent after we apply the global balancing method to $\mathbf{X}$, and then estimated $g(\cdot)$ are unbiased across unknown environments and we can make stable prediction for $Y$ using balanced $\mathbf{X}$.
}
\begin{assumption}[Overlap]
\label{asmp:overlap}
For any variable $\mathbf{X}_{\cdot,j}$ when setting it as the treatment variable, it has
$\forall j, 0 < P(\mathbf{X}_{\cdot,j} = 1 | \mathbf{X}_{\cdot,-j}) < 1$.
\end{assumption}

Then, we have following Lemma and Theorem:
\begin{lemma} \label{pro:population_overlap}
If $\forall j, 0 < P(\mathbf{X}_{\cdot,j} = 1 | \mathbf{X}_{\cdot,-j}) < 1$, and $\mathbf{X}$ are binary, then $\forall i, 0 < P(\mathbf{X}_i=x) < 1$, where $\mathbf{X}_i$ is $i^{th}$ row in $X$.
\end{lemma}

\begin{proof}
See Appendix \ref{Appendix:A}.
\end{proof}

\begin{theorem} \label{thm1}
Let $X \in \mathbb{R}^{n \times p}$. Under Lemma \ref{pro:population_overlap},  if number of covariates $p$ is finite, then $\exists W$ such that
\begin{eqnarray} \label{thm1lim1}
\lim_{n \rightarrow \infty} \Scale[0.9]{\sum_{j=1}^p \left\|\frac{ \mathbf{X}_{-j}^T (W \odot \mathbf{X}_{\cdot, j} ) }{ W^T \mathbf{X}_{\cdot, j}}  - \frac{ \mathbf{X}_{-j}^T (W \odot (1 - \mathbf{X}_{\cdot, j}) ) }{ W^T (1 - \mathbf{X}_{\cdot, j})}\right\|_2^2 = 0}
\end{eqnarray}
with probability 1. In particular, a $W$ that satisfies (\ref{thm1lim1}) is $\Scale[1.0]{W_i^* = \frac{1}{P(\mathbf{X}_{i} = x)}}$.
\end{theorem}

\begin{proof}
Since $\norm{\cdot} \geq 0$, Eq. (\ref{thm1lim1}) can be simplified to $\forall j$, $\forall k \neq j$
\hide{\begin{eqnarray} \label{tmp2}
\lim_{n \rightarrow \infty} \Scale[0.9]{\left(  \frac{ \mathbf{X}_{\cdot,k}^T (W \odot \mathbf{X}_{\cdot, j} ) }{ W^T \mathbf{X}_{\cdot, j}}  - \frac{ \mathbf{X}_{\cdot, k}^T (W \odot (1 - \mathbf{X}_{\cdot, j}) ) }{ W^T (1 - \mathbf{X}_{\cdot, j})} \right)  = 0}
\end{eqnarray}
with probability 1. Eq. (\ref{tmp2}) can be rewritten as }
\begin{eqnarray*}
\lim_{n \rightarrow \infty} \Scale[0.9]{\left(  \frac{\sum_{i: \mathbf{X}_{i,k}=1, \mathbf{X}_{i,j}=1} W_i }{\sum_{i: \mathbf{X}_{i,j}=1} W_i }  - \frac{\sum_{i: \mathbf{X}_{i,k}=1, \mathbf{X}_{i,j}=0} W_i }{\sum_{i: \mathbf{X}_{i,j}=0} W_i } \right)  = 0}
\end{eqnarray*}
with probability 1. For $W^*$, from Lemma \ref{pro:population_overlap}, $0 < P(\mathbf{X}_i = x) < 1$, $\forall x$, $\forall i$, $t = 1$ or $0$,
\begin{eqnarray*}
\lim_{n \rightarrow \infty} \Scale[0.9]{\frac{1}{n} \sum_{i: \mathbf{X}_{i,j}=t} W^*_i}\!\!\! &=&\!\!\!\lim_{n \rightarrow \infty}  \Scale[0.9]{\frac{1}{n} \sum_{x: x_j = t} \sum_{i: \mathbf{X}_{i}=x} W^*_i} \\
\!\!\!&=&\!\!\! \lim_{n \rightarrow \infty} \Scale[0.9]{ \sum_{x: x_j = t} \frac{1}{n} \sum_{i: \mathbf{X}_{i}=x} \frac{1}{P(\mathbf{X}_i= x)}}\\
\!\!\!&=&\!\!\! \lim_{n \rightarrow \infty}  \Scale[0.9]{\sum_{x: x_j = t} P(\mathbf{X}_i= x) \cdot \frac{1}{P(\mathbf{X}_i= x)} = 2^{p-1}}
\end{eqnarray*}
with probability 1 from Law of Large Number. Since features are binary,
\begin{eqnarray*}
\!\!\!\!\!&& \lim_{n \rightarrow \infty} \Scale[1.0]{\frac{1}{n} \sum_{i: \mathbf{X}_{i,k}=1, \mathbf{X}_{i,j}=1} W^*_i = 2^{p-2} }\\
\!\!\!\!\!&& \lim_{n \rightarrow \infty} \Scale[1.0]{\frac{1}{n} \sum_{i: \mathbf{X}_{i,j}=0} W^*_i = 2^{p-1}},\\
\!\!\!\!\!&& \lim_{n \rightarrow \infty} \Scale[1.0]{\frac{1}{n} \sum_{i: \mathbf{X}_{i,k}=1, \mathbf{X}_{i,j}=0} W^*_i = 2^{p-2}}
\end{eqnarray*} 
and therefore, we have following equation with probability 1:
\begin{eqnarray*}
\!\!\! \lim_{n \rightarrow \infty} \Scale[0.9]{\left(  \frac{ \mathbf{X}_{\cdot,k}^T (W^* \odot \mathbf{X}_{\cdot, j} ) }{ W^{*T} \mathbf{X}_{\cdot, j}}  - \frac{ \mathbf{X}_{\cdot, k}^T (W^* \odot (1 - \mathbf{X}_{\cdot, j}) ) }{ W^{*T} (1 - \mathbf{X}_{\cdot, j})} \right)}  = \Scale[1.0]{\frac{2^{p-2}}{2^{p-1}} -  \frac{2^{p-2}}{2^{p-1}} = 0.}
\end{eqnarray*}
\end{proof}

The following result shows that if there is sufficient data such that all realizations of $x$ appear in the data, exact balancing weights can be derived. Subsequently, we show that in this case, the components of $\mathbf{X}$ are mutually independent in the reweighted data.  This highlights that overlap is a strong assumption.  In real-world data sets, when $p$ is large the cardinality of $\mathcal{X}$ is large, and so exactly balancing weights are not available, but the results still highlight that balancing weights will reduce the covariance among features.

Then, based on Lemma~\ref{pro:population_overlap} and Theorem~\ref{thm1lim1}, we have following propositions for stable prediction.
\begin{proposition}\label{prop:svindep}
If $0 < \hat{P}(\mathbf{X}_i = x) < 1$ for all $x$, where $\hat{P}(\mathbf{X}_i = x) = \frac{1}{n} \sum_{i} \mathbbm{1}(\mathbf{X}_i = x)$, there exists a solution  $W^*$ satisfies equation (\ref{eq:L_balancing}) equals 0 and variables in $\mathbf{X}$ are independent after balancing by $W^*$.
\end{proposition}

\begin{proof}
See Appendix \ref{Appendix:B}.
\end{proof}

\begin{proposition}\label{prop:yvindep}
If $0 < \hat{P}(\mathbf{X}^e_i = x) < 1$ for all $x$ in environment $e$, $Y^{e'}$ and $\mathbf{V}^{e'}$ are independent when the joint probability mass function of $(\mathbf{X}^{e'},Y^{e'})$ is given by reweighting the distribution from environment $e$ using weights $W^*$, so that $p^{e'}(x,y)=p^e(y|x) \cdot (1/|\mathcal{X}|)$.
\end{proposition}

\begin{proof}
It is immediate that $Pr(Y^{e'}=y|\mathbf{X}^{e'}=x)=Pr(Y^e=y|\mathbf{X}^e=x)$.  Putting this together with Assumption \ref{asmp:stable}, $Pr(Y^{e'}=y|\mathbf{X}^{e'}=x)=Pr(Y^{e'}=y|\mathbf{S}^{e'}=s)$. From Proposition \ref{prop:svindep}, $(\mathbf{S}^{e'},\mathbf{V}^{e'})$ are mutually independent. Thus, we have
\begin{align}
&Pr(Y^{e'}=y|\mathbf{V}^{e'}=v) \nonumber\\
=&E_{\mathbf{S}^{e'}}[Pr(Y^{e'}=y|\mathbf{S}^{e'},\mathbf{V}^{e'}=v)|\mathbf{V}^{e'}=v] \nonumber \\
=&E_{\mathbf{S}^{e'}}[Pr(Y^{e'}=y|\mathbf{S}^{e'})|\mathbf{V}^{e'}=v] \nonumber \\
=&Pr(Y^{e'}=y). \label{eqn:y|s,v-cond} \nonumber 
\end{align}
Thus, $Y^{e'}$ and $\mathbf{V}^{e'}$ are independent.
\end{proof}

Propositions \ref{prop:svindep} and \ref{prop:yvindep} suggest that the GBR algorithm can make a stable prediction across environments that satisfy Assumption ~\ref{asmp:stable}, since after reweighting, only the stable features are correlated with outcomes, and $p(y|s)$ is unchanged in the reweighted dataset. The objective function of GBR algorithm is to equivalent to maximize log-likelihood of logistic regression, which is known to be consistent. Even though the regularization constraints will cause some bias to the estimated  $p(y|s)$, but the bias reduces with sample size $n$. Thus, with sufficient data, the GBR algorithm should learn $p(y|s)$.

Now consider the properties of the DGBR algorithm: 
\begin{enumerate}[wide = 0pt]
\item \textit{Preserves the above properties of the GBR algorithm while making the overlap property easier to satisfy and reducing the variance of balancing weights.} 
The Johnson-Lindenstrauss (JL) lemma~(\cite{johnson1984extensions}) implies that for any $0<\epsilon<1/2$ and $x_1,\cdots,x_n \in \mathbb{R}^p$, there exists a mapping $f:\mathbb{R}^p \rightarrow \mathbb{R}^k$, with $k = O(\epsilon^{-2}\log n)$, such that
\noindent $\forall i,j \ \  (1-\epsilon)\|x_i-x_j\|^2\leq \|f(x_i)-f(x_j)\|^2 \leq (1+\epsilon) \|x_i-x_j\|^2$,
we can transform high-dimensional data into a lower suitable dimensional space while approximately preserving the original distances between points. Our DGBR algorithm reduces the feature dimension, so that the population overlap assumption is more likely to be satisfied and we are less likely to see extreme values of balancing weights, so that better balance can be attained while maintaining low variance of the weights. 
\item  \textit{Enables more accurate estimation of $p(y|s)$}, because with multiple non-linear mapping functions in our DGBR algorithm, it can more easily capture the underlying non-linear relationship between stable features and response variables even with many stable features.
\end{enumerate}

\subsection{Analysis on Upper Bound}
\subsubsection{Notation}
Theorem \ref{thm1} states that if $W_i^*=1/P(\mathbf{X}_i = x)$, the global balancing regularizer in Eq. (\ref{eq:L_balancing}) converges to 0 as number of observations $n$ goes to infinity. Nevertheless, in finite samples, Eq. (\ref{eq:L_balancing}) may not be equal to 0 for any $W$. We define the maximum covariate imbalance as 
\begin{eqnarray}
\!\!\!\!\! \alpha  \!\!\!\!\! &=& \!\!\!\!\! \Scale[0.9]{\max_j \norm{\frac{\sum_{i: X_{i,k}=1, X_{i,j}=1} \hat{W}_i }{\sum_{i: X_{i,j}=1} \hat{W}_i }  - \frac{\sum_{i: X_{i,k}=1, X_{i,j}=0} \hat{W}_i }{\sum_{i: X_{i,j}=0} \hat{W}_i }}_{\infty}}, \\
\!\!\!\!\!  \!\!\!\!\! &=& \!\!\!\!\!  \Scale[0.9]{ \max_j \max_{k \neq j}\left| \frac{\sum_{i: X_{i,k}=1, X_{i,j}=1} \hat{W}_i }{\sum_{i: X_{i,j}=1} \hat{W}_i }  - \frac{\sum_{i: X_{i,k}=1, X_{i,j}=0} \hat{W}_i }{\sum_{i: X_{i,j}=0} \hat{W}_i } \right|}.
\end{eqnarray}

We define $m$ as the number of values in $\mathcal{X}$ that do not appear in $\mathbf{X}$, 
$$
m  = |\mathcal{X}| - |\mathcal{X}_\mathbf{X}| = 2^p - |\mathcal{X}_\mathbf{X}|, 
$$
where $\mathcal{X}_\mathbf{X} = \{x| \mathbf{X}_i = x \text{, for some } i \} $, $|\mathcal{X}| $ and $ |\mathcal{X}_\mathbf{X}|$ are the cardinalities of $\mathcal{X}$ and $\mathcal{X}_\mathbf{X}$. $m$ is random, where the variation of $m$ comes from sampling a finite sample $\mathbf{X}$  from the population distribution $P_X$ on $\mathcal{X}$.

We define $\mathbb{E}[\alpha]$ as the expectation of $\alpha$ over the random sample $\mathbf{X}$.

\subsubsection{Upper Bound of Global Balancing Regularizer}

The maximum covariate imbalance $\alpha$ is different under different random finite sample $\mathbf{X}$. The following Lemma states that $\alpha$  is determined by $\mathbf{X}$ through $m$, the number of values in $\mathcal{X}$ that do not appear in $\mathbf{X}$, as well as the number of covariates $p$.

\begin{lemma}  \label{lemma:alpha}
Given $\mathbf{X} \in \mathbb{R}^{n \times p}$, $p \geq 2$, if $m$ different values in $\mathcal{X}$ do not appear in $\mathbf{X}$, then we have
\begin{enumerate}
\item if $m = 0$, then $\alpha = 0$
\item if $0 < m \leq 2^{p-2}$, then $\alpha = \frac{2^{p-2}}{2^{p-1} - m} - \frac{1}{2}$
\item if $2^{p-2} < m < 2^{p-1}$, then $\alpha = 1 - \frac{2^{p-1} - m}{3 \times 2^{p-2} - m}$
\item if $2^{p-1} \leq m \leq 2^{p}-2$ \footnote{Given $p$, $0 \leq m \leq 2^{p} - 1$. When $m = 2^{p} - 1$, either $\frac{\sum_{i: X_{i,k}=1, X_{i,j}=1} \hat{W}_i }{\sum_{i: X_{i,j}=1} \hat{W}_i } $ or $\frac{\sum_{i: X_{i,k}=1, X_{i,j}=0} \hat{W}_i }{\sum_{i: X_{i,j}=0} \hat{W}_i }$ is $\frac{0}{0}$, which is undefined. For simplicity and completeness, we define $\frac{0}{0} = 1$, and therefore $\alpha = 1$.}, then $\alpha = 1$
\end{enumerate}

\end{lemma}

\begin{proof}
See Appendix \ref{Appendix:C}.
\end{proof}

$m$ measures how severe the overlap assumption is violated in the empirical distribution of $\mathbf{X}$. $\alpha$ increases with $m$ for a fixed $p$. If $m$ is fixed, when $0 < m \leq 2^{p-2}$,  $\alpha$ is decreasing in $p$; when $2^{p-2} < m < 2^{p-1}$,  $\alpha$ is increasing in $p$; \footnote{This scenario can never happen, because given $m$ and $p$, if $2^{p-2} < m < 2^{p-1}$, when $p$ increases, it will have $m < 2^{p-2}$. } when $2^{p-1} \leq m \leq 2^{p}-1$, $\alpha$ does not depend on $p$.  A more realistic scenario is that the number of observations $n$ is fixed, if $p$ increases, the overlap assumption is harder to be satisfied in the empirical distribution, so $m$ will also increase, which might result in an increase of $\alpha$. Although $n$ is not explicitly expressed in $\alpha$, $n$ affects $m$ in that $m$ converges to 0 as $n$ goes to infinity, and therefore $\alpha$ may decrease with $n$ through $m$. 

Theorem \ref{thm:mean_alpha} extends Lemma \ref{lemma:alpha} and states $\mathbb{E}[\alpha]$, the expected value of $\alpha$ over the random sample $\mathbf{X}$. We show that  $\mathbb{E}[\alpha]$ also equals to the expected value of $\alpha$ over $m$ that is determined by $\mathbf{X}$.
\begin{theorem}\label{thm:mean_alpha}
Let $\alpha = \Scale[0.85]{\max_j \norm{\frac{\sum_{i: X_{i,k}=1, X_{i,j}=1} \hat{W}_i }{\sum_{i: X_{i,j}=1} \hat{W}_i }  - \frac{\sum_{i: X_{i,k}=1, X_{i,j}=0} \hat{W}_i }{\sum_{i: X_{i,j}=0} \hat{W}_i }}_{\infty}}
$, we have 
\begin{eqnarray}
\Scale[1.0]{\mathbb{E}\left[ \alpha \right] = \frac{1}{{{n + 2^p - 1} \choose {2^p - 1}}}  \sum_{m=0}^{2^{p} - 1} {2^p \choose m} {{n-1} \choose {2^p-1-m}} g(p,m) }
\end{eqnarray}
where $g(p,m)$ is 
\begin{eqnarray}\label{eqn:def_g_p_m}
g(p,m)= \begin{cases}
       \frac{2^{p-2}}{2^{p-1} - m} - \frac{1}{2}, & \text{if}\  0 \leq m \leq 2^{p-2} \\
      1 - \frac{2^{p-1} - m}{3 \times 2^{p-2} - m} & \text{if } 2^{p-2} < m < 2^{p-1} \\
      1 & \text{if } 2^{p-1} \leq m \leq 2^{p}-1
    \end{cases}
\end{eqnarray}
\end{theorem}

\begin{proof}
See Appendix \ref{Appendix:D}.
\end{proof}

$\mathbb{E}[\alpha]$ depends on $n$ and $p$. In particular, $\mathbb{E}[\alpha]$ is decreasing in $n$, but increasing in $p$, see Fig.  \ref{fig:alpha_n_p}. 

\begin{figure}
\centering
\includegraphics[width=2.4in]{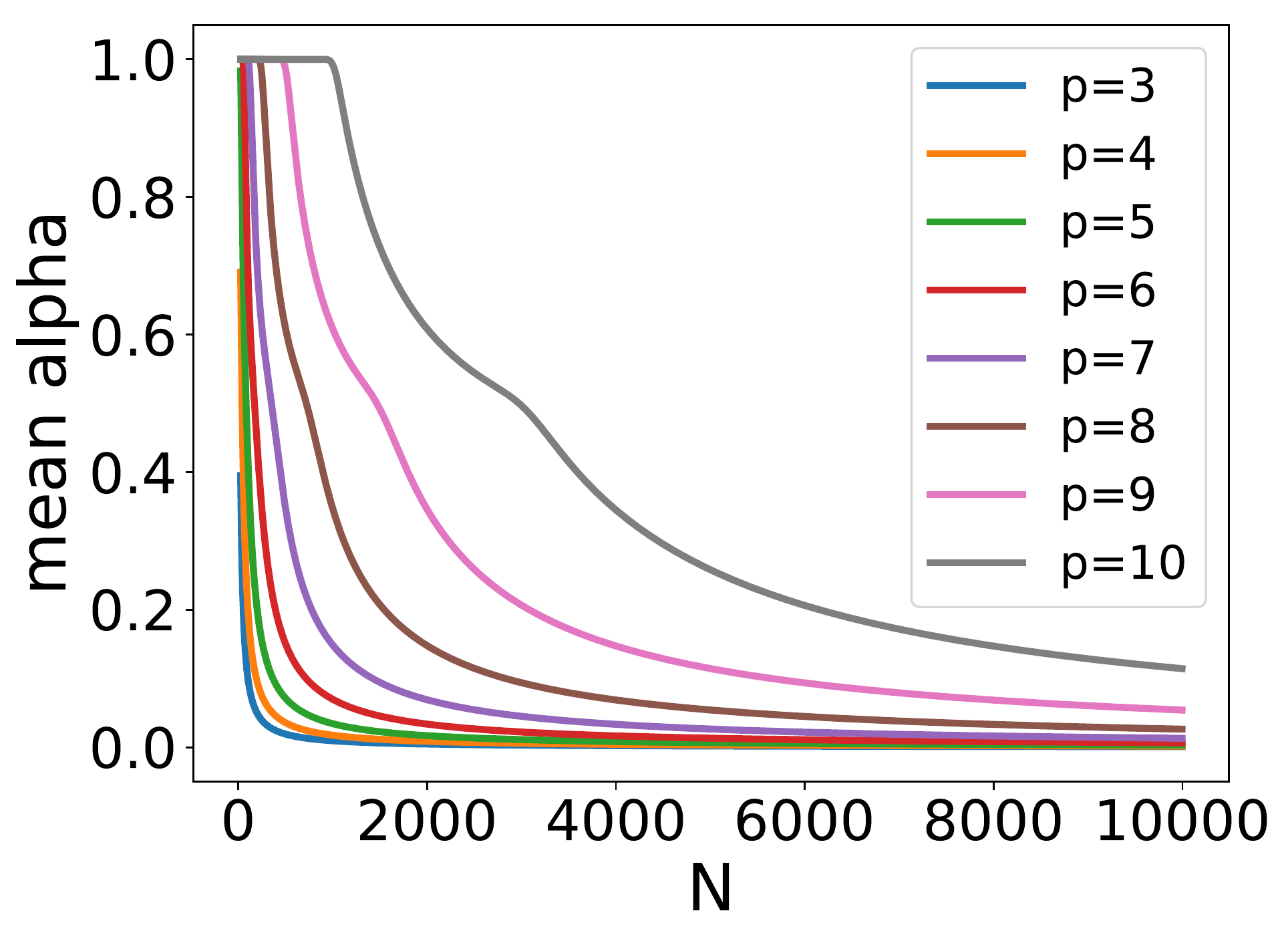}
\caption{$\mathbb{E}[\alpha]$ with different $N$ and $p$}
\label{fig:alpha_n_p}
\end{figure}

\hide{\inred{We also give the upper bound between empirical risk of approximate balancing and expected risk of exact balancing in Appendix.}}

\subsubsection{Upper Bound of the Risk in Approximate Balancing}

It is possible that our DGBR algorithm may not perfectly balance all covariates in $\mathbf{X}$. That is, the minimum value $zero$ may not be attained by any $W$ in Eq. (\ref{eq:L_balancing}). Denote $W^*$ as an  \textbf{approximate} solution to the global balancing problem if $W^*$ is the optimal solution to our DGBR algorithm and Eq. (\ref{eq:L_balancing}) evaluated at $W^*$ is not $zero$. Define the \textbf{imbalance} at $x$ in $\mathbf{X}$ balanced by $W^*$ as the difference between the joint weighted probability and the product of weighted marginal probabilities, that is, 
\begin{eqnarray}\label{eqn:epsilon_x-def}
\Scale[1.0]{\epsilon_x = \tilde{p}_{x} - \prod_j^p \tilde{p}_{x_j} = \tilde{p}_{x} - p_x}
\end{eqnarray}
where $\tilde{p}_x = \frac{1}{\tilde{n}} \sum_{i=1}^n W_i^* \mathbbm{1}(\mathbf{X}_i = x)$, $\tilde{n} = \sum_{i=1}^n W_i^*$,  $\tilde{p}_{x_j}  = \frac{1}{\tilde{n}}\sum_{i=1}^n W_i^* \mathbbm{1}(\mathbf{X}_{ij} = x_j)$ and $p_x = \prod_j^p \tilde{p}_{x_j}$. If Eq. (\ref{eq:L_balancing}) evaluated at $W^*$ is not $zero$, $\epsilon_x$ must be nonzero at some $x$ and $\epsilon_x$ measures how far covariates in weighted $X$ are from independence.

In this subsection, we would like to evaluate our DGBR algorithm if optimal $W^*$ in our DGBR algorithm is an approximate solution. That is, we would like to upper bound the expected risk of the optimal $\hat f$ learned from our DGBR algorithm, where $f(\cdot)$ is a function to predict the response variable $Y_i$ from covariates $\mathbf{X}_i$. The expected risk between $f(\mathbf{X})$ and $Y$ is defined as $L_{P}(f) = E_{P}(l(f(\mathbf{X}_i), Y_i))$. In $L_{P}(f)$, the loss function $l(f(\mathbf{X}_i), Y_i) = \log(1+\exp((1-2Y_i)\cdot (\phi(\mathbf{X}_i)\beta)))$ is the same as the objective function in our DGBR algorithm; the probability mass function $P(\mathbf{X}_i, Y_i) = P(\mathbf{X}_i) P(Y_i | \mathbf{X}_i)$ has $P(Y_i = y|\mathbf{X}_i = x) = P(Y_i = y|\mathbf{S}_i=s, \mathbf{V}_i=v) =P(y|s)$ to be the same as that in Assumption \ref{asmp:stable} and  $P(\mathbf{X}_i = x) = p_x$, where $p_x$ is defined in Eq. (\ref{eqn:epsilon_x-def}). In Empirical Risk Minimization (ERM), the population probability mass function is assumed to be fixed (but unknown). However, $p_x$ is not fixed because it is determined by $W^*$ and $W^*$ is learned from our DGBR algorithm. To avoid the identification problem, we set $W$ as fixed after we sequentially update $\beta$, $W$ and $\theta$ for some steps; we denote this $W$ as $W^*$ and only update $\beta$ and $\theta$ afterwards. This $W^*$ is used to calculate $p_x$. We analyze the empirical risk of the following fixed weight DGBR (FWDGBR) algorithm
\begin{eqnarray*}
&\min& \Scale[0.9]{\frac{1}{\sum_{i=1}^n W^*_i } \sum_{i=1}^{n}W^*_i\cdot \log(1+\exp((1-2Y_i)\cdot (\phi(\mathbf{X}_i)\beta)))},\\
&s.t.& \Scale[0.9]{\|(W^*\cdot \bm{1})\odot(X-\hat{X})\|_F^2\leq \lambda_2},\\
&\quad&  \|\beta\|_2^2 \leq \lambda_4,  \|\beta\|_1 \leq \lambda_5,\\
&\quad& \Scale[0.9]{\sum_{k=1}^{K}(\|A^{(k)}\|_F^2+\|\hat{A}^{(k)}\|_F^2) \leq \lambda_7, \|b^{(k)}\|_2 \leq M^{(k)},} \\
&\quad& \text{ for } k = 1, 2, \cdots, K.
\end{eqnarray*}
The empirical risk is defined as $\hat{L}(f) = \frac{1}{\tilde{n}} \sum_{i=1}^n W^*_i l(f(\mathbf{X}_i), Y_i)$ and $\tilde{n} = \sum_{i=1}^n W^*_i$. In the FWDGBR algorithm, it does not have the constraints related to $W$ and has an additional bias constraint $\|b^{(k)}\|_2 \leq M^{(k)}$. The additional bias constraint guarantees that $f(\cdot)$ is bounded for any $f(\cdot)$ that satisfies the constraints.  All the other constraints in the FWDGBR algorithm are the same as those in the DGBR algorithm.

In ERM, if empirical average converges to the expectation (equivalent to exact balancing), the difference between the expectation of each function and the empirical average of the function ($|L_{P}(f) - \hat L(f)|$) can be bounded in terms of the Rademacher complexity of the model class and an error term depending on the confidence parameter and sample size. Note that the Rademacher complexity of the FWDGBR algorithm is determined by its constraints. In other words, the Rademacher complexity decreases with $\lambda_2$, $\lambda_4$, $\lambda_5$ and $\lambda_7$.  Furthermore, if empirical average converges to the expectation, $|L_{P}(\hat f) - \hat L(f^*)|$ can also be bounded in terms of the Rademacher complexity of the model class and an error term depending on the confidence parameter and sample size, where $\hat f$ minimizes $\hat L(\cdot)$, $f^*$ minimizes $L_{P}(\cdot)$ and both $\hat f$ and $f^*$ are in the model class. Note that $f^*$ also needs to satisfy the constraints in the FWDGBR algorithm.

However, in approximate balancing, the empirical average $\hat L(f)$ for some $f$ does not converge to expectation $L_{P}(f)$ because the joint distribution does not equal to the product of the marginals. In this case, $|L_{P}(\hat f) - \hat L(f^*)|$ can still be bounded, but with an extra term measuring covariates' imbalance, that is, $\epsilon_x$.

\begin{theorem}
\label{thm:upper_bound_risk}
Let $B_k = \sqrt{\lambda_7 + (M^{(k)})^2}$, $l_k$ be the size of $[\phi(\mathbf{X}_i)^{(k)}; 1]$ for $k=1, 2, \cdots, K$ and $l_0 = p+1$. With probability at least $\geq 1-\delta$, 
\begin{eqnarray}
\nonumber L_{P}(\hat{f})   \!\!\!\!\!\!&\leq&\!\!\!\!\!\! \Scale[0.9]{L_{P}(f^*) + 2^{K+3} \sqrt{\frac{2log(2p)}{n}} \min( \sqrt{\lambda_4 l_K},\lambda_5)\prod_{k=1}^K B_k (l_{k-1})^{1/2} }\\
&&\Scale[0.9]{+ 3 \sqrt{\frac{log(2/\delta)}{2n}} + 2 \max_{x, f} \mathbb{E}[l(f(x), y)|x] \sum_{x} |\epsilon_x|}, \label{ineq:low-bound} 
\end{eqnarray}
where $\hat{f} = \arg \min_{f} \hat L(f)$ and $f^* = \arg \min_{f} L_{P}(f)$.
\end{theorem}

\hide{
\begin{theorem}
\label{thm:upper_bound_risk}
If $f \in \mathcal{F}$, where the model class $ \mathcal{F}$ is finite and $l(f(x), y) \in [a,b]$ \footnote{truncate the cross entropy loss to make the loss function $l(f(x), y)$ bounded.}, with probability $\geq 1-\delta$, 
\begin{eqnarray}
\nonumber L_{P}(\hat{f})   \!\!\!\!\!\!&\leq&\!\!\!\!\!\! \Scale[1.0]{L_{P}(f^*) + 2(b-a) \sqrt{\frac{log2|\mathcal{F}|+log(1/\delta)}{2n}} }\\
&&+ 2 \max_{x, f} E[l(f(x), y)|x] \sum_{x} |\epsilon_x|, \label{ineq:lowerbound}
\end{eqnarray}
where $\hat{f} = \arg \min_{f} \hat{R}(f)$ and $f^* = \arg \min_{f} R_p(f)$.
\end{theorem}
}
\begin{proof}
See Appendix \ref{Appendix:E}.
\end{proof}

\hide{Note that $\hat{f}$ is the optimal solution in empirical risk minimization (ERM) while the empirical distribution is approximately balanced.  Assuming all parameters in $f(\mathbf{X})$ have norm constriants, Theorem \ref{thm:upper_bound_risk} provides an upper bound for $L_{P}(\hat{f})$.}

Theorem \ref{thm:upper_bound_risk} states that in approximate balancing, the upper bound for $L_{P}(\hat f)$ consists of four components: 1. The minimum expected risk $L_{P}(f^*)$; 2. The Rademacher complexity, $2^{K+3} \sqrt{\frac{2log(2p)}{n}} \min( \sqrt{\lambda_4 l_K},\lambda_5)\prod_{k=1}^K B_k (l_{k-1})^{1/2}$, which depends on number of layers, number of hidden units and the constraints in the FWDGBR algorithm; 3. An error term depending on the confidence parameter $\delta$ and sample size $n$, that is, $3 \sqrt{\frac{log(2/\delta)}{2n}}$;   4. The risk from covariates' imbalance in approximate balancing, $2 \max_{x, f}  \mathbb{E}[l(f(x), y)|x] \sum_{x} |\epsilon_x|$.  Note that $ \max_{x, f} E[l(f(x), y)|x]$ is bounded because $x$ are $y$  are binary and all weights in $f(\cdot)$ are bounded.

The term related to Rademacher complexity in Theorem \ref{thm:upper_bound_risk}, $RC = 2^{K+3} \sqrt{\frac{2log(2p)}{n}}\cdot \min( \sqrt{\lambda_4 l_K},\lambda_5)\prod_{k=1}^K B_k (l_{k-1})^{1/2}$, is an upper bound of the size of the model class in the FWDGBR algorithm. This is derived from the upper bound of the size of the model class in the reduced fixed weight DGBR (RFWDGBR) algorithm
\begin{eqnarray*}
&\min& \Scale[0.9]{\frac{1}{\sum_{i=1}^n W^*_i } \sum_{i=1}^{n}W^*_i\cdot \log(1+\exp((1-2Y_i)\cdot (\phi(\mathbf{X}_i)\beta)))},\\
&s.t.& \|\beta\|_2^2 \leq \lambda_4,  \|\beta\|_1 \leq \lambda_5\\
&\quad& \Scale[1.0]{\|A^{(k)}\|_F^2} \leq \lambda_7, \text{ for } k = 1, 2, \cdots, K.
\end{eqnarray*}
This algorithm does not have the auto-encoder and -decoder constraint $\|(W^*\cdot \bm{1})\odot(X-\hat{X})\|_F^2\leq \lambda_2$. The weight constraints of the auto-encoder and -decoder are relaxed. The model class described by the RFWDGBR algorithm is larger than the model class of the fixed weight DGBR algorithm. Thus, the upper bound of the Rademacher complexity of the RFWDGBR algorithm is also an upper bound of the Rademacher complexity of the FWDGBR algorithm.

The derivation of the term RC is closely related to deriving the Rademacher complexity with dropouts in Neural Networks in (\cite{wan2013regularization}) and (\cite{zhai2018adaptive}). Compare with (\cite{wan2013regularization}) and (\cite{zhai2018adaptive}), the RFWDGBR algorithm does not have dropouts, but it has both $\ell_1$ and $\ell_2$ constraints. The model class with both $\ell_1$ and $\ell_2$ constraints is no larger than the minimum of the model class with either $\ell_1$ or $\ell_2$ constraint, which explains the component, $ \min( \sqrt{\lambda_4 l_K},\lambda_5)$, in the term RC.

\hide{Note that the Rademacher complexity describes the model complexity of $f(\mathbf{X})$. Only the weight constraints in the auto-encoder and in the logistic regression determine the Rademacher complexity. This complexity is closely related to deriving the Rademacher complexity with dropouts in Neural Networks in (\cite{wan2013regularization}) and (\cite{zhai2018adaptive}). Our DGBR does not have dropouts, so the retain rate is 1. Furthermore, weights in the logistic regression both have $l1$ and $l2$ constraints, so the feasible set from both constraints is a subset of the feasible set from each individual constraint. Thus, the Rademacher complexity is upper bounded by the minimum of the complexities from $l1$ and $l2$ constraints respectively. }

Moreover, the upper bound for  $L_{P}(\hat{f})$ depends on how close the approximate balancing is to the exact balancing, measured by $\epsilon_x$. If the approximate balancing gets very close to the exact balancing, the term $2 \max_{x, f}  \mathbb{E}[l(f(x), y)|x] \sum_{x} |\epsilon_x|$ approaches 0. The upper bound for $L_{P}(\hat{f})$ will mainly depend on three other terms in the right-hand side of Inequality (\ref{ineq:low-bound}).  

From this upper bound of $L_{P}(\hat{f})$, we have two main conclusions: 1. If the feasible set of the auto-encoder and the weights in the logistic regression is smaller, implying simpler models, the upper bound for $L_{P}(\hat{f})$ is smaller; 2. If covariates' imbalance is smaller, the upper bound for $L_{P}(\hat{f})$ is smaller.

\section{Optimization and Discussion}

\subsection{Optimization}

To optimize the aforementioned DGBR model in Eq.~(\ref{eq:deep_global_balancing_algorithm}), we need to minimize following $\mathcal{L}_{mix}$ as a function of parameter $W$, $\beta$, and $\theta = \{\mathbf{A}^{(k)}, \hat{\mathbf{A}}^{(k)}, b^{(k)}, \hat{b}^{(k)}\}$.
\begin{eqnarray}
\label{eq:final_obj}
\mathcal{L}_{mix}\!\!\!\!\!\! &=&\!\!\!\!\!\! \mathcal{L}_{Pre} + \lambda_1\mathcal{L}_{Bal} + \lambda_2\mathcal{L}_{AE} + \mathcal{L}_{Reg},\\
\nonumber &s.t.&\!\!\! W\succeq 0,
\end{eqnarray}
where
\begin{eqnarray}
\label{eq:part_obj}
\!\!\!\!\!\!\mathcal{L}_{Pre}\!\!\!\! &=&\!\!\!\!  \Scale[0.9]{\sum_{i=1}^{n}W_i\cdot \log(1+\exp((1-2Y_i)\cdot (\phi(\mathbf{X}_i)\beta)))},\\
\!\!\!\!\!\!\mathcal{L}_{Bal} \!\!\!\! &=&\!\!\!\!  \Scale[0.9]{\sum_{j=1}^{p} \big\|\frac{\phi(\mathbf{X}_{\cdot,-j})^{T}\cdot(W\odot \mathbf{X}_{\cdot,j})}{W^{T}\cdot \mathbf{X}_{\cdot,j}}-\frac{\phi(\mathbf{X}_{\cdot,-j})^{T}\cdot(W\odot (1-\mathbf{X}_{\cdot,j}))}{W^{T}\cdot (1-\mathbf{X}_{\cdot,j})}\big\|_2^2},\\
\!\!\!\!\!\!\mathcal{L}_{AE}\!\!\!\!  &=&\!\!\!\! \Scale[0.9]{\|(W\cdot \bm{1})\odot(X-\hat{X})\|_F^2},\\
\!\!\!\!\!\! \mathcal{L}_{Reg}\!\!\!\! &=&\!\!\!\! \Scale[0.9]{\lambda_3\|W\|_2^2 + \lambda_4 \|\beta\|_2^2 + \lambda_5 \|\beta\|_1 + \lambda_6 \Scale[1.0]{(\sum_{i=1}^{n}W_i-1)^{2}}}\\
\nonumber &+&\!\!\!\! \lambda_7 \Scale[0.9]{\sum_{k=1}^{K}(\|\mathbf{A}^{(k)}\|_F^2+\|\hat{\mathbf{A}}^{(k)}\|_F^2)}.
\end{eqnarray}

Here, we propose an iterative method to minimize the above objective function in Eq. (\ref{eq:final_obj}).
Starting from some random initialization on parameters $W$, $\beta$ and $\theta$, we update each of them alternatively with the other two parameters as fixed at each iteration until convergence.
These steps are described below:

\par \noindent \textbf{Update $\beta$}: When fixing $W$ and $\theta$, the problem (\ref{eq:final_obj}) is equivalent to optimize following objective function:
\begin{eqnarray}
\label{eq:part_obj_beta}
\mathcal{L}_{mix}(\beta)\!\!\!\! &=&\!\!\!\! \Scale[0.9]{\sum_{i=1}^{n}W_i\cdot \log(1+\exp((1-2Y_i)\cdot (\phi(\mathbf{X}_i)\beta)))}\\
\nonumber &+&\!\!\!\! \lambda_4 \|\beta\|_2^2 + \lambda_5 \|\beta\|_1,
\end{eqnarray}
which is a standard $\ell_1$ norm regularized least squares problem and can be easily solved by any LASSO (or elastic net) solver.

\par \noindent \textbf{Update $W$}: By fixing $\beta$ and $\theta$, the key step for updating $W$ is to calculate the partial derivative of $\frac{\partial \mathcal{L}_{mix}}{\partial W}$. The detailed mathematical form of the partial derivative is shown as following:
\begin{eqnarray}
\label{eq:part_obj_W}
\frac{\partial \mathcal{L}_{mix}}{\partial W} = \frac{\partial \mathcal{L}_{Pre}}{\partial W} + \frac{\partial \mathcal{L}_{Bal}}{\partial W} + \frac{\partial \mathcal{L}_{AE}}{\partial W} + \frac{\partial \mathcal{L}_{Reg}}{\partial W},
\end{eqnarray}
where
\begin{eqnarray}
\frac{\partial \mathcal{L}_{Pre}}{\partial W} &=& \log(1+\exp((1-2Y)\cdot(\phi(\mathbf{X})\cdot \beta))),\\
\frac{\partial \mathcal{L}_{Bal}}{\partial W} &=& \Scale[1.0]{2\sum_{j=1}^{p}\mathcal{L}_{Bal_j}\cdot \frac{\partial \mathcal{L}_{Bal_j}}{\partial W}},\\
\frac{\partial \mathcal{L}_{AE}}{\partial W} &=& 2((W\cdot \bm{1})\odot(\mathbf{X}-\hat{\mathbf{X}}))\odot(\mathbf{X}-\hat{\mathbf{X}})\cdot\bm{1}^{T},\\
\frac{\partial \mathcal{L}_{Reg}}{\partial W} &=& 2\lambda_3 W.
\end{eqnarray}
where
$\Scale[1.0]{\mathcal{L}_{Bal_j} = \frac{\phi(\mathbf{X}_{\cdot,-j})^{T}\cdot(W\odot \mathbf{X}_{\cdot,j})}{W^{T}\cdot \mathbf{X}_{\cdot,j}}-\frac{\phi(\mathbf{X}_{\cdot,-j})^{T}\cdot(W\odot (1-\mathbf{X}_{\cdot,j}))}{W^{T}\cdot (1-\mathbf{X}_{\cdot,j})}}$ and 
\begin{eqnarray}
\nonumber \Scale[1.0]{\frac{\partial \mathcal{L}_{Bal_j}}{\partial W}} &=& \Scale[1.0]{\frac{\phi(\mathbf{X}_{\cdot,-j})^{T}\odot (\mathbf{X}_{\cdot,j}\cdot \mathbf{1}^T)^T\cdot(W^T\cdot \mathbf{X}_{\cdot,j})}{(W^T\cdot \mathbf{X}_{\cdot,j})^2}} \\
\nonumber && - \Scale[1.0]{\frac{\phi(\mathbf{X}_{\cdot,-j})^{T}\odot ((1-\mathbf{X}_{\cdot,j})\cdot \mathbf{1}^T)^T\cdot(W^T\cdot (1-\mathbf{X}_{\cdot,j}))}{(W^T\cdot (1-\mathbf{X}_{\cdot,j}))^2}}.
\end{eqnarray}
For ensuring the non-negative of $W$ with constraint $W\succeq 0$, we let $W = \omega\odot \omega$, where $\omega \in \mathbb{R}^{n \times1}$.Then we update $W$ by updating $\omega$ with following partial derivative.
\begin{eqnarray}
\frac{\partial \mathcal{L}_{mix}}{\partial \omega} = \frac{\partial \mathcal{L}_{mix}}{\partial W} \cdot \frac{\partial W}{\partial \omega}
\end{eqnarray}

\par \noindent \textbf{Update $\theta$}: By fixing $W$ and $\beta$, the key step for updating $\theta$ is to calculate the partial derivative of $\frac{\partial \mathcal{L}_{mix}}{\partial A^{(k)}}$ and $\frac{\partial \mathcal{L}_{mix}}{\partial \hat{A}^{(k)}}$. The detailed mathematical form of the partial derivative is shown as following:
\begin{eqnarray}
\label{eq:part_obj_theta1}
\frac{\partial \mathcal{L}_{mix}}{\partial A^{(k)}}\!\!\!\! &=&\!\!\!\! \frac{\partial \mathcal{L}_{Pre}}{\partial A^{(k)}} + \lambda_1\frac{\partial \mathcal{L}_{Bal}}{\partial A^{(k)}} + \lambda_2\frac{\partial \mathcal{L}_{AE}}{\partial A^{(k)}} + \frac{\partial \mathcal{L}_{Reg}}{\partial A^{(k)}},\\
\label{eq:part_obj_theta2}
\frac{\partial \mathcal{L}_{mix}}{\partial \hat{A}^{(k)}}\!\!\!\! &=&\!\!\!\! \lambda_2\frac{\partial \mathcal{L}_{AE}}{\partial \hat{A}^{(k)}} + \frac{\partial \mathcal{L}_{Reg}}{\partial \hat{A}^{(k)}}, \ \ \ \  k = 1,\cdots,K
\end{eqnarray}
First we look at the first term $\frac{\partial \mathcal{L}_{Pre}}{\partial A^{(K)}}$in $\frac{\partial \mathcal{L}_{mix}}{\partial A^{(K)}}$, which can be rephrased as follows:
\begin{eqnarray}
\frac{\partial \mathcal{L}_{Pre}}{\partial A^{(K)}} = \frac{\partial \mathcal{L}_{Pre}}{\partial \phi(X)}\cdot \frac{\partial \phi(X)}{\partial A^{(K)}},
\end{eqnarray}
According to Eq.~\ref{eq:part_obj}, we can obtain $\frac{\partial \mathcal{L}_{Pre}}{\partial \phi(X)}$.
The calculation of the second term $\frac{\partial \phi(X)}{\partial A^{(k)}}$ is easy since $\phi(X) = \sigma(\phi(X)^{(K-1)}A^{(K)}+b^{(K)})$.
Then $\frac{\partial \mathcal{L}_{Pre}}{\partial A^{(K)}}$ is accessible.
Based on the back-propagation, we can iteratively obtain $\frac{\partial \mathcal{L}_{Pre}}{\partial A^{(k)}}, k= 1,\cdots,K-1$. Now the calculation of the partial derivative of $\mathcal{L}_{Pre}$ is finished.

Similarly, by using back-propagation we can finish the calculation of $\mathcal{L}_{Bal}$ and $\mathcal{L}_{AE}$. 
Finally we can obtain the $\frac{\partial \mathcal{L}_{mix}}{\partial A^{(k)}}$ and $\frac{\partial \mathcal{L}_{mix}}{\partial \hat{A}^{(k)}}$ for $k = 1,\cdots,K$, and update our parameter $\theta$.

We update $W$, $\beta$ and $\theta = \{A^{(k)}, \hat{A}^{(k)}, b^{(k)}, \hat{b}^{(k)}\}$ iteratively until the objective function converges. The whole algorithm is summarized in Algorithm~\ref{alg:dgb}.

Finally, with the optimized regression coefficient $\beta$ and deep auto-encoder parameters $\theta$ by our DGBR algorithm, we can make a stable prediction on various agnostic test datasets.

\begin{algorithm}[tbp]
\caption{{Deep Global Balancing Regression algorithm}}
\label{alg:dgb}
\begin{algorithmic}[1]
\Require
Observed Variables Matrix X and Response Variable $Y$.
\Ensure
Updated Parameters $W$, $\beta$, $\theta$.
\State Initialize parameters  $W^{(0)}$, $\beta^{(0)}$ and $\theta^{(0)}$,
\State Calculate the current value of $\mathcal{L}_{mix}^{(0)} = \mathcal{J}(W^{(0)},\beta^{(0)},\theta^{(0)})$ with Equation~(\ref{eq:final_obj}),
\State Initialize the iteration variable $t\leftarrow 0$,

\Repeat
\State $t\leftarrow t+1$,
\State Update $W^{(t)}$ based on Eq.~(\ref{eq:part_obj_W}),
\State Update $\beta^{(t)}$ by solving $\mathcal{L}_{mix}(\beta^{(t-1)})$ in Equation~(\ref{eq:part_obj_beta}),
\State Based on Eq.~(\ref{eq:part_obj_theta1}) and (\ref{eq:part_obj_theta2}), use $\Scale[1.0]{\frac{\partial \mathcal{L}_{mix}}{\partial \theta}}$ to back-propagate through  the entire deep network to get updated parameters $\theta$,
\State Calculate $\mathcal{L}_{mix}^{(t)} = \mathcal{J}(W^{(t)},\beta^{(t)},\theta^{(t)})$,
\Until{$\mathcal{L}_{mix}^{(t)}$ converges or max iteration is reached}.\\
\Return $W$, $\beta$, $\theta$.
\end{algorithmic}
\end{algorithm}

\hide{
\begin{algorithm}[tbp]
\caption{{Deep Global Balancing Regression algorithm}}
\label{alg:dgb}
\begin{algorithmic}[1]
\Require
Observed Feature Matrix X and Response Variable $Y$.
\Ensure
Updated Parameters $W$, $\beta$, $\theta$.  
\State Initialize parameters  $W^{(0)}$, $\beta^{(0)}$ and $\theta^{(0)}$,
\State Calculate value of loss function with parameters $\Scale[0.9]{(W^{(0)},\beta^{(0)},\theta^{(0)})}$,
\State Initialize the iteration variable $t\leftarrow 0$,

\Repeat
\State $t\leftarrow t+1$,
\State Update $W^{(t)}$ by gradient descent and fixing $\beta$ and $\theta$,
\State Update $\beta^{(t)}$ by gradient descent and fixing $W$ and $\theta$,
\State Update $\theta^{(t)}$ by gradient descent and fixing $W$ and $\beta$,
\State Calculate loss function with parameters $(W^{(t)},\beta^{(t)},\theta^{(t)})$,
\Until{Loss function converges or max iteration is reached}.\\
\Return $W$, $\beta$, $\theta$.
\end{algorithmic}
\end{algorithm}
}

\subsection{Complexity Analysis}

During the procedure of optimization, the main time cost is to calculate the loss function, update parameters $W$, $\beta$ and $\theta$.
For calculating the loss function, its complexity is $O(npd)$, where $n$ is the sample size, $p$ is the dimension of observed variables and $d$ is the maximum dimension of the hidden layer in deep auto-encoder model.
For updating parameter $W$, its complexity is dominated by the step of calculating the partial gradients of loss function with respect to variable $W$. Its complexity is also $O(npd)$.
For updating parameter $\beta$, it is a standard LASSO problem and its complexity is $O(nd)$.
For updating $\theta$, its complexity is $O(npd)$.

In total, the complexity of each iteration in Algorithm~\ref{alg:dgb} is $O(npd)$.

\subsection{Parameter Tuning}

To tune the parameters for our algorithm and baselines, we need multiple validation datasets whose distributions are diverse from each other and different with the training data.
In our experiments, we generate such validation datasets $\mathcal{E}$ by non-random data resampling on training data. 
We calculate the $Average\_Error$ and $Stability\_Error$ of all algorithms on validation datasets by choosing $RMSE$ as $Error$ metrics in Eq. (\ref{metrics:acc}) and (\ref{metrics:stb}). In this paper, we tune all the parameters for our algorithm and baselines by minimizing $Average\_Error + \lambda \cdot Stability\_Error$ on validation datasets with cross validation by grid searching. We set $\lambda = 5$ in our experiments.
\hide{\todo{Where are the details of how the tuning set is constructed?  Should be in Appendix, I can't find details here.}}
\hide{Note that our results may in principle be sensitive to how the validation datasets are constructed, and that our approach is heuristic.  We leave a further exploration of this issue for future work.}

\noindent \textbf{Construction of Validation Data.}  The key point in construction of validation data is to construct datasets where the joint distribution of the covariates changes across environments, particularly when this might create bias if we don't control for all of the stable features.
However, we do not have prior knowledge about which features are noisy features. Fortunately, our estimation approach can identify noisy features as those that do not have a large estimated effect after balancing.
Using the empirically identified noisy features, we can generate validation datasets that change the distribution of noisy features and use these for parameter tuning.

\section{Experiments}

In this section, we evaluate our algorithm on both synthetic and real world dataset, comparing with the state-of-the-art methods.

\subsection{Baselines}
We implement following baselines for comparition.
\begin{itemize}[leftmargin=0.5cm]
\item \emph{Logistic Regression (LR)} (\cite{menard2002applied})
\item \emph{Deep Logistic Regression (DLR)} (\cite{chen2014deep}): Combines a deep auto-encoder and logistic regression.
\item \emph{Global Balancing Regression (GBR)}: Combines a global balancing regularizer and logistic regression as shown in Eq~(\ref{eq:global_balancing_algorithm}).
\end{itemize}
Since our proposed algorithm is based on logistic regression, so we compare our algorithm with only logistic regression methods. For other predictive methods, we can propose corresponding global balancing algorithm based on them, and compare with them.

\hide{
\subsection{Baselines}
\begin{itemize}[leftmargin=0.5cm]
\item \emph{Logistic Regression (LR)} \cite{menard2002applied}: It is a typical correlation-based method and widely used in many classical classification and prediction problems. It can not remove the spurious effect of noisy features during model training, so it can hardly make stable prediction across environments.
\item \emph{Deep Logistic Regression (DLR)} \cite{chen2014deep}: It makes a prediction with combination of deep auto-encoder and logistic regression. It can capture the non-linear interaction among variables for prediction. But it also can not remove the spurious effect on noisy features.
\item \emph{Global Balancing Regression (GBR)}: It makes a prediction by combining global balancing regularizer and logistic regression as shown in Eq~(\ref{eq:global_balancing_algorithm}). It can reduce some selection bias on training data, but can not handle it well under non-linear and high dimensional settings.
\end{itemize}
Since our proposed algorithm is based on logistic regression, we compare our algorithm with only logistic regression methods in our experiments. If we considered an alternative predictive method, we could compare the baseline version of the method with a global balancing algorithm based on it.
}

\subsection{Experiments on Synthetic Data}

In this section, we describe the synthetic datasets and demonstrate the effectiveness of our proposed algorithm.

\subsubsection{Dataset}

As shown in Fig. \ref{fig:graph}, there are three relationships between $\mathbf{X} = \{\mathbf{S}, \mathbf{V}\}$ and $Y$, including $\mathbf{S}\perp \mathbf{V}$, $\mathbf{S}\rightarrow \mathbf{V}$, and $\mathbf{V}\rightarrow \mathbf{S}$.

\noindent \textbf{$\mathbf{S}\perp \mathbf{V}$:} In this setting, $\mathbf{S}$ and $\mathbf{V}$ are independent. 
Recalling Fig. \ref{fig:graph}, we generate predictor $\mathbf{X} =\{\mathbf{S}_{\cdot,1}, \cdots, \mathbf{S}_{\cdot,p_s}, \mathbf{V}_{\cdot,1}, \cdots, \mathbf{V}_{\cdot,p_v}\}$ with independent Gaussian distributions as:
\begin{eqnarray}
\Scale[1.0]{\mathbf{\tilde{S}}_{\cdot,1}, \cdots, \mathbf{\tilde{S}}_{\cdot,p_s}, \mathbf{\tilde{V}}_{\cdot,1}, \cdots, \mathbf{\tilde{V}}_{\cdot,p_v}\ \quad  \overset{iid}{\sim}\quad \mathcal{N}(0,1)}, \nonumber
\end{eqnarray}
where $p_s = 0.4*p$, and $\mathbf{S}_{\cdot,j}$ represents the $j^{th}$ variable in $\mathbf{S}$. 
To make $\mathbf{X}$ binary, we let $\mathbf{X}_{\cdot,j}=1$ if $\mathbf{\tilde{X}}_{\cdot,j}\geq 0$, otherwise $\mathbf{X}_{\cdot,j}=0$.

\noindent \textbf{$\mathbf{S}\rightarrow \mathbf{V}$:} In this setting, the stable features $\mathbf{S}$ are the causes of noisy features $\mathbf{V}$. We first generate the stable features $\mathbf{\tilde{S}}$ with independent Gaussian distributions, and let $\mathbf{S}_{\cdot,j}=1$ if $\mathbf{\tilde{S}}_{\cdot,j}\geq 0$, otherwise $\mathbf{S}_{\cdot,j}=0$.
Then, we generate noisy features $\mathbf{\tilde{V}} = \{\mathbf{\tilde{V}}_{\cdot,1}, \cdots, \mathbf{\tilde{V}}_{\cdot,p_v}\}$ based on $\mathbf{\tilde{S}}$: $$\Scale[1.0]{\mathbf{\tilde{V}}_{\cdot,j} = \mathbf{\tilde{S}}_{\cdot,j}+\mathbf{\tilde{S}}_{\cdot,j+1}+\mathcal{N}(0,2)},$$ and let $\mathbf{V}_{\cdot,j}=1$ if $\mathbf{\tilde{V}}_{\cdot,j} > 1$, otherwise $\mathbf{V}_{\cdot,j}=0$.

\noindent \textbf{$\mathbf{V}\rightarrow \mathbf{S}$:} In this setting, the noisy features $\mathbf{V}$ are the causes of stable features $\mathbf{S}$. We first generate the noisy features $\mathbf{\tilde{V}}$ with independent Gaussian distribution, and let $\mathbf{V}_{\cdot,j}=1$ if $\mathbf{\tilde{V}}_{\cdot,j}\geq 0$, otherwise $\mathbf{V}_{\cdot,j}=0$. Then, we generate stable features $\mathbf{S} = \{\mathbf{S}_{\cdot,1}, \cdots, \mathbf{S}_{\cdot,p_s}\}$ based on $\mathbf{\tilde{V}}$: $$\Scale[1.0]{\mathbf{\tilde{S}}_{\cdot,j} = \mathbf{\tilde{V}}_{\cdot,j}+\mathbf{\tilde{V}}_{\cdot,j+1}+\mathcal{N}(0,2)},$$ and let $\mathbf{S}_{\cdot,j}=1$ if $\mathbf{\tilde{S}}_{\cdot,j} > 1$, otherwise $\mathbf{S}_{\cdot,j}=0$.

Finally, we generate the response variable $Y$ for all above three settings with the same function $g$ as following:
\begin{eqnarray}
\nonumber Y\!\!\!&=&\!\!\!\Scale[0.88]{1/(1+\exp(-\sum_{\mathbf{X}_{\cdot,i}\in \mathbf{S}_l}\alpha_i\cdot \mathbf{X}_{\cdot,i}-\sum_{\mathbf{X}_{\cdot,j}\in \mathbf{S}_n}\beta_j\cdot \mathbf{X}_{\cdot,j}\cdot \mathbf{X}_{\cdot,j+1}))}\\
\nonumber && + \Scale[0.9]{\mathcal{N}(0,0.2)},
\end{eqnarray}
where we separate the stable features $\mathbf{S}$ into two parts, linear part $\mathbf{S}_l$ and non-linear part $\mathbf{S}_n$. And $\alpha_i = (-1)^{i}\cdot (i\%3+1)\cdot p/3$ and $\beta_j = p/2$. To make $Y$ binary, we set $Y=1$ when $Y\geq0.5$, otherwise $Y=0$.

To test the stability of all algorithms, we need to generate a set environments $e$, each with a distinct joint distribution.
Under Assumption \ref{asmp:stable}, instability in prediction arises because $P(Y|\mathbf{V})$ or $P(\mathbf{V}|\mathbf{S})$ varies across environments.
Therefore, we generate different environments in our experiments by varying $P(Y|\mathbf{V})$ and $P(\mathbf{V}|\mathbf{S})$.

\begin{figure*}[tb]
\centering
\subfloat[\scriptsize{Trained on $n=1000,p=20,r=0.65$} \label{fig:RMSE_1000_20_65}]{
  \includegraphics[width=2.1in]{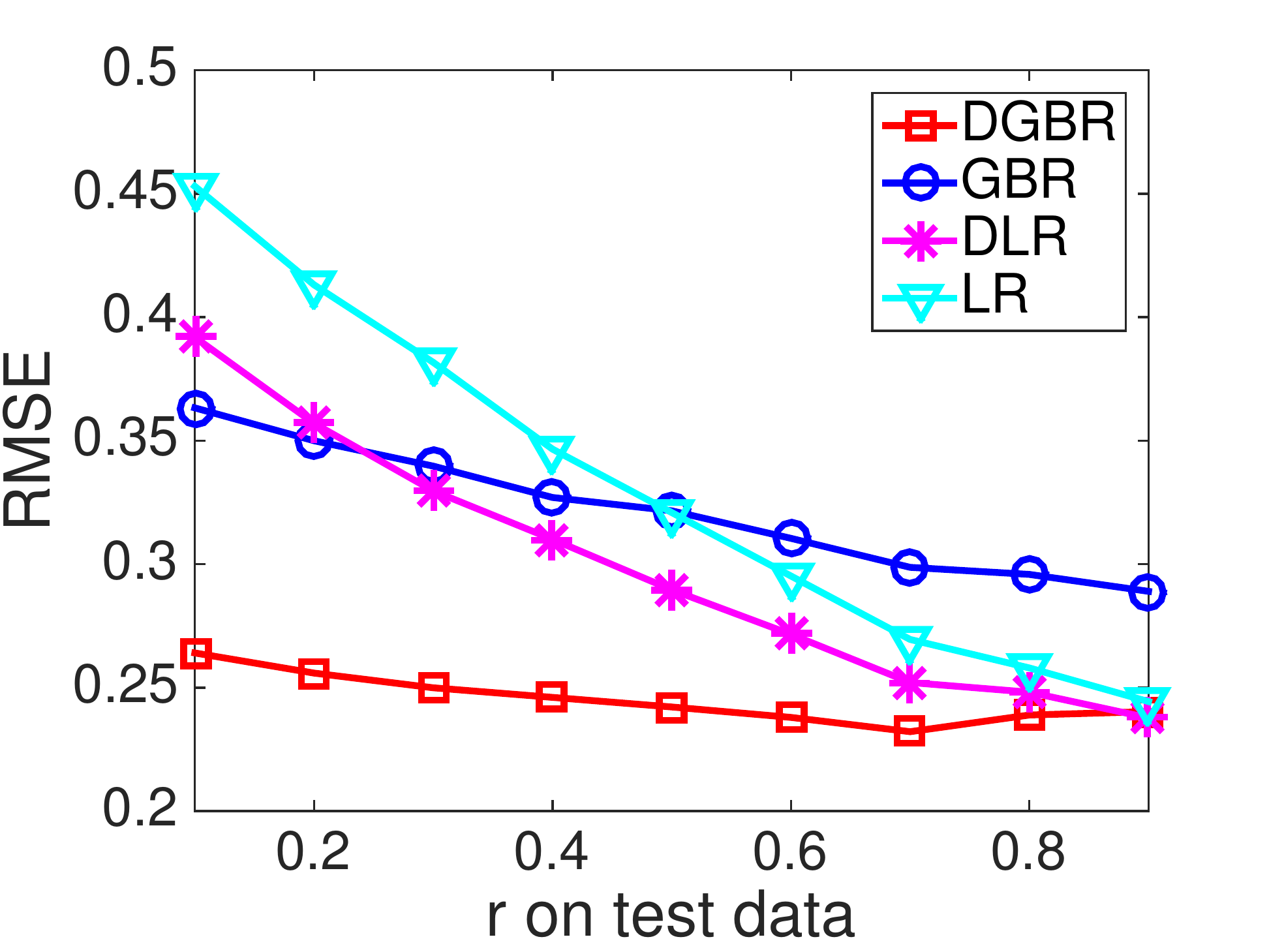}
}
\subfloat[\scriptsize{Trained on $n=1000,p=20,r=0.75$}\label{fig:RMSE_1000_20_75}]{
  \includegraphics[width=2.1in]{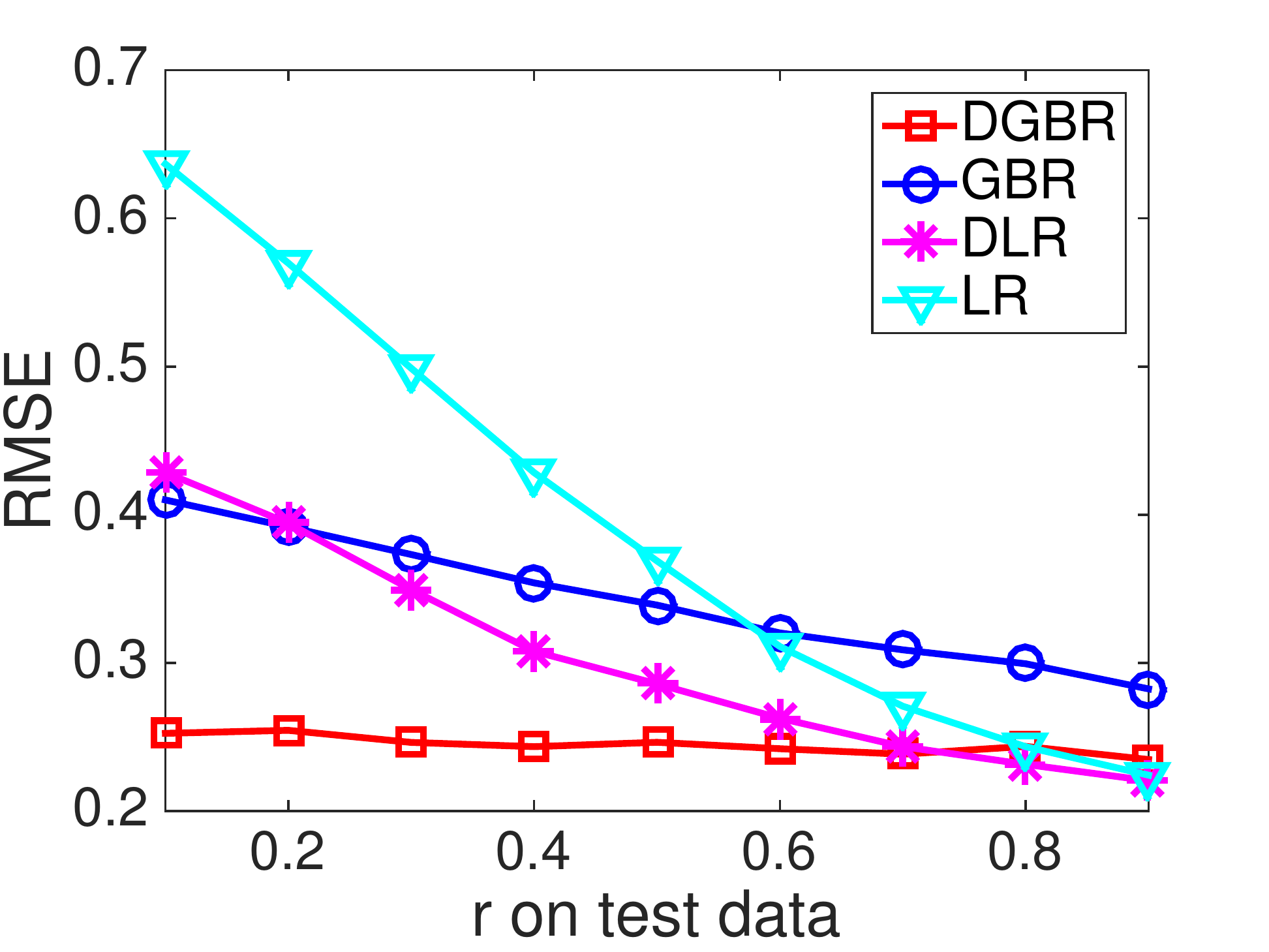}
}
\subfloat[\scriptsize{Trained on $n=1000,p=20,r=0.85$}\label{fig:RMSE_1000_20_85}]{
  \includegraphics[width=2.1in]{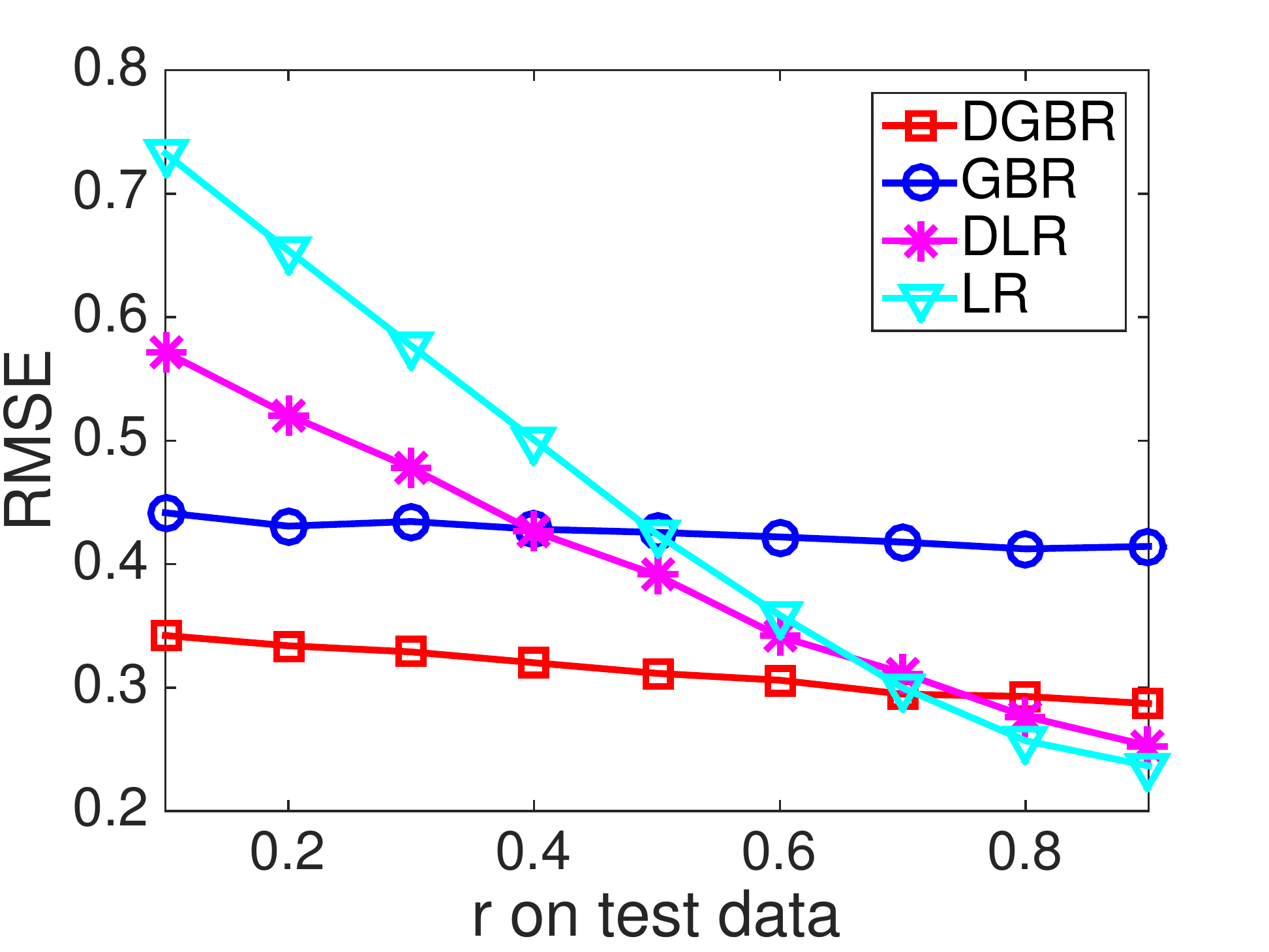}
}
\\
\subfloat[\scriptsize{Trained on $n=2000,p=20,r=0.65$} \label{fig:RMSE_2000_20_65}]{
  \includegraphics[width=2.1in]{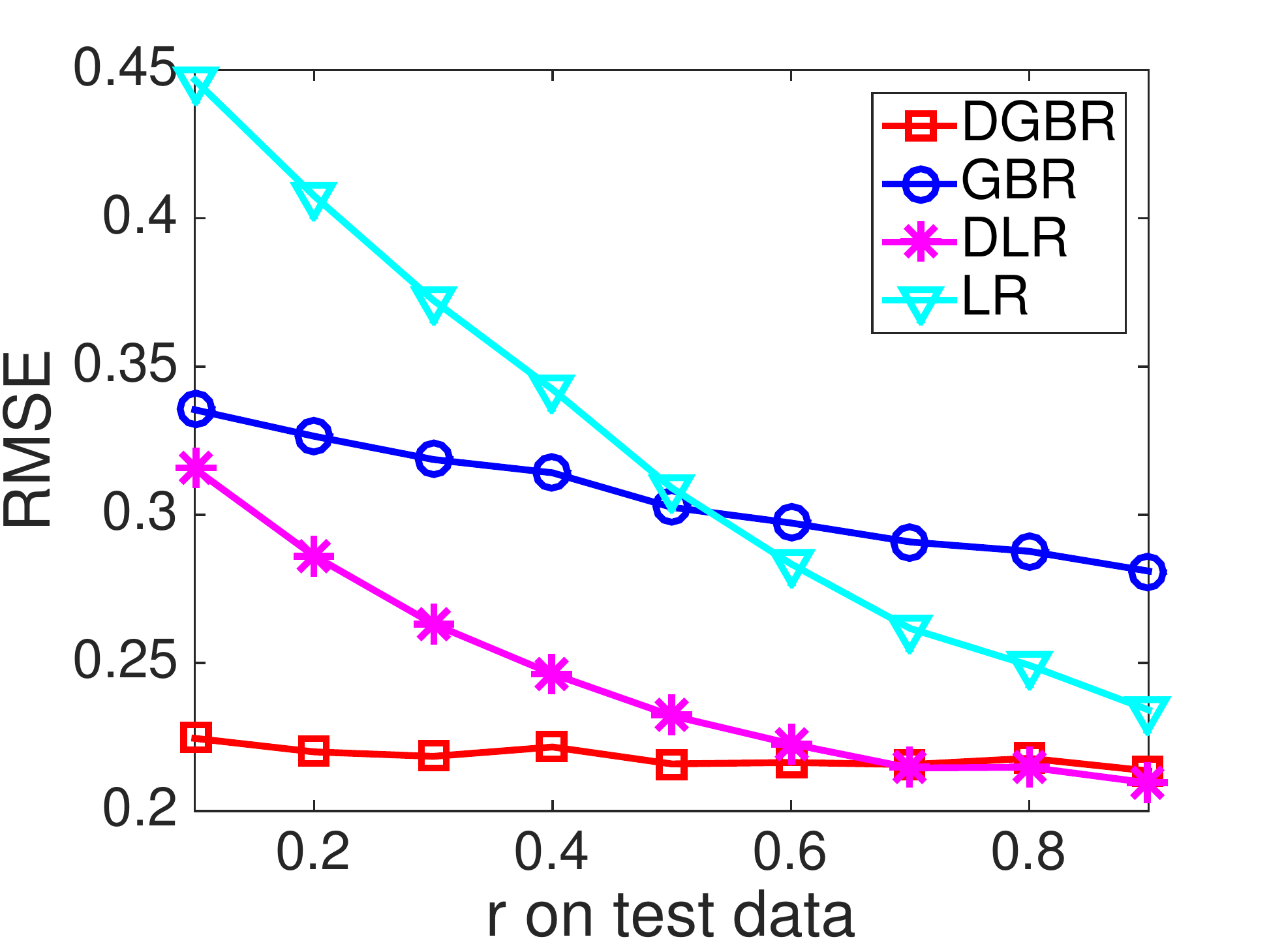}
}
\subfloat[\scriptsize{Trained on $n=2000,p=20,r=0.75$}\label{fig:RMSE_2000_20_75}]{
  \includegraphics[width=2.1in]{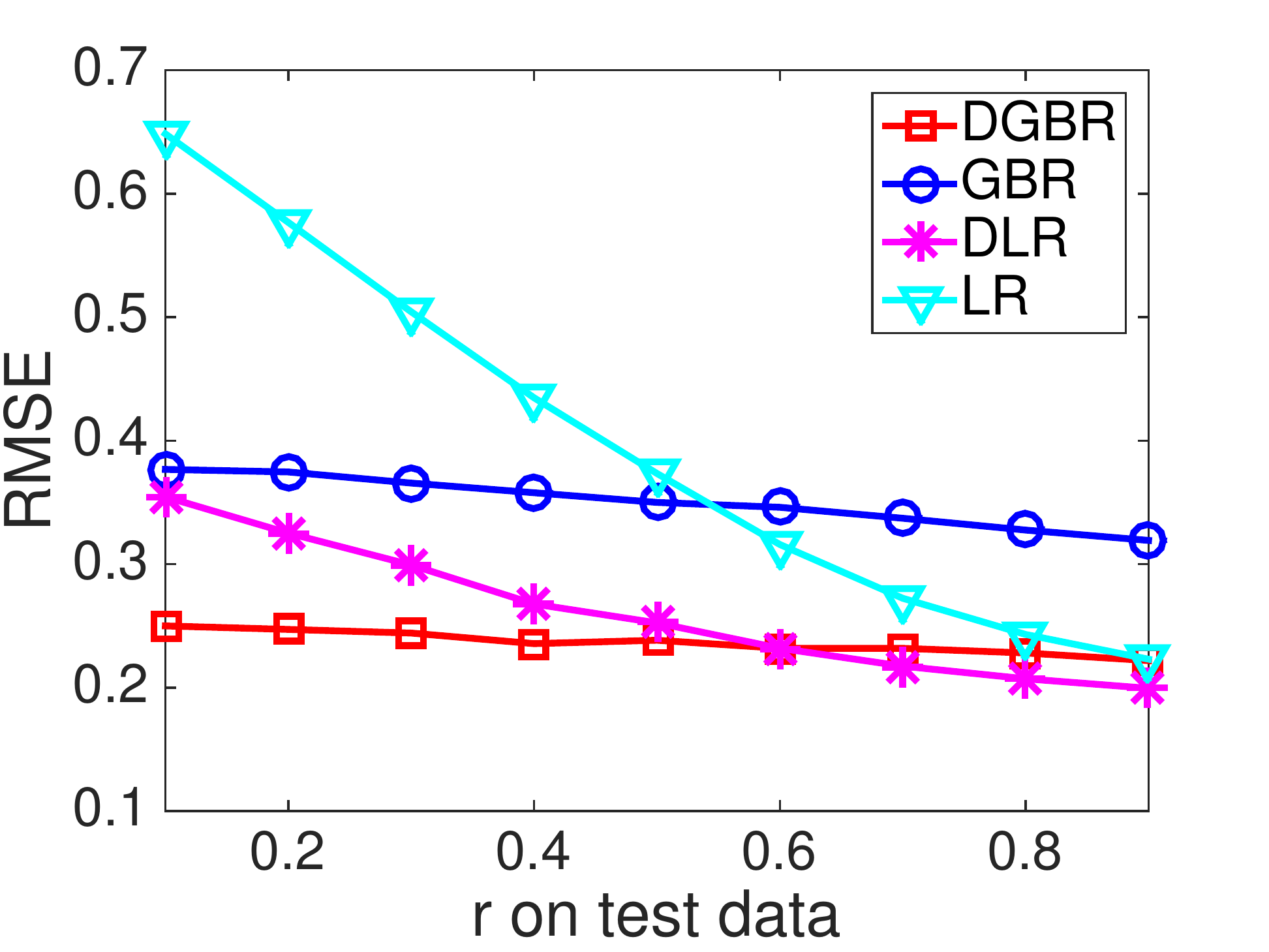}
}
\subfloat[\scriptsize{Trained on $n=2000,p=20,r=0.85$}\label{fig:RMSE_2000_20_85}]{
  \includegraphics[width=2.1in]{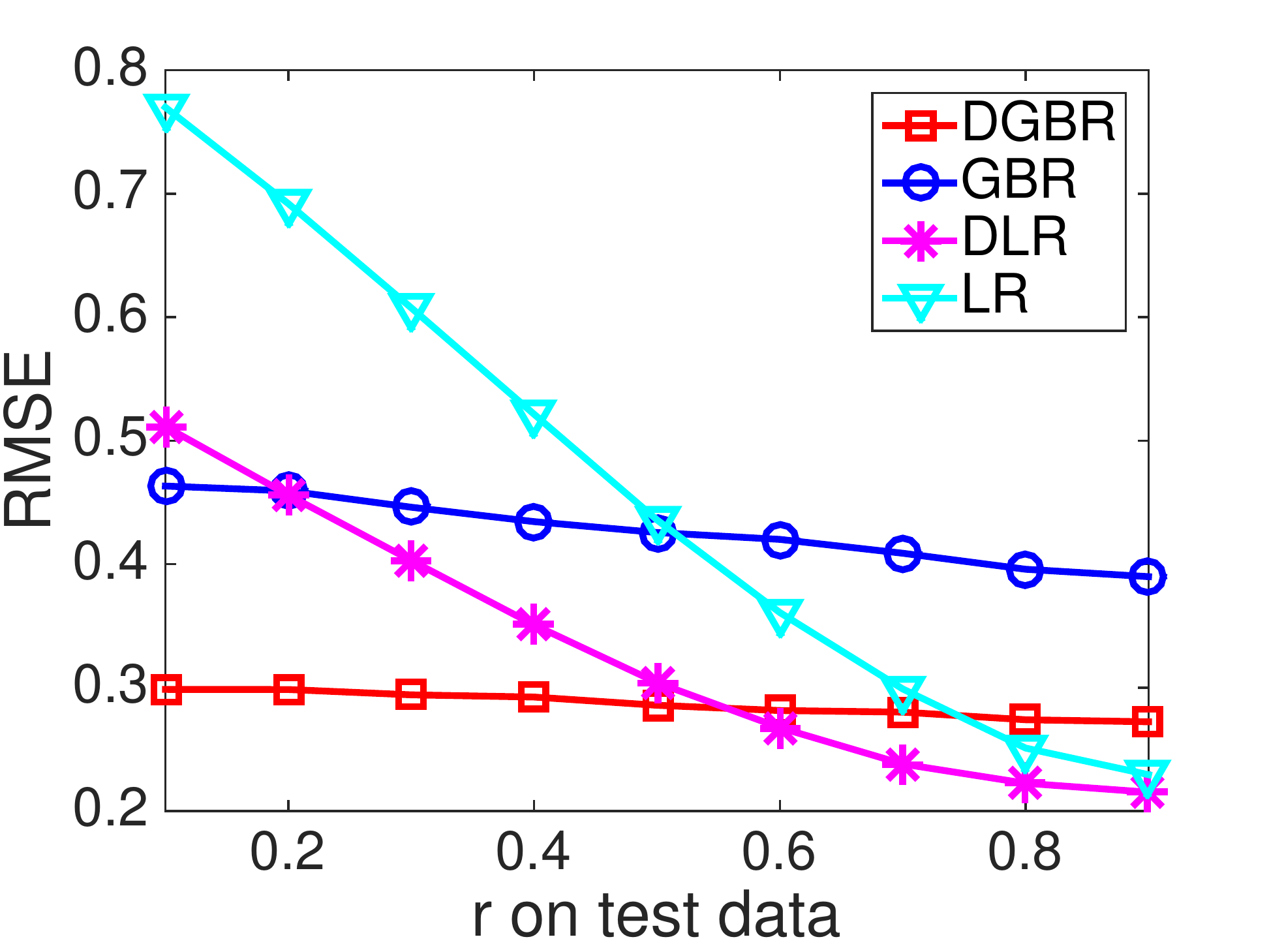}
}
\\
\subfloat[\scriptsize{Trained on $n=4000,p=20,r=0.65$} \label{fig:RMSE_4000_20_65}]{
  \includegraphics[width=2.1in]{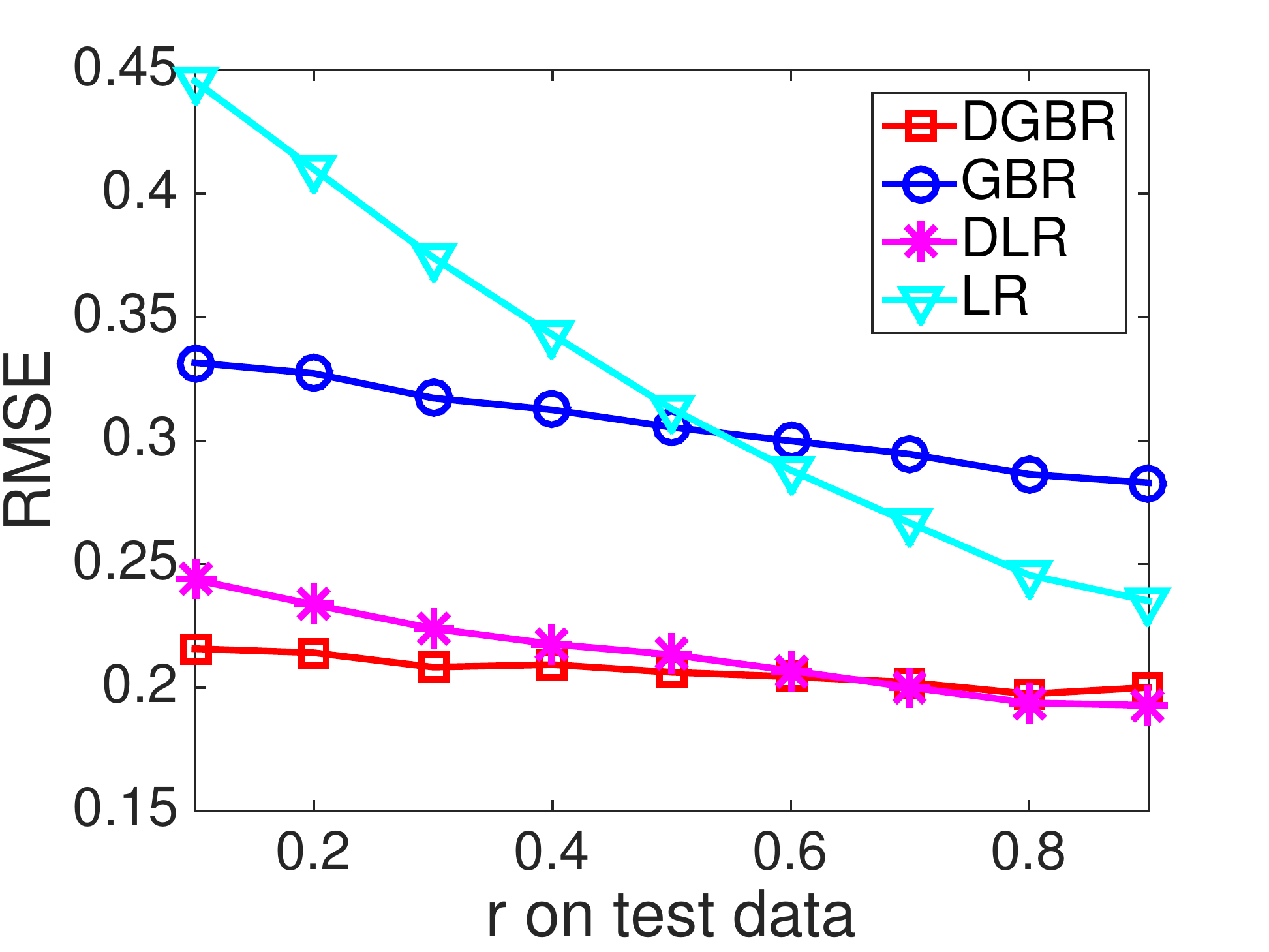}
}
\subfloat[\scriptsize{Trained on $n=4000,p=20,r=0.75$}\label{fig:RMSE_4000_20_75}]{
  \includegraphics[width=2.1in]{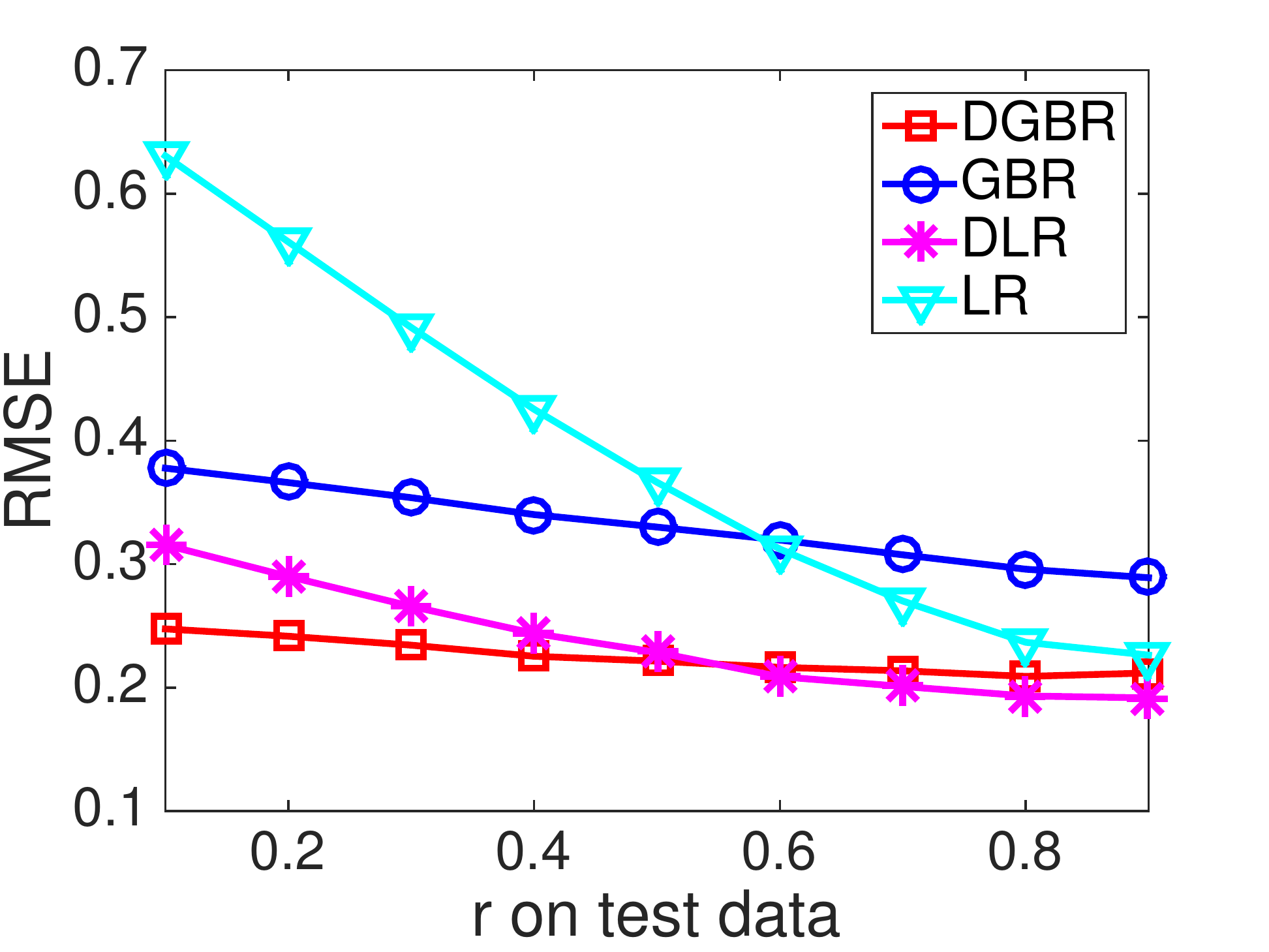}
}
\subfloat[\scriptsize{Trained on $n=4000,p=20,r=0.85$}\label{fig:RMSE_4000_20_85}]{
  \includegraphics[width=2.1in]{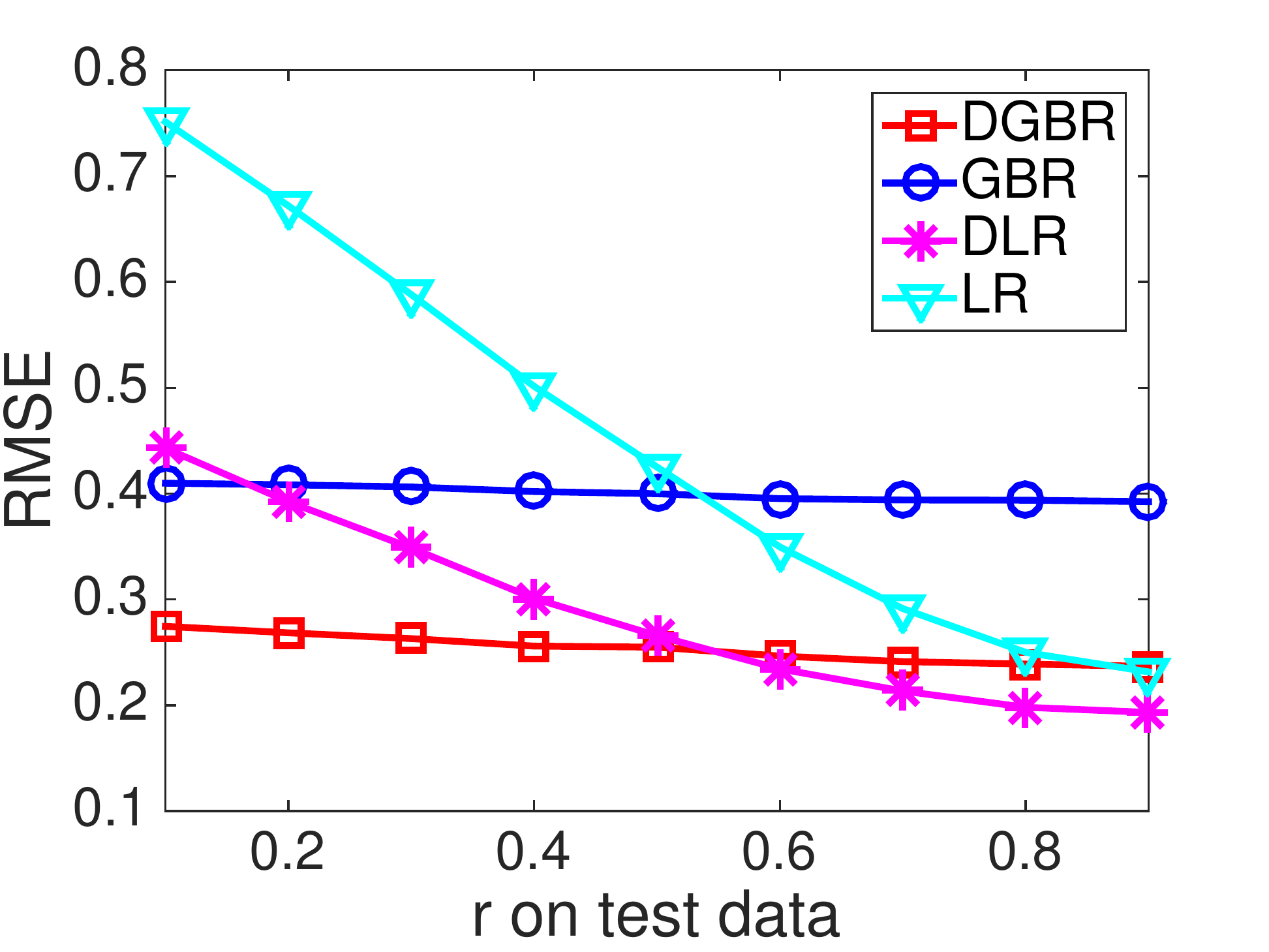}
}
\caption{\textbf{Setting $\mathbf{S}\perp \mathbf{V}$}: RMSE of outcome prediction on various test datasets by varying sample size $n$ (vertical) and bias rate $r$ (horizontal) on training dataset. The $r$ of the X-axis in each figure represents the bias rate on test data.}
\label{fig:simulation_n_r}
\end{figure*}

\begin{figure*}[tb]
\centering
\subfloat[\scriptsize{Trained on $n=4000,p=20,r=0.75$} \label{fig:RMSE_4000_20_75}]{
  \includegraphics[width=2.1in]{figures/RMSE_4000_20_75}
}
\subfloat[\scriptsize{Trained on $n=4000,p=40,r=0.75$}\label{fig:RMSE_4000_40_75}]{
  \includegraphics[width=2.1in]{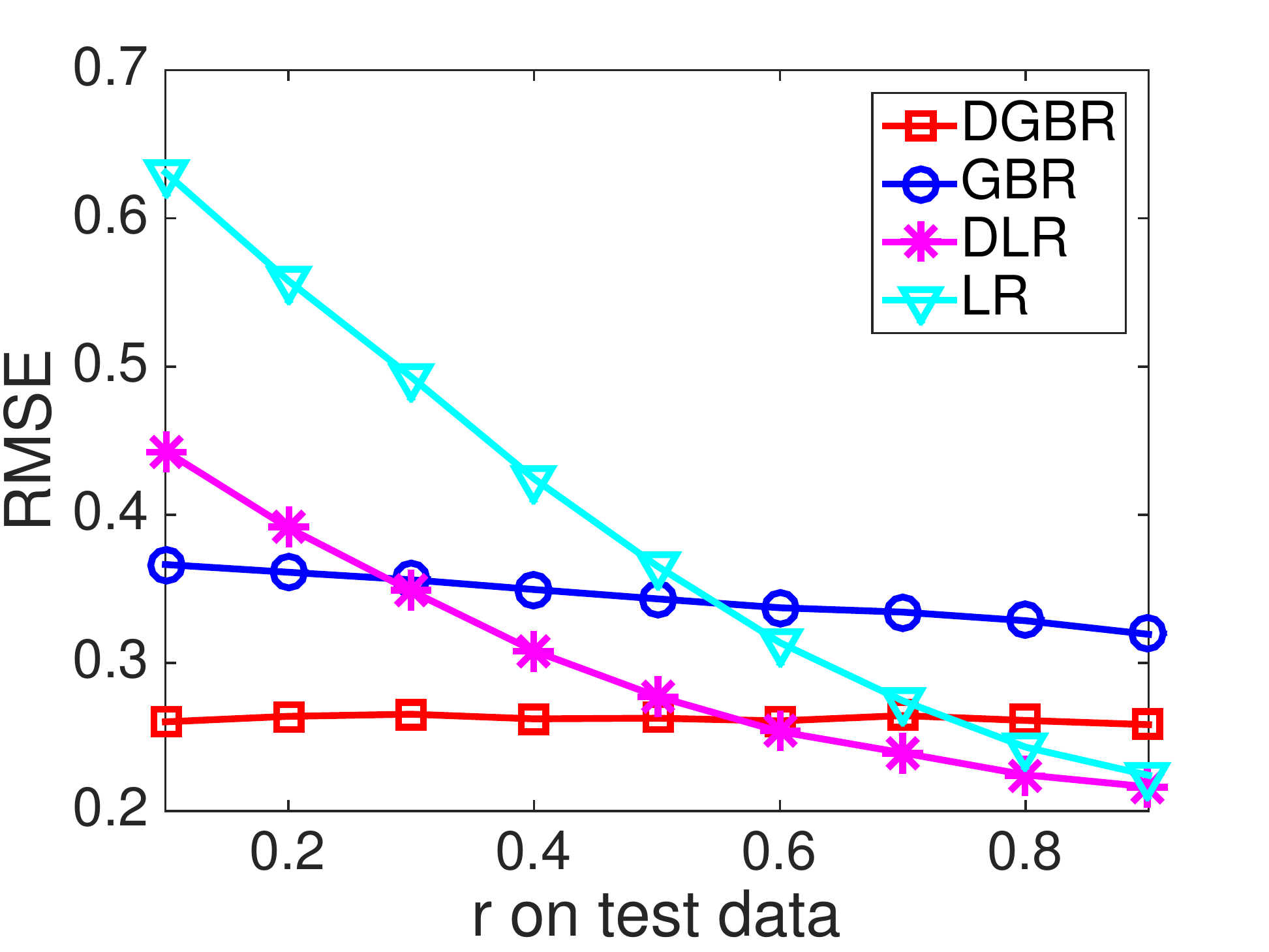}
}
\subfloat[\scriptsize{Trained on $n=4000,p=80,r=0.75$}\label{fig:RMSE_4000_80_75}]{
  \includegraphics[width=2.1in]{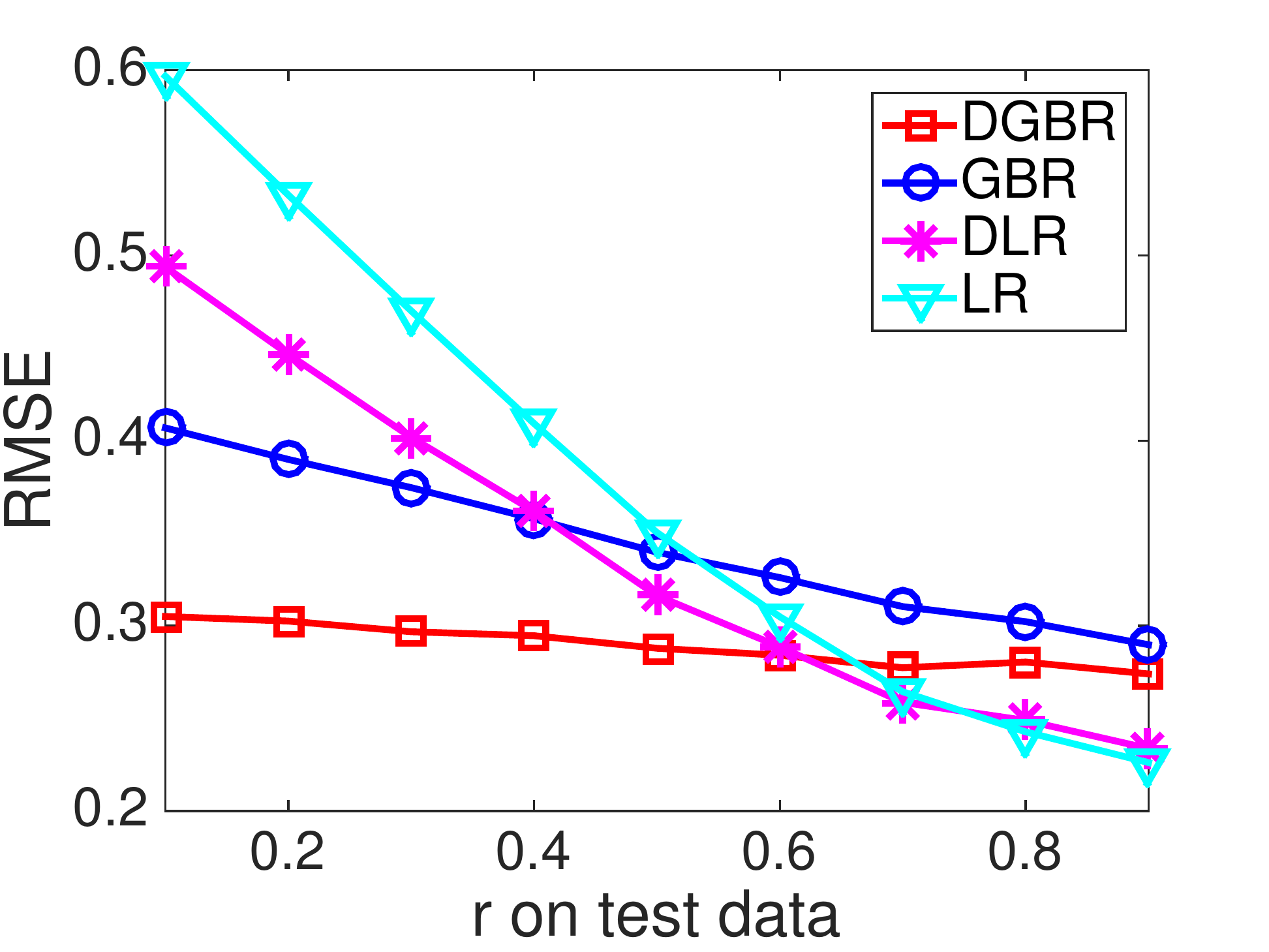}
}
\caption{\textbf{Setting $\mathbf{S}\perp \mathbf{V}$}: RMSE of outcome prediction on various test datasets by varying variables' dimension $p$ on training dataset.}
\label{fig:simulation_p}
\end{figure*}

\begin{figure*}[tb]
\centering
\subfloat[Trained on $n=1000,p=20,r=0.75$ \label{fig:RMSE_1000_20_75_s2v}]{
  \includegraphics[width=2.1in]{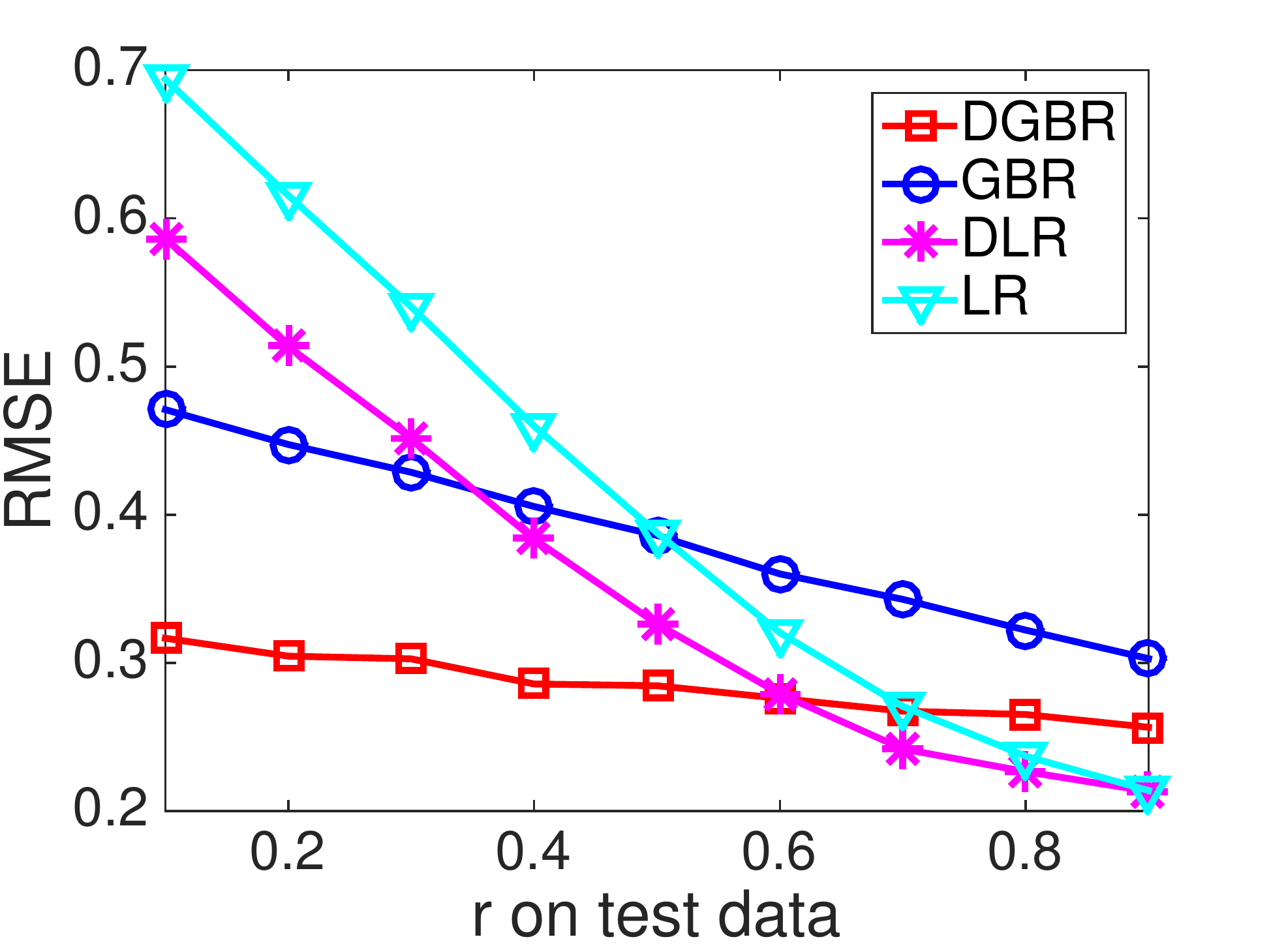}
}
\subfloat[Trained on $n=2000,p=20,r=0.75$\label{fig:RMSE_2000_20_75_s2v}]{
  \includegraphics[width=2.1in]{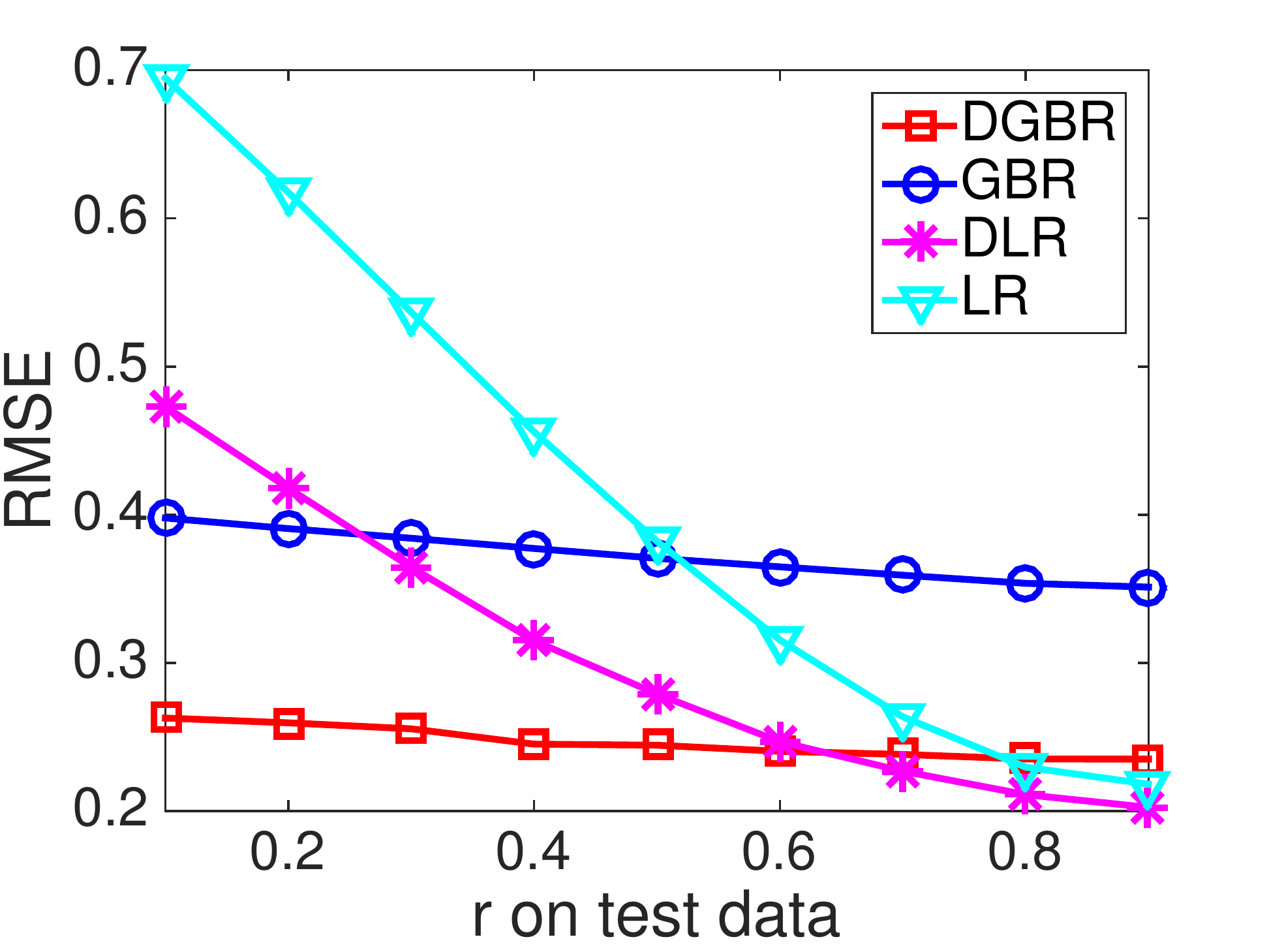}
}
\subfloat[Trained on $n=4000,p=20,r=0.75$\label{fig:RMSE_4000_20_75_s2v}]{
  \includegraphics[width=2.1in]{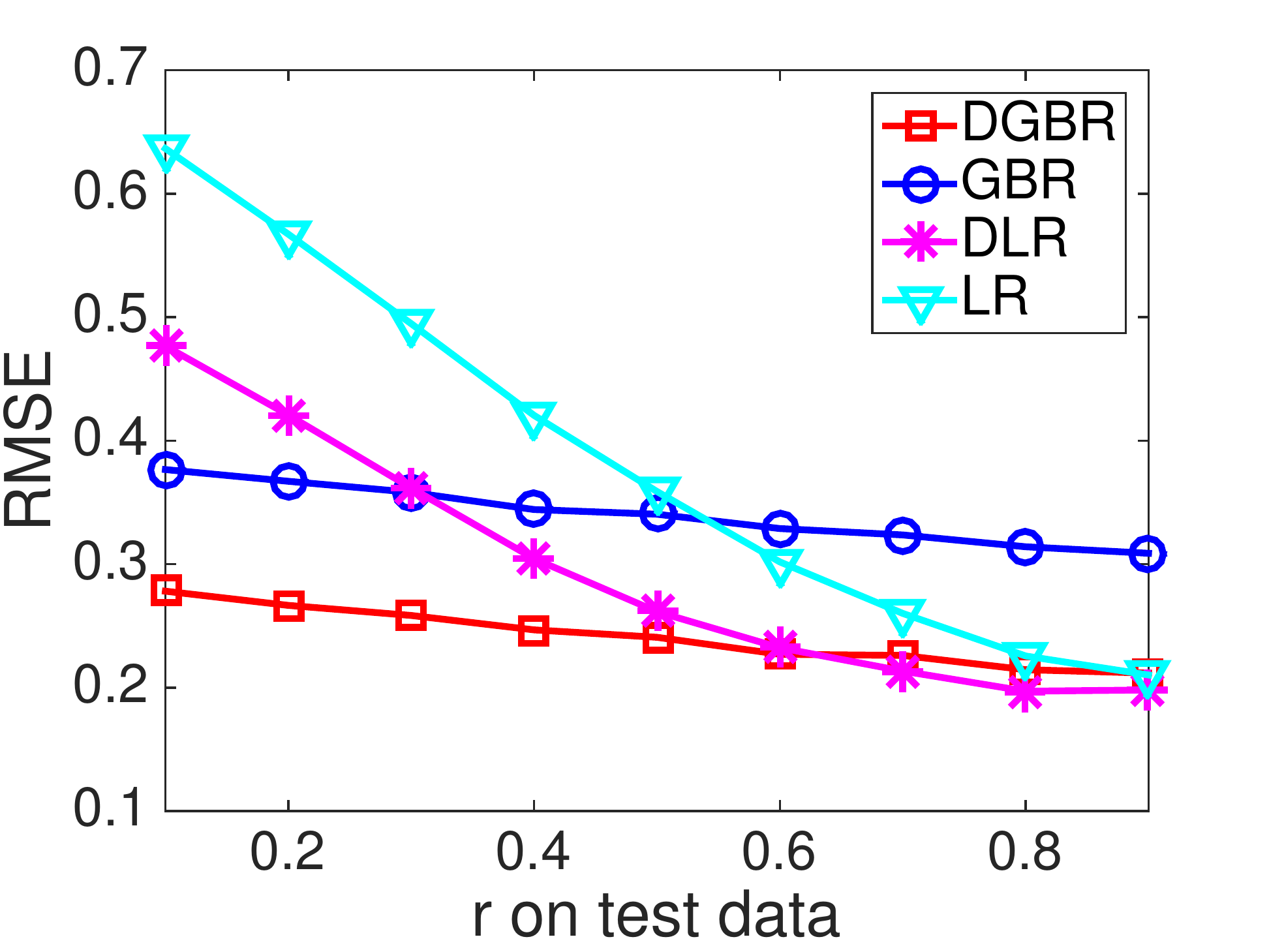}
}
\caption{\textbf{Setting $\mathbf{S}\rightarrow \mathbf{V}$}: RMSE of outcome prediction on various testing datasets by varying sample size $n$ on training dataset.}
\label{fig:simulation_s2v}
\end{figure*}

\begin{figure*}[tb]
\centering
\subfloat[Trained on $n=1000,p=20,r=0.75$ \label{fig:RMSE_1000_20_75_v2s}]{
  \includegraphics[width=2.1in]{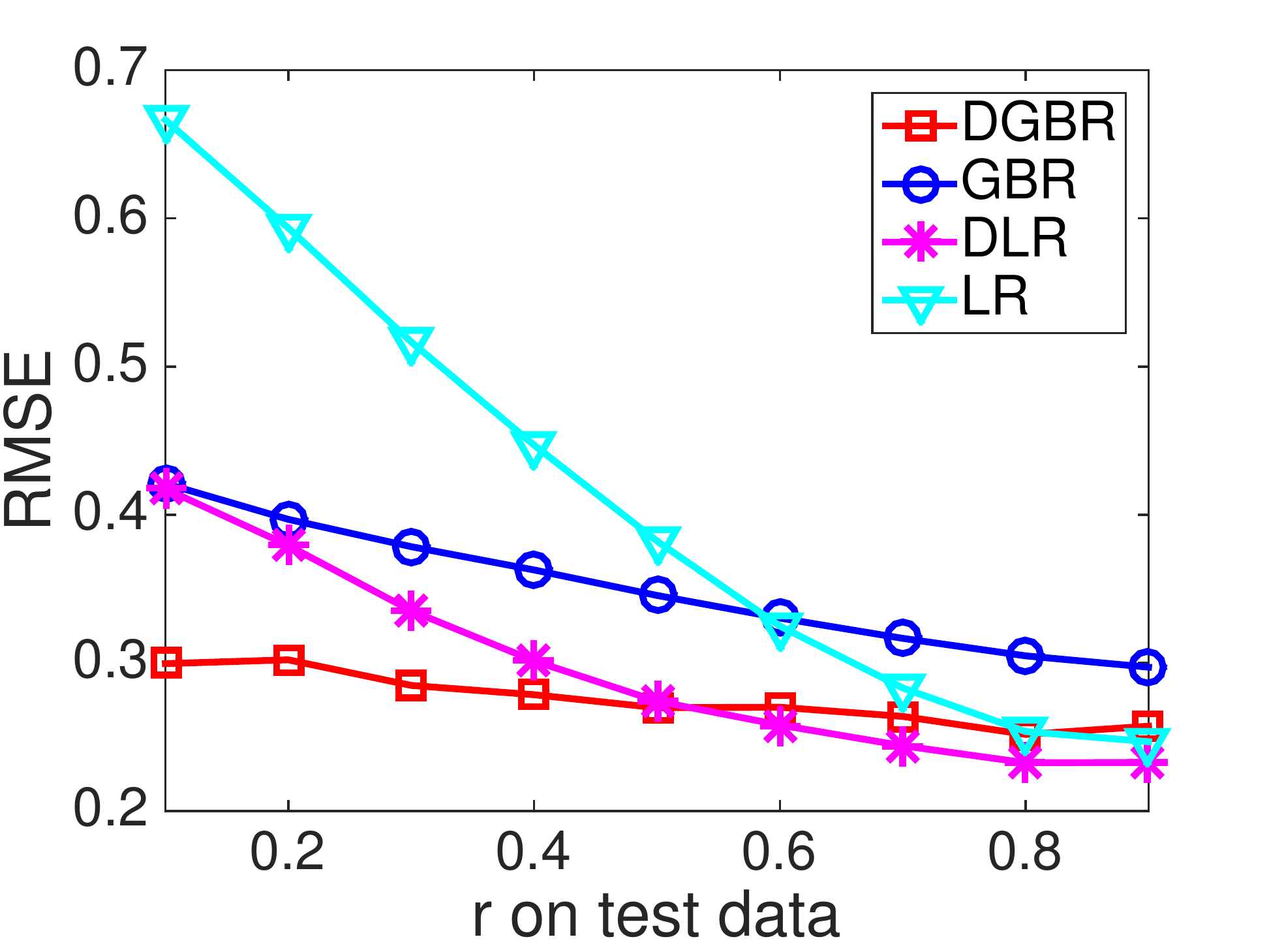}
}
\subfloat[Trained on $n=2000,p=20,r=0.75$\label{fig:RMSE_2000_20_75_v2s}]{
  \includegraphics[width=2.1in]{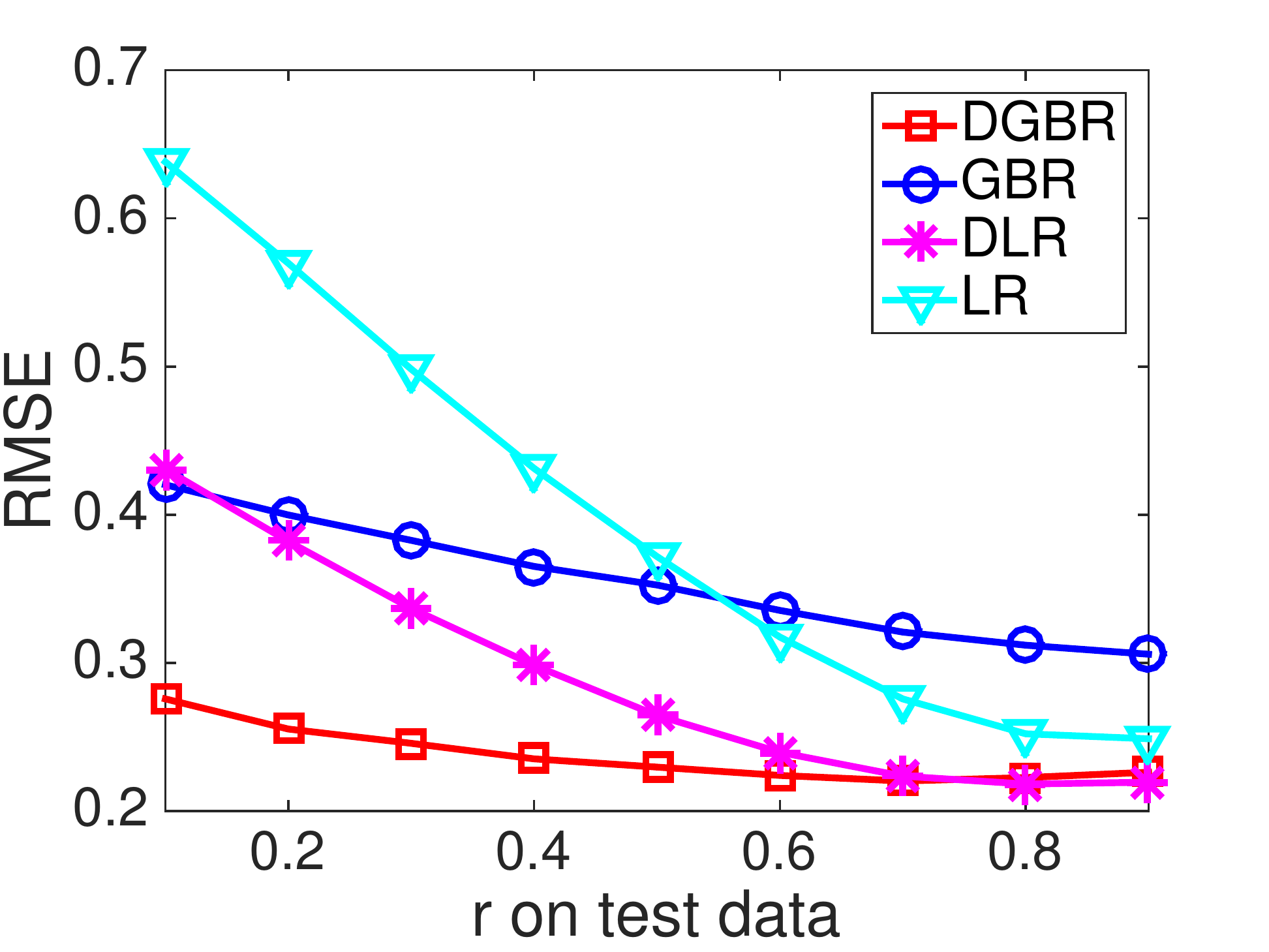}
}
\subfloat[Trained on $n=4000,p=20,r=0.75$\label{fig:RMSE_4000_20_75_v2s}]{
  \includegraphics[width=2.1in]{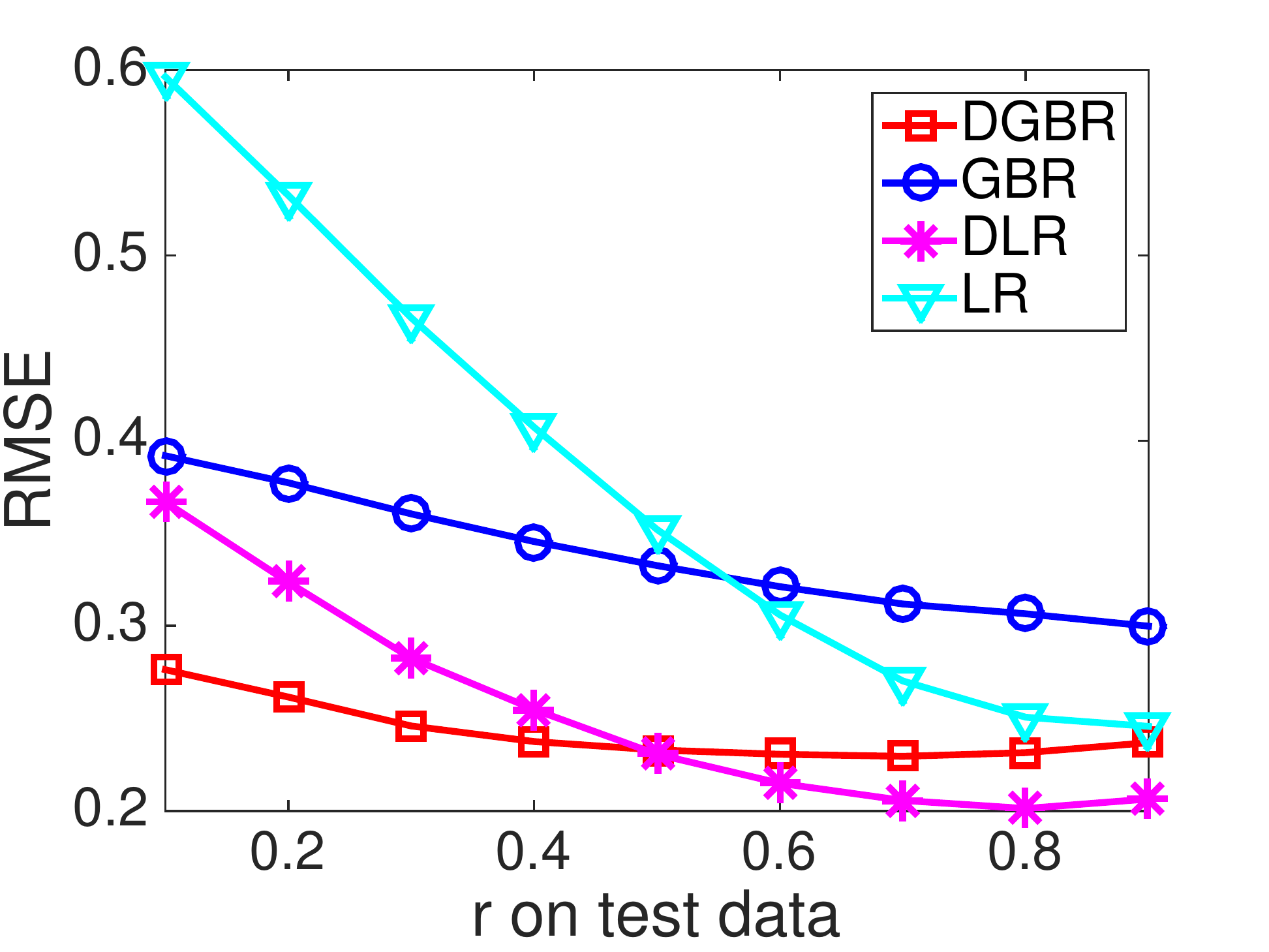}
}
\caption{\textbf{Setting $\mathbf{V}\rightarrow \mathbf{S}$}: RMSE of outcome prediction on various testing datasets by varying sample size $n$ on training dataset.}
\label{fig:simulation_v2s}
\end{figure*}

\subsubsection{Experiments by Varying $P(Y|\mathbf{V})$ and Results}

Specifically, we vary $P(Y|\mathbf{V})$ via biased sample selection with a bias rate $r\in(0,1)$. 
For each sample, we select it with probability $r$ if its noisy features equal to response variable, that is $\mathbf{V}=Y$; otherwise we select it with probability $1-r$, where $r>.5$ corresponds to positive correlation between $Y$ and $\mathbf{V}$.
From the generation of $Y$, we know that, given stable features $\mathbf{S}$, noisy features $\mathbf{V}$ are independent of $Y$. 
But after biased sample selection, $\mathbf{V}$ could be correlated with response variable $Y$ conditional on $\mathbf{S}$ due to selection bias.
However, since $\mathbf{S}$ is an important factor in determining $Y$ and thus whether a unit is selected when its noisy features are high, controlling for $\mathbf{S}$ when estimating the correlation between $Y$ and $\mathbf{V}$ reduces that correlation.  Note that this data-generating environment violates Assumption \ref{asmp:stable} after the sample selection is introduced.  This creates an environment that is challenging for our algorithm, since
we do not have strict guarantees that it will work, and also serves to illustrate that even when the strong assumptions of theory fail,
the algorithm can still lead to substantial improvements.

To comprehensively and systematically evaluate stability of predictive models, we generate different synthetic data by varying sample size $n = \{1000,2000,4000\}$, dimensions of variables $p = \{20,40,80\}$, and bias rate $r = \{0.65,0.75,0.85\}$. 
We report the results of setting $\mathbf{S}\perp \mathbf{V}$ in Figure~\ref{fig:simulation_n_r} $\&$ \ref{fig:simulation_p}, and report the results of setting $\mathbf{S}\rightarrow \mathbf{V}$ and $\mathbf{V}\rightarrow \mathbf{S}$ in Figure~\ref{fig:simulation_s2v} and \ref{fig:simulation_v2s}, respectively.

From the results,  we have following observations and analysis:
\begin{itemize}[leftmargin=0.7cm]
\item \par \noindent The methods LR and DLR can not address the stable prediction problem in all settings. Since they can not remove the spurious correlation between noisy features and the response variable during model training, they often predict large effects of the noisy features, which leads to instability across environments.
\item \par \noindent Comparing with baselines, our method achieves a more stable prediction in different settings. The GBR method is more stable than LR, and our DGBR algorithm is more stable than DLR. The main reason is that the global balancing regularizer used in our models helps ensure accurate estimation of the effect of the stable features, and reduces the estimates of the effect of the noisy features.
\item \par \noindent Our DGBR model makes a more precise and stable prediction than GBR model across environments. 
The deep embedding model in DGBR algorithm makes global balancing weights less noisy and simplifies estimates of the effect of stable features. \hide{\todo{Note: we didn't really show these things e.g. by comparing the sum  of squares of balancing weights or the accuracy of the coefficients--could do in revision}}
\item \par \noindent By varying the sample size $n$, dimension of variables $p$ and training bias rate $r$, the RMSE of our DGBR algorithm is consistently stable and small across environments. Another important observation is that comparing with baselines, our algorithm makes more and more significantly improvement on prediction performance when $n$ is small relative to $p$ and $r$.
\end{itemize}

\begin{figure*}[tb]
\centering
\subfloat[Trained on $n=1000,p=20,r=0.85$ \label{fig:Pvs-RMSE_1000_20_85_s0v}]{
  \includegraphics[width=2.1in]{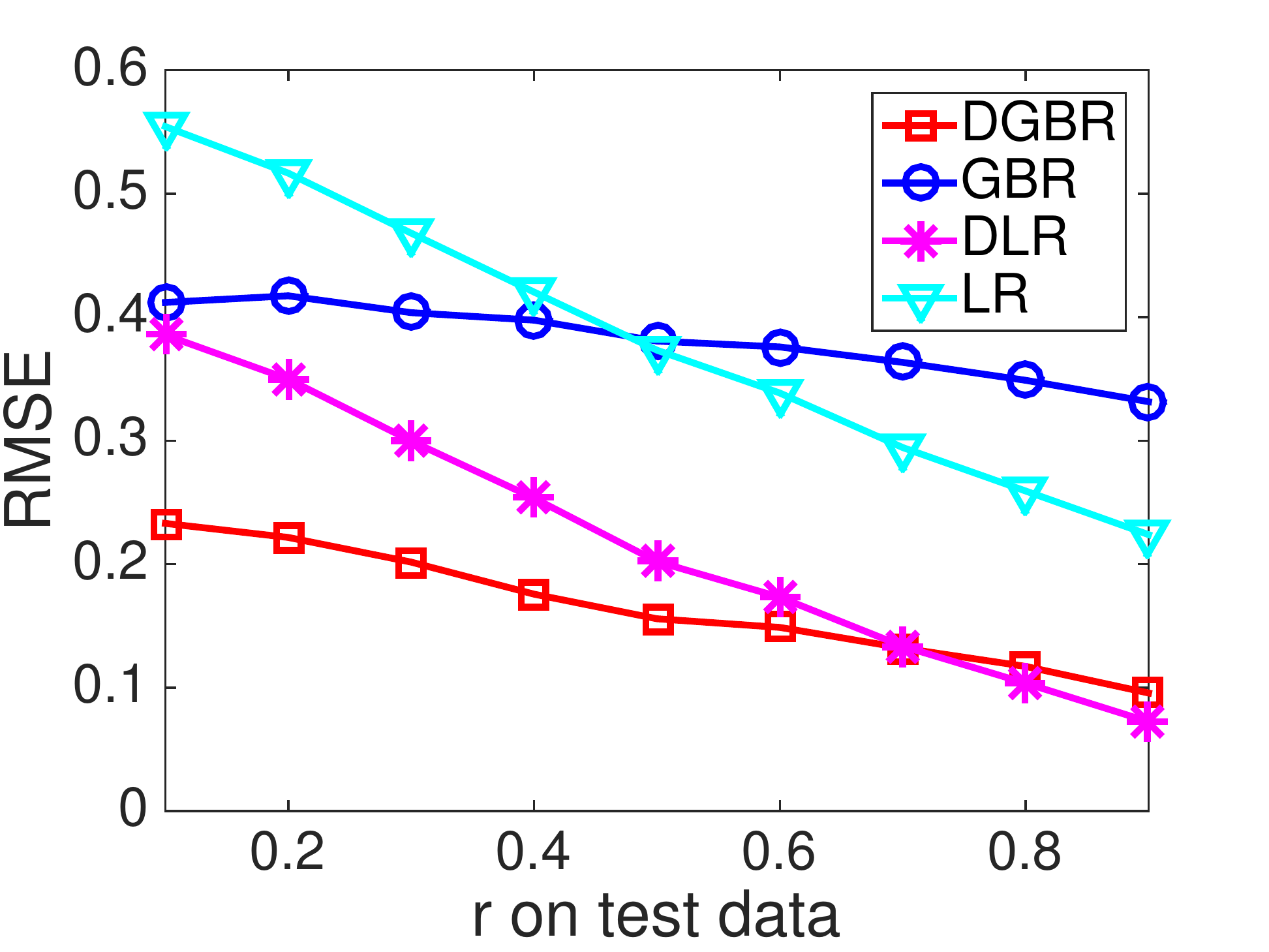}
}
\subfloat[Trained on $n=2000,p=20,r=0.85$\label{fig:Pvs-RMSE_2000_20_85_s0v}]{
  \includegraphics[width=2.1in]{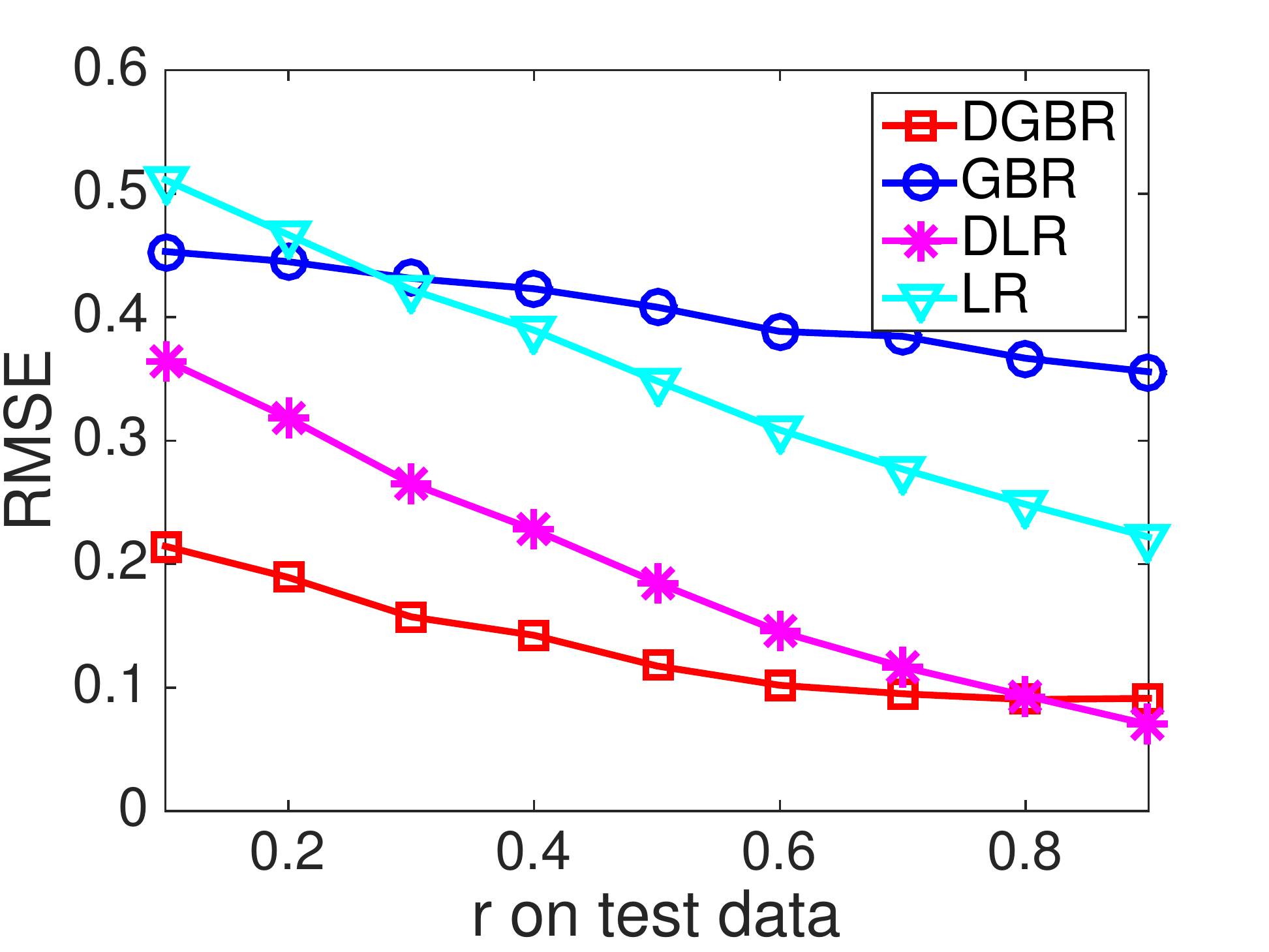}
}
\subfloat[Trained on $n=4000,p=20,r=0.85$\label{fig:Pvs-RMSE_4000_20_85_s0v}]{
  \includegraphics[width=2.1in]{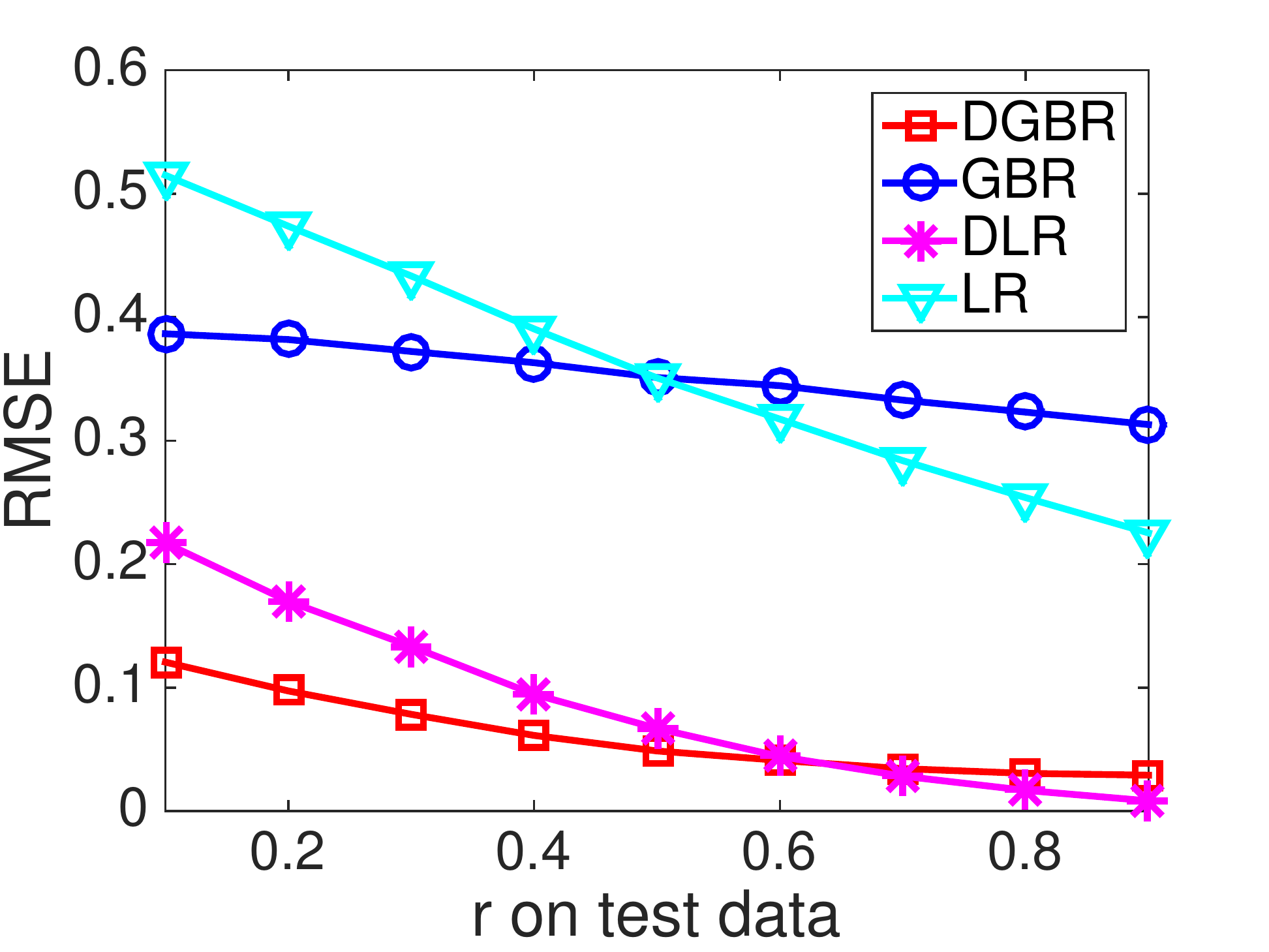}
}
\caption{\textbf{Setting $\mathbf{S}\perp \mathbf{V}$}: RMSE of outcome prediction on various testing datasets by varying sample size $n$ on training dataset.}
\label{fig:Pvs-simulation_s0v}
\end{figure*}

\begin{figure*}[tb]
\centering
\subfloat[Trained on $n=1000,p=20,r=0.85$ \label{fig:Pvs-RMSE_1000_20_85_s2v}]{
  \includegraphics[width=2.1in]{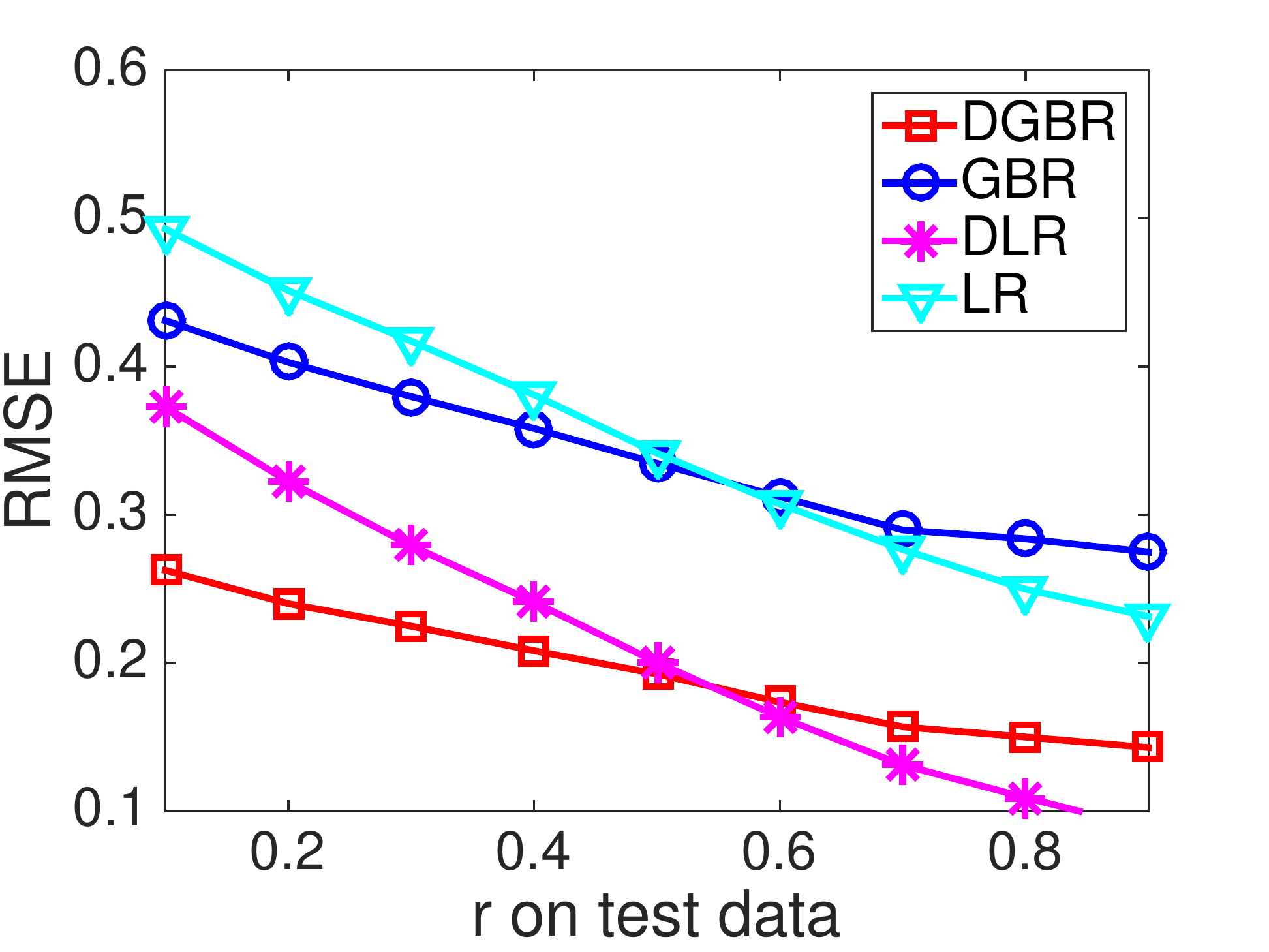}
}
\subfloat[Trained on $n=2000,p=20,r=0.85$\label{fig:Pvs-RMSE_2000_20_85_s2v}]{
  \includegraphics[width=2.1in]{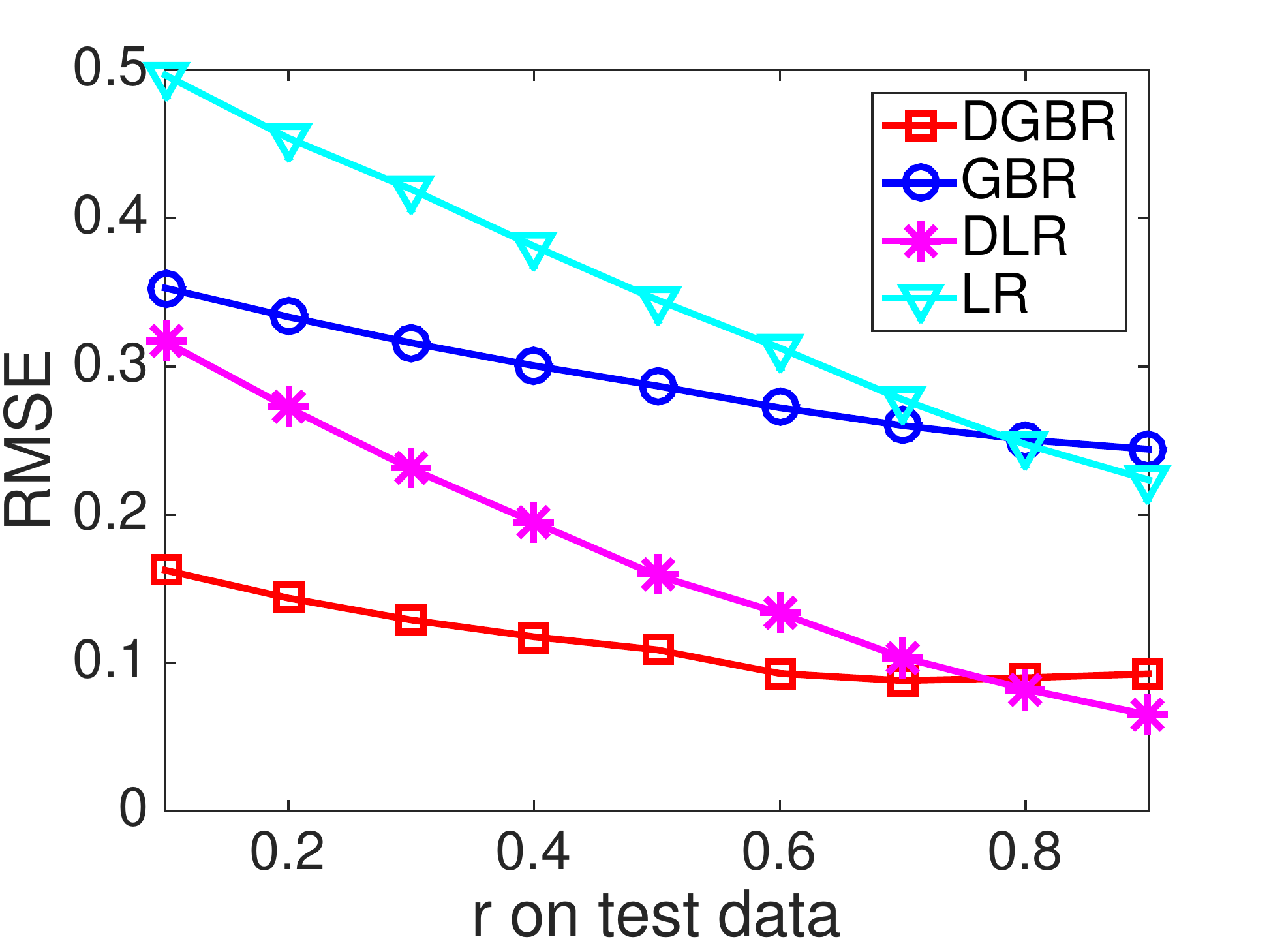}
}
\subfloat[Trained on $n=4000,p=20,r=0.85$\label{fig:Pvs-RMSE_4000_20_85_s2v}]{
  \includegraphics[width=2.1in]{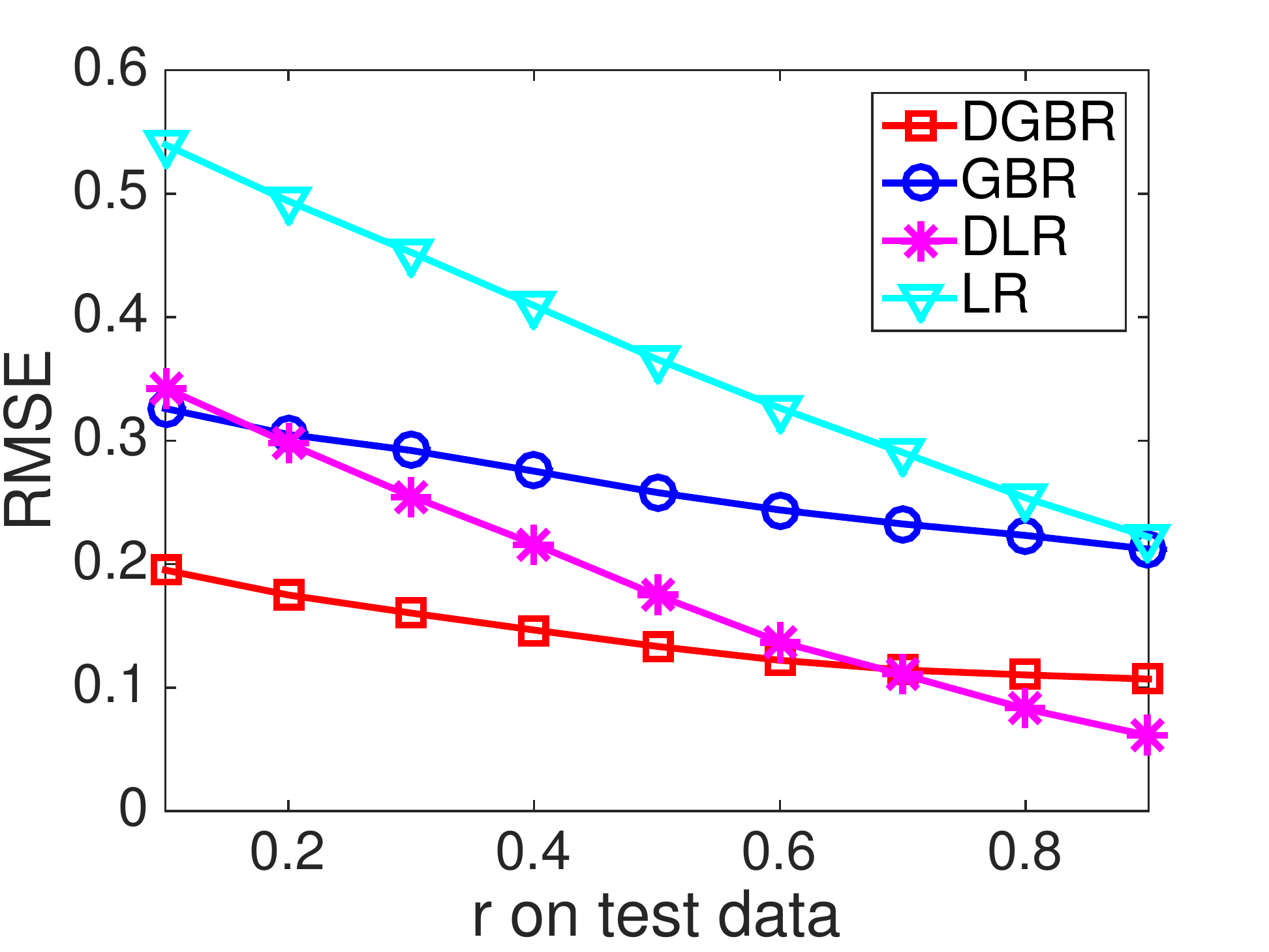}
}
\caption{\textbf{Setting $\mathbf{S}\rightarrow \mathbf{V}$}: RMSE of outcome prediction on various testing datasets by varying sample size $n$ on training dataset.}
\label{fig:Pvs-simulation_s2v}
\end{figure*}

\begin{figure*}[tb]
\centering
\subfloat[Trained on $n=1000,p=20,r=0.85$ \label{fig:Pvs-RMSE_1000_20_85_v2s}]{
  \includegraphics[width=2.1in]{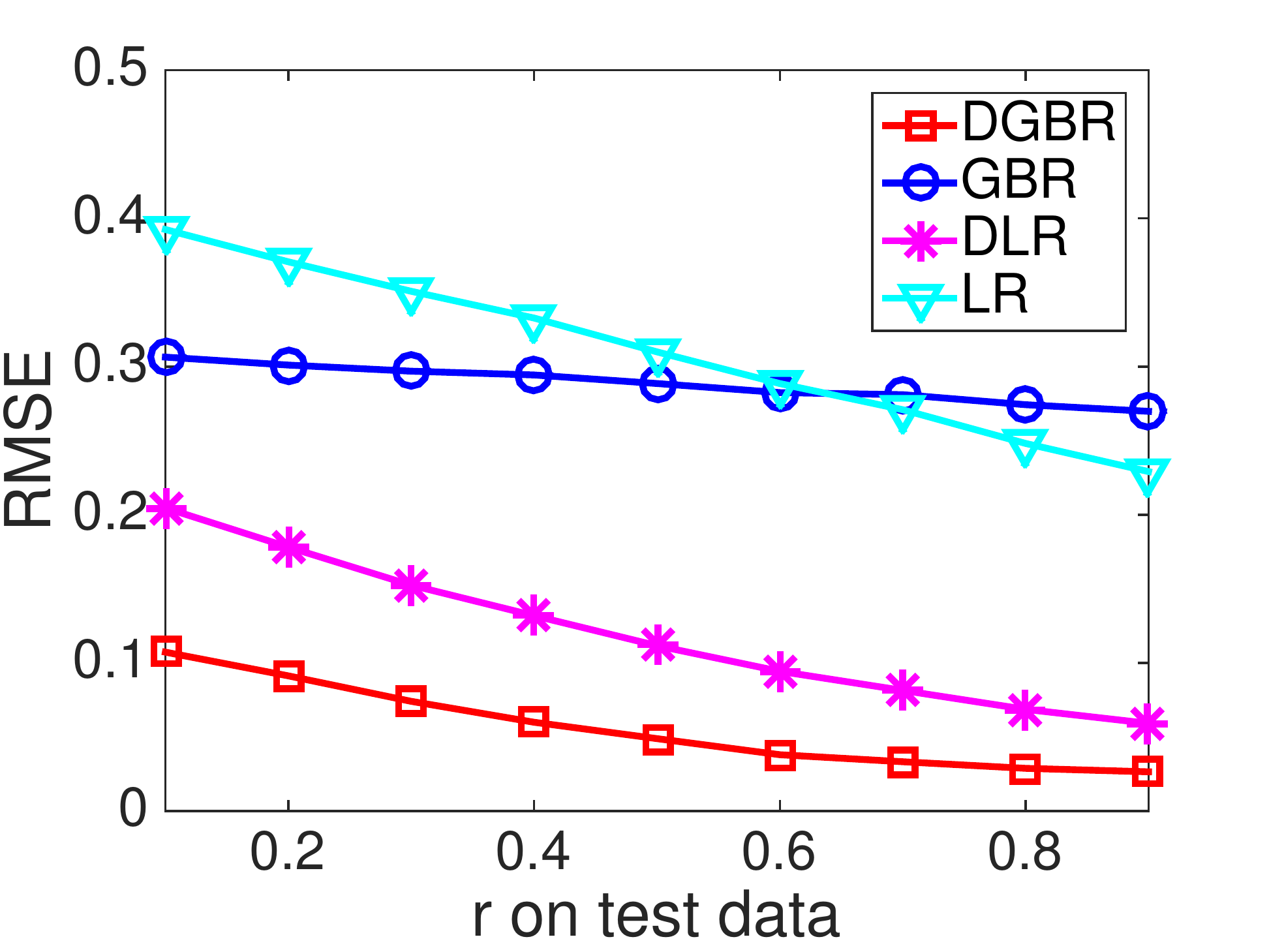}
}
\subfloat[Trained on $n=2000,p=20,r=0.85$\label{fig:Pvs-RMSE_2000_20_85_v2s}]{
  \includegraphics[width=2.1in]{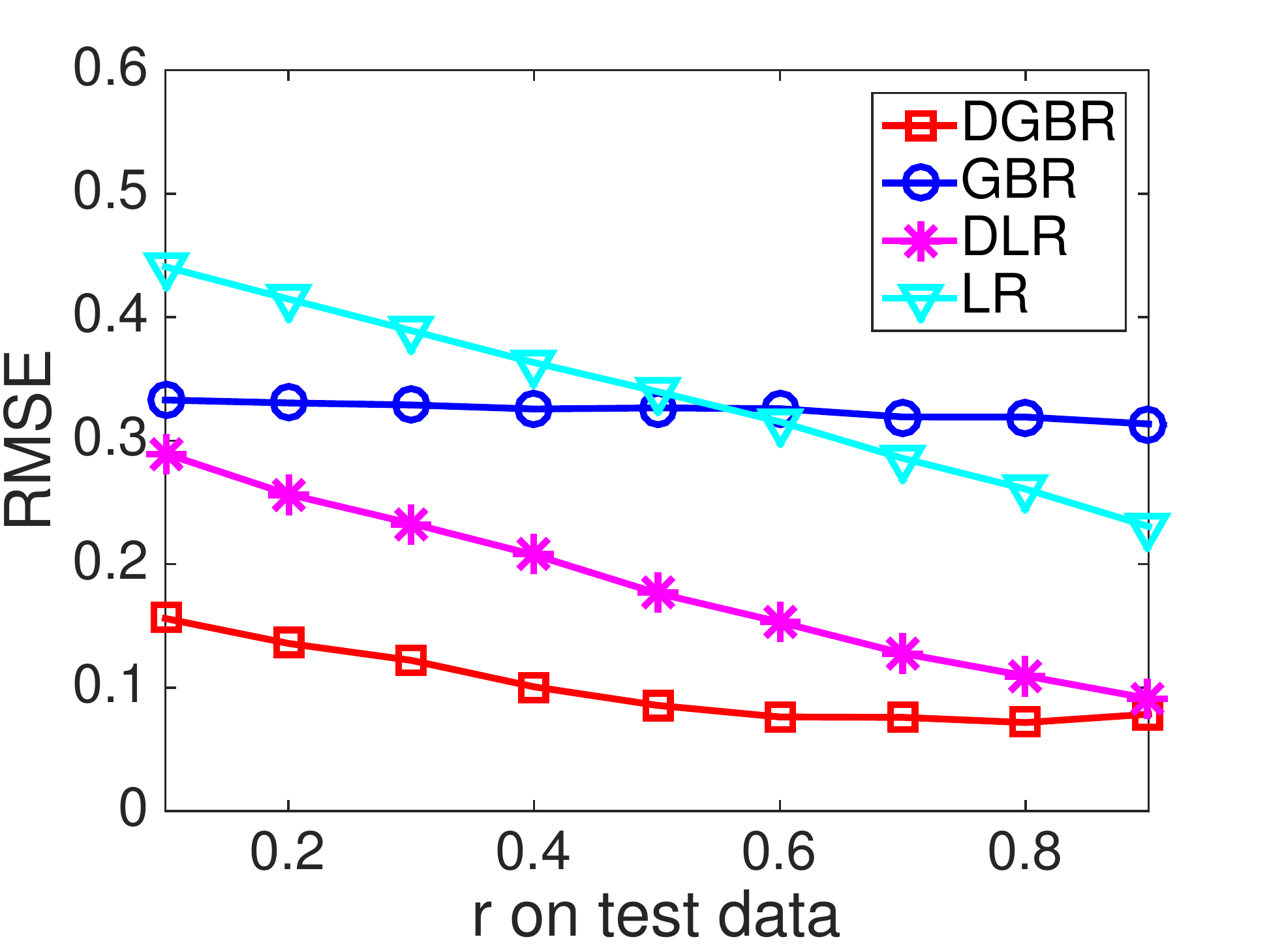}
}
\subfloat[Trained on $n=4000,p=20,r=0.85$\label{fig:Pvs-RMSE_4000_20_85_v2s}]{
  \includegraphics[width=2.1in]{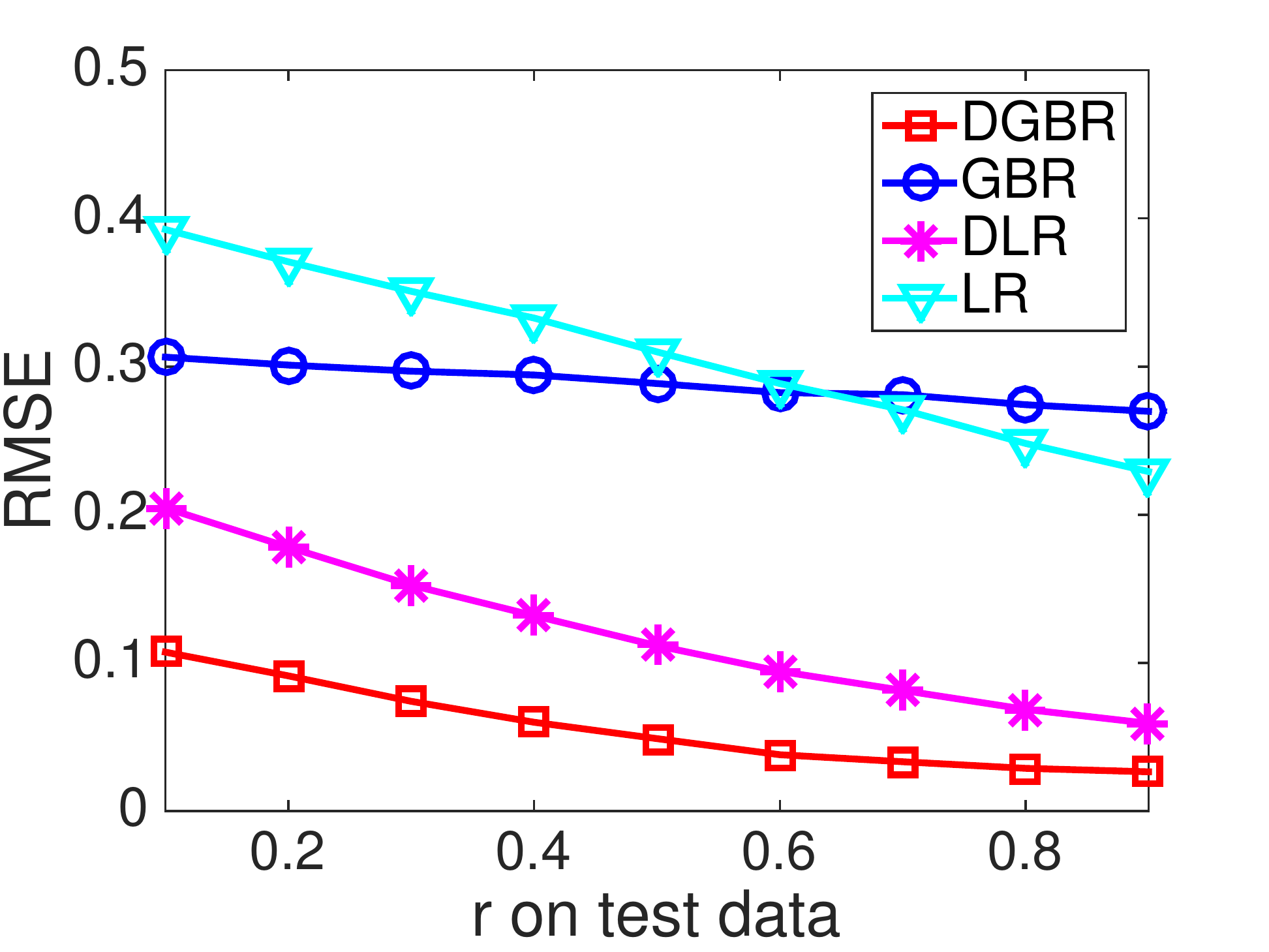}
}
\caption{\textbf{Setting $\mathbf{V}\rightarrow \mathbf{S}$}: RMSE of outcome prediction on various testing datasets by varying sample size $n$ on training dataset.}
\label{fig:Pvs-simulation_v2s}
\end{figure*}

\subsubsection{Experiments by Varying $P(\mathbf{V}|\mathbf{S})$ and Results}

In this sub experiment, we generate the response variable $Y$ with the function $g$ as following:
\begin{eqnarray}
\nonumber Y\!\!\!&=&\!\!\!\Scale[0.88]{1/(1+\exp(-\sum_{\mathbf{X}_{\cdot,i}\in \mathbf{S}_l}\alpha_i\cdot \mathbf{X}_{\cdot,i}-\sum_{\mathbf{X}_{\cdot,j}\in \mathbf{S}_n}\beta_j\cdot \mathbf{X}_{\cdot,j}\cdot \mathbf{X}_{\cdot,j+1}))},
\end{eqnarray}
where $\alpha_i = (-1)^{i}$ and $\beta_j = p/2$.  And to make $Y$ binary, we set $Y=1$ when $Y\geq0.5$, otherwise $Y=0$.

Here, we vary $P(\mathbf{V}|\mathbf{S})$ also via biased sample selection with a bias rate $r\in(0,1)$. 
Specifically, for each sample, we select it with probability $r$ if its noisy features equal to a mediate variable $\mathbf{Z}$, that is $\mathbf{V}_i=\mathbf{Z}_i$; otherwise we select it with probability $1-r$, where $\mathbf{Z}_i = \sum_{j=i}^{i+5}(-1)^{j}\cdot \mathbf{S}_j$, and $r>.5$ corresponds to positive correlation between $\mathbf{Z}_i$ and $\mathbf{V}_i$.
The same, from the generation of $Y$, we know that, given stable features $\mathbf{S}$, noisy features $\mathbf{V}$ are independent of $Y$. 
After biased sample selection, $\mathbf{V}$ could be highly correlated with response variable $Y$, but it is still independent with $Y$ conditional on stable features $\mathbf{S}$.
Thus, the Assumption \ref{asmp:stable} is valid under this setting.
Therefore, with identifying stable features $\mathbf{S}$, our algorithm can make a stable prediction across environments.

\hide{
However, since $\mathbf{S}$ is an important factor in determining $Y$ and thus whether a unit is selected when its noisy features are high, controlling for $\mathbf{S}$ when estimating the correlation between $Y$ and $\mathbf{V}$ reduces that correlation. 
}

In this part experiments, we generate different synthetic data by varying sample size $n=\{1000,2000,4000\}$.
We report the experimental results under settings $\mathbf{S}\perp \mathbf{V}$, $\mathbf{S}\rightarrow \mathbf{V}$, and $\mathbf{V}\rightarrow \mathbf{S}$ in Figure~\ref{fig:Pvs-simulation_s0v}, \ref{fig:Pvs-simulation_s2v} $\&$ \ref{fig:Pvs-simulation_v2s} , respectively. 
From these results, we can obtain the same observations that (i) The traditional classification methods LR and DLR can not address the stable prediction problem in all settings, (ii) Comparing with baselines, our method achieves a more stable prediction in different settings. The GBR method is more stable than LR, and our DGBR algorithm is more stable than DLR, and (iii) Our DGBR model makes a more precise and stable prediction than GBR model across environments.

\subsubsection{Visualization of Embedded Features}

\begin{figure}[t]
\centering
\includegraphics[width=2.4in]{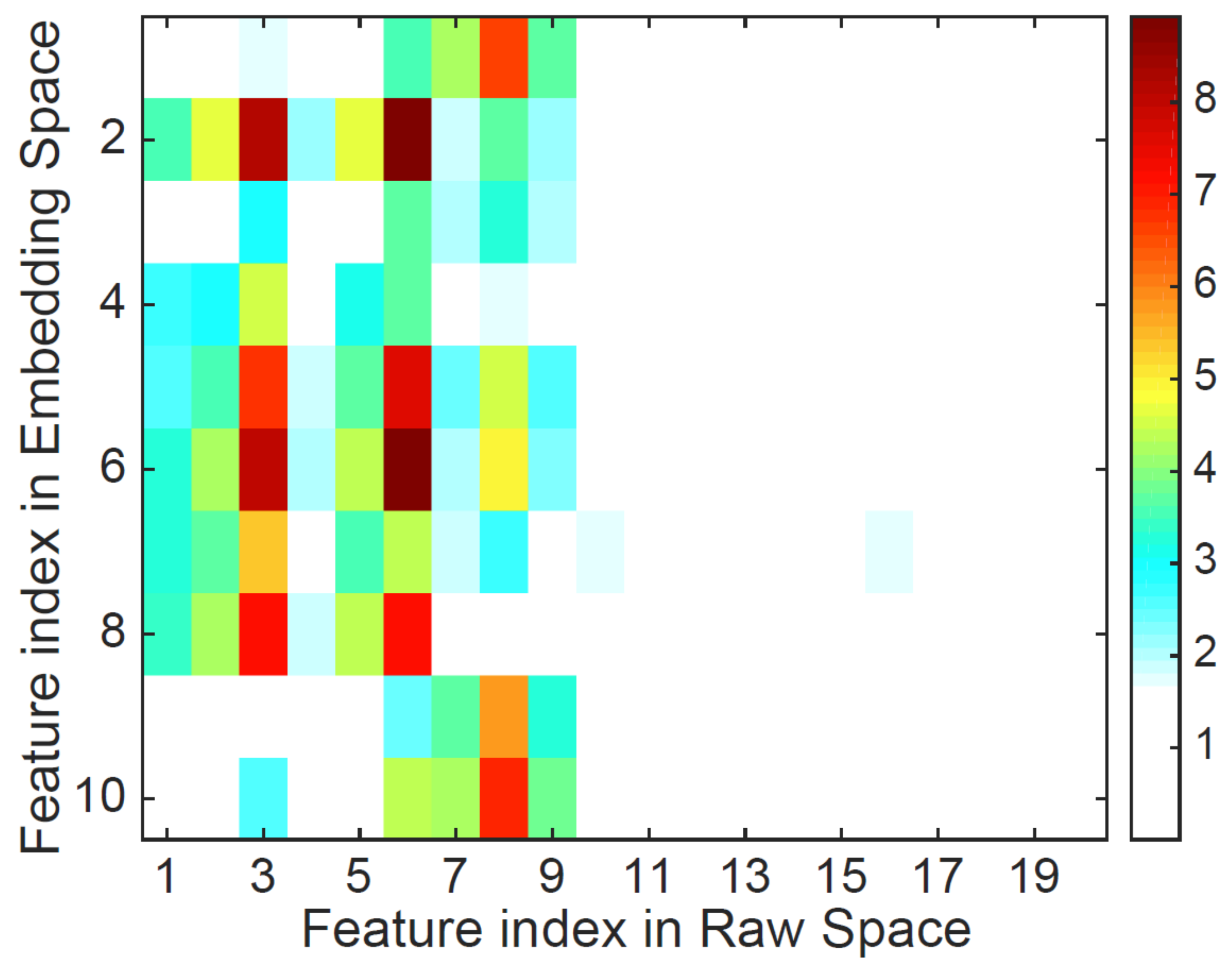}
\caption{Embedding weights in our DGBR algorithm, where $\mathbf{X}_{\cdot,1}, \cdots, \mathbf{X}_{\cdot,9}$ are stable features $\mathbf{S}$ and others are noisy features $\mathbf{V}$. It illustrates that our DGBR can achieve $Y$ and $\mathbf{V}$ are independent, since the features in embedding space have few information of noisy feature $\mathbf{V}$ in raw sapce.}
\label{fig:raw2embedding}
\end{figure}

In Figure~\ref{fig:raw2embedding}, we also show that the embedded features in our DGBR algorithm have few information of noisy features $\mathbf{V}$ from raw space. This demonstrates that our DGBR could approximately preserve the independence between $Y$ and $\mathbf{V}$ of global balancing, thus can identify stable features and make a stable prediction across unknown environments.

\subsubsection{Parameter Analysis}

\begin{figure*}[tb]
\centering
\subfloat[\label{fig:PA_balancing}]{
  \includegraphics[width=2.1in]{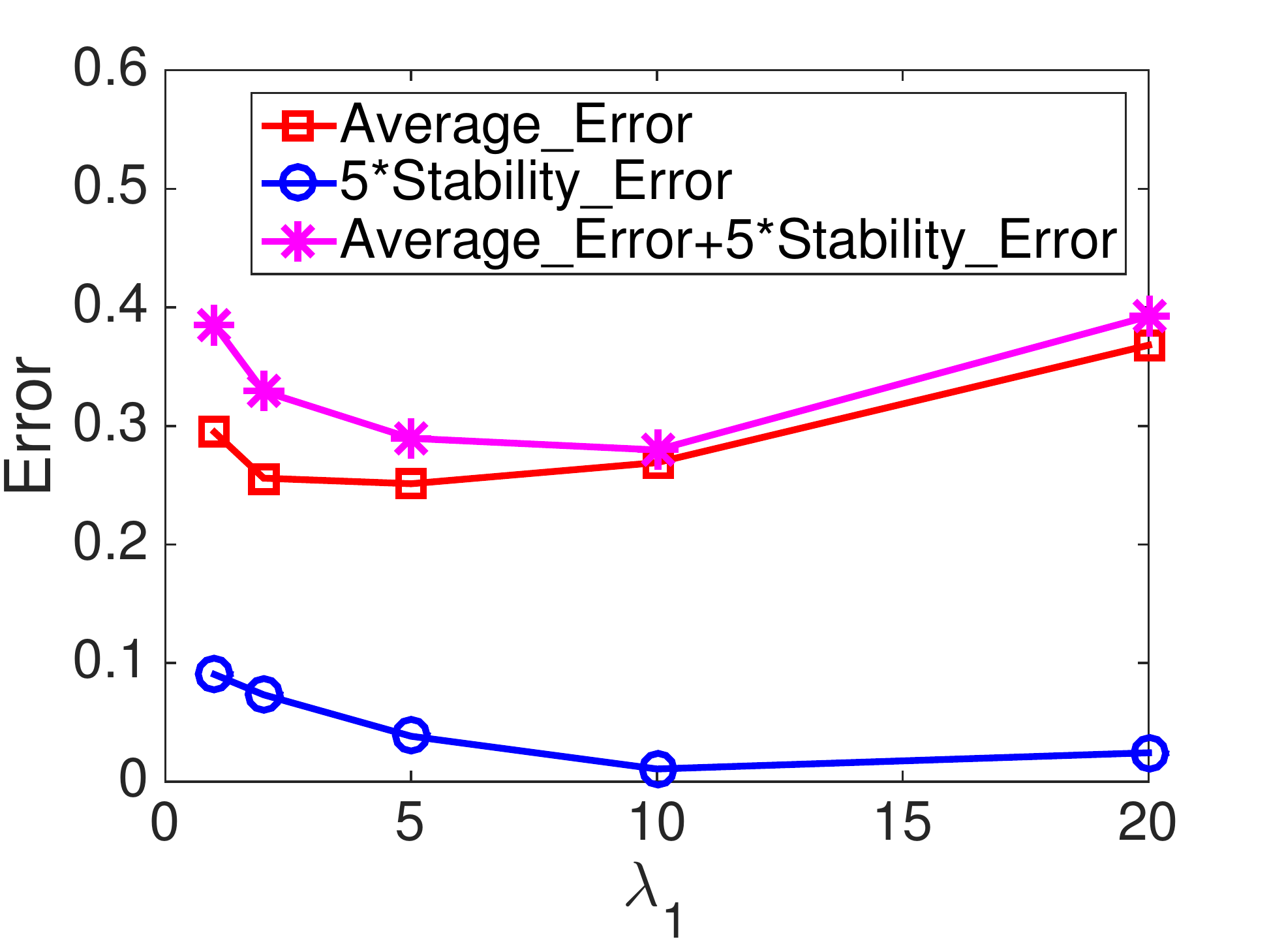}
}
\subfloat[\label{fig:PA_autoencoder}]{
  \includegraphics[width=2.1in]{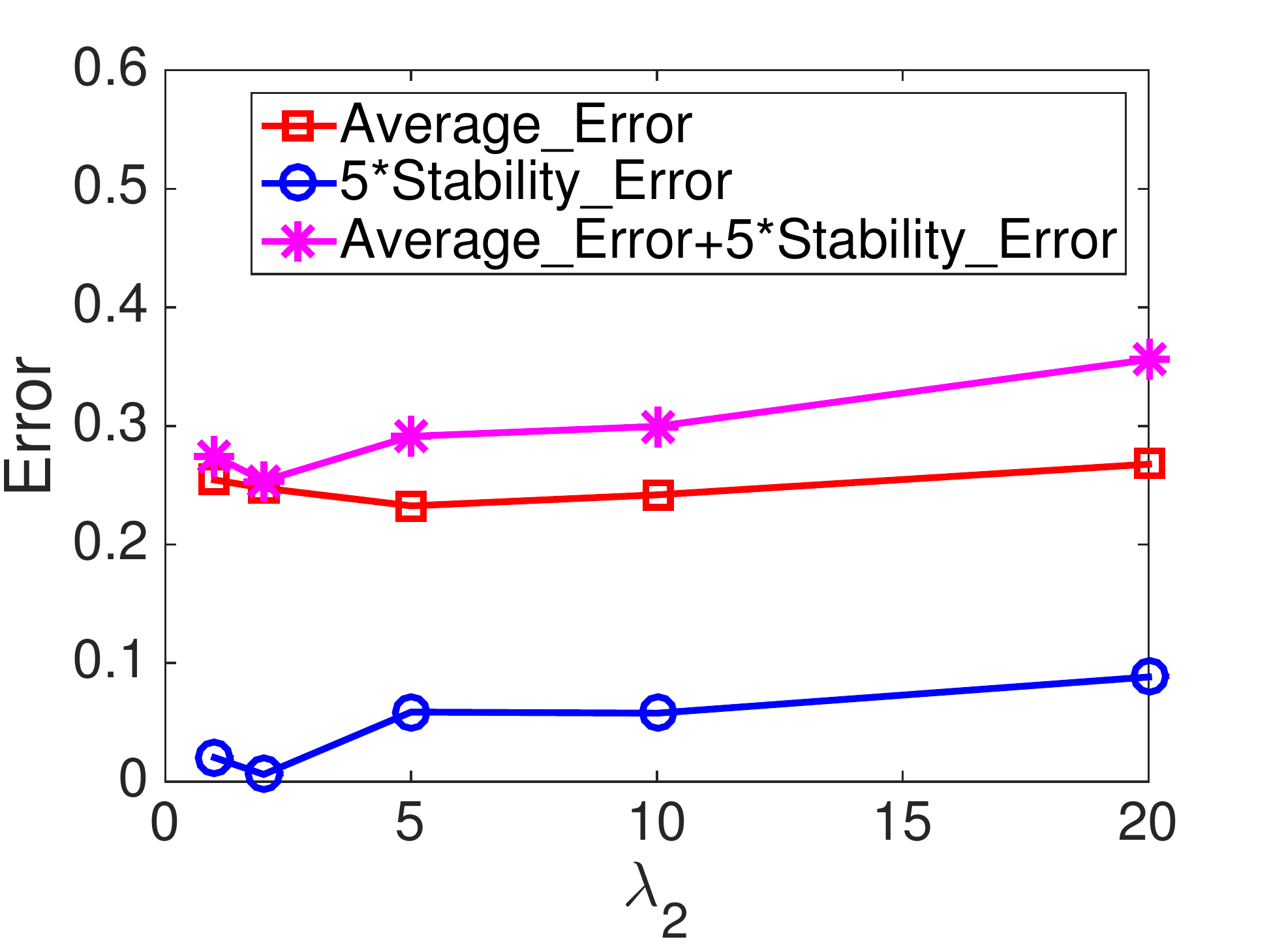}
}
\subfloat[\label{fig:PA_weights}]{
  \includegraphics[width=2.1in]{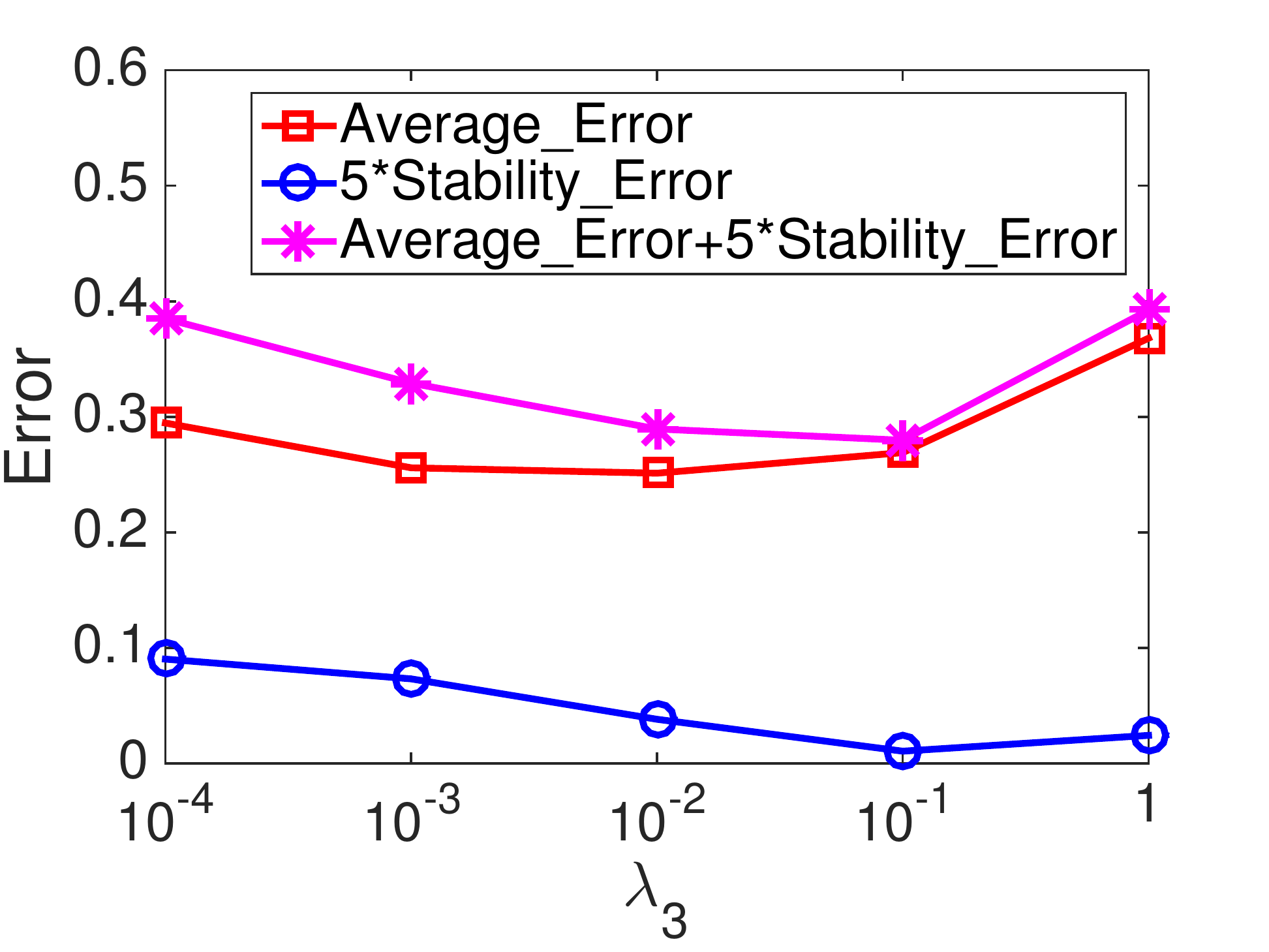}
}
\caption{The effect of hyper-parameters $\lambda_1$, $\lambda_2$, and $\lambda_3$.}
\label{fig:Parameter_Analysis}
\end{figure*}

In our DGBR algorithm, we have some hyper-parameters, such as $\lambda_1$ for constraining the error of global balancing, $\lambda_2$ constraining the loss of auto-encoder term, $\lambda_3$ constraining the variance of the global sample weights, and so on. 
In this section, we investigate how these hyper-parameters affect the results.
We tuned these parameters in our experiments with cross validation by grid searching, based on our constructed validation data.
We report the $Average\_Error$, $5*Stability\_Error$, and $Average\_Error+5*Stability\_Error$ on a synthetic dataset under setting $\mathbf{S}\perp \mathbf{V}$ with $n=2000$ and $p=20$.

\textbf{Tradeoffs between prediction and covariate balancing:} We first show how the hyper-parameter $\lambda_1$ affects the performance in Figure~\ref{fig:PA_balancing}.
The parameter of $\lambda_1$ restrain the error of global balancing. 
We can see that initially the value of both $Average\_Error$ and $Stability\_Error$ decreases when the value of $\lambda_1$ increases. This is intuitive as the data could be more balanced with the increased value of $\lambda_1$, and balanced data could help to identify stable features and remove some noise for more precise prediction. However, when the value of $\lambda_1$ increases further, the value of $Stability\_Error$ decreases, but the value of $Average\_Error$ starts to increase slowly. Large value of $\lambda_1$ makes the algorithm concentrate on global balancing component at the expense of the prediction component.
Both prediction and global balancing components are essential for stable prediction.

\textbf{Feature representation:}
Here, we show how the hyper-parameter $\lambda_2$ affects the results in Figure~\ref{fig:PA_autoencoder}. The value of $Average\_Error$ decreases with $\lambda_2$, since a high value of $\lambda_2$ leads to more accurate prediction. Initially, $Stability\_Error$ decreases with $\lambda_2$, but it starts to increase when $\lambda_2 \geq 5$. 
It is important to choose an appropriate value of $\lambda_2$ for learning feature representation, but our method is not very sensitive to this parameter.

\textbf{The variance of global sample weights:}
Figure \ref{fig:PA_weights} shows how the value of $\lambda_3$ affect performance.
Both the value of $Average\_Error$ and $Stability\_Error$ decrease when the value of $\lambda_3$ increases, since appropriate constraints on the variance of global sample weights could prevent some samples from becoming dominate in whole data, and thus help to improve the precision and robustness of prediction. However, when the value of $\lambda_3$ grows too large, those errors increase. Too large value of $\lambda_3$ could lead the learned global sample weight to fail to make appropriate tradeoffs between balancing and prediction.  

\subsection{Experiments on Real World Data}

\subsubsection{Online Advertising Dataset}

The real online advertising dataset we used is collected from Tencecnt WeChat App\footnote{http://www.wechat.com/en/} during September 2015.
In WeChat, each user can share (receive) posts to (from) his/her friends as like the Twitter and Facebook.
Then the advertisers could push their advertisements to users, by merging them into the list of the user's wallposts.
For each advertisement, there are two types of feedbacks: ``Like'' and ``Dislike''.
When the user clicks the ``Like'' button, his/her friends will receive the advertisements with this action.

The online advertising campaign used in our paper is about the LONGCHAMP handbags
for young women.\footnote{http://en.longchamp.com/en/womens-bags}
This campaign contains 14,891 user feedbacks with Like and 93,108 Dislikes.
For each user, we have their features including (1) demographic attributes, such as age, gender,
(2) number of friends, (3) device (iOS or Android), and (4) the user settings on WeChat, for example, whether allowing strangers
to see his/her album  and whether installing the online payment service.

\par \textbf{Experimental Settings.} In our experiments, we set $Y_i=1$ when user $i$ likes the ad, otherwise $Y_i=0$.
For non-binary user features, we dichotomize them around their mean value.
Considering the overlap assumption in assumption \ref{asmp:overlap}, we only preserve users' features which satisfied $0.2 \leq\frac{\#\{x=1\}}{\#\{x=1\}+\#\{x=0\}} \leq 0.8$.
All the predictors and response variable in our experiment are binary.
\hide{
Finally, our experimental dataset contains 19 user features as predictor variables and user feedback as outcome variable, all of them are binary.
}

\begin{figure}[t]
\centering
\subfloat[RMSE\label{fig:RMSE_ad}]{
  \includegraphics[width=1.8in]{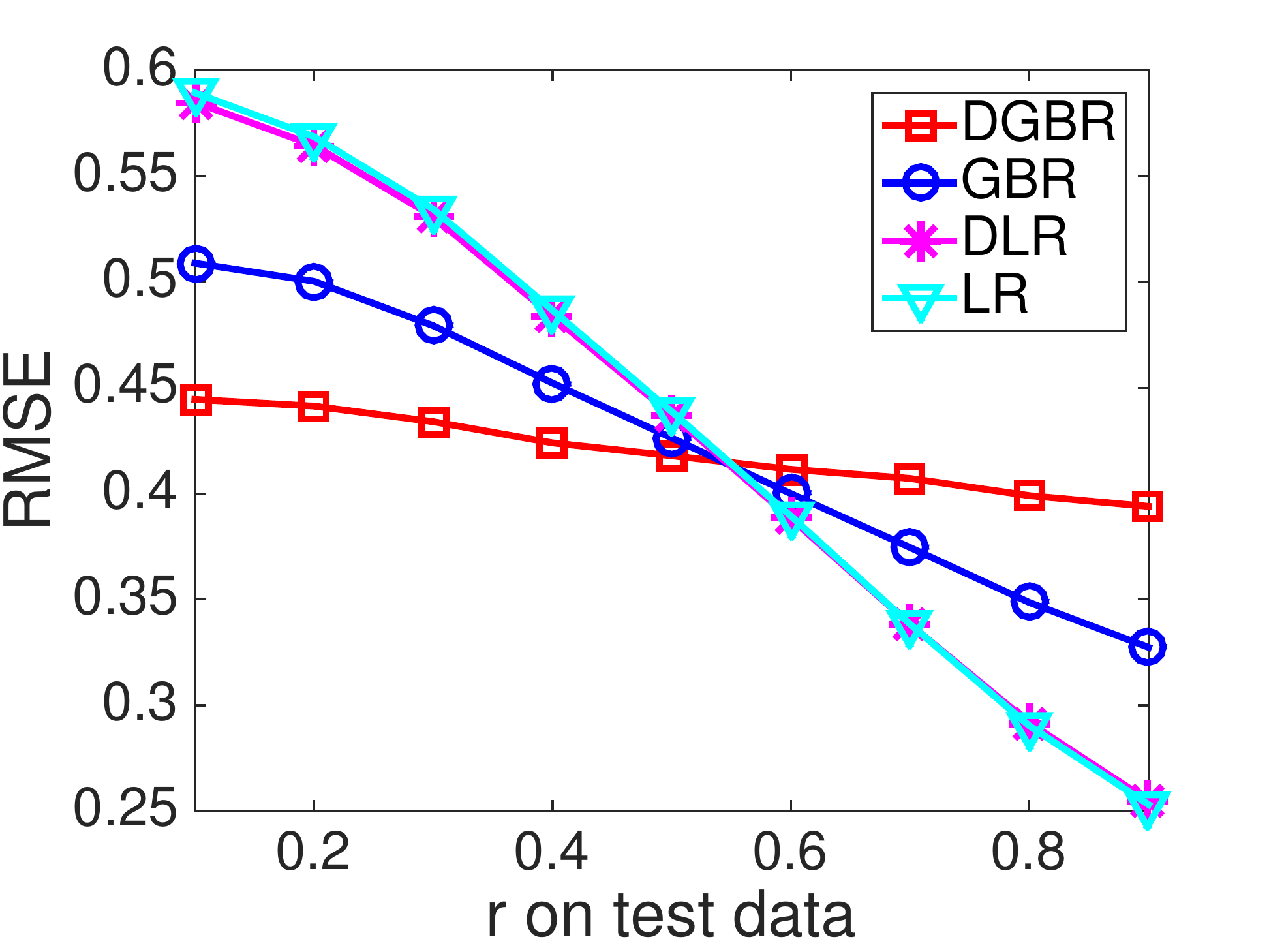}
}
\subfloat[Average\_Error \& Stability\_Error\label{fig:Accuracy_Stability}]{
  \includegraphics[width=1.8in]{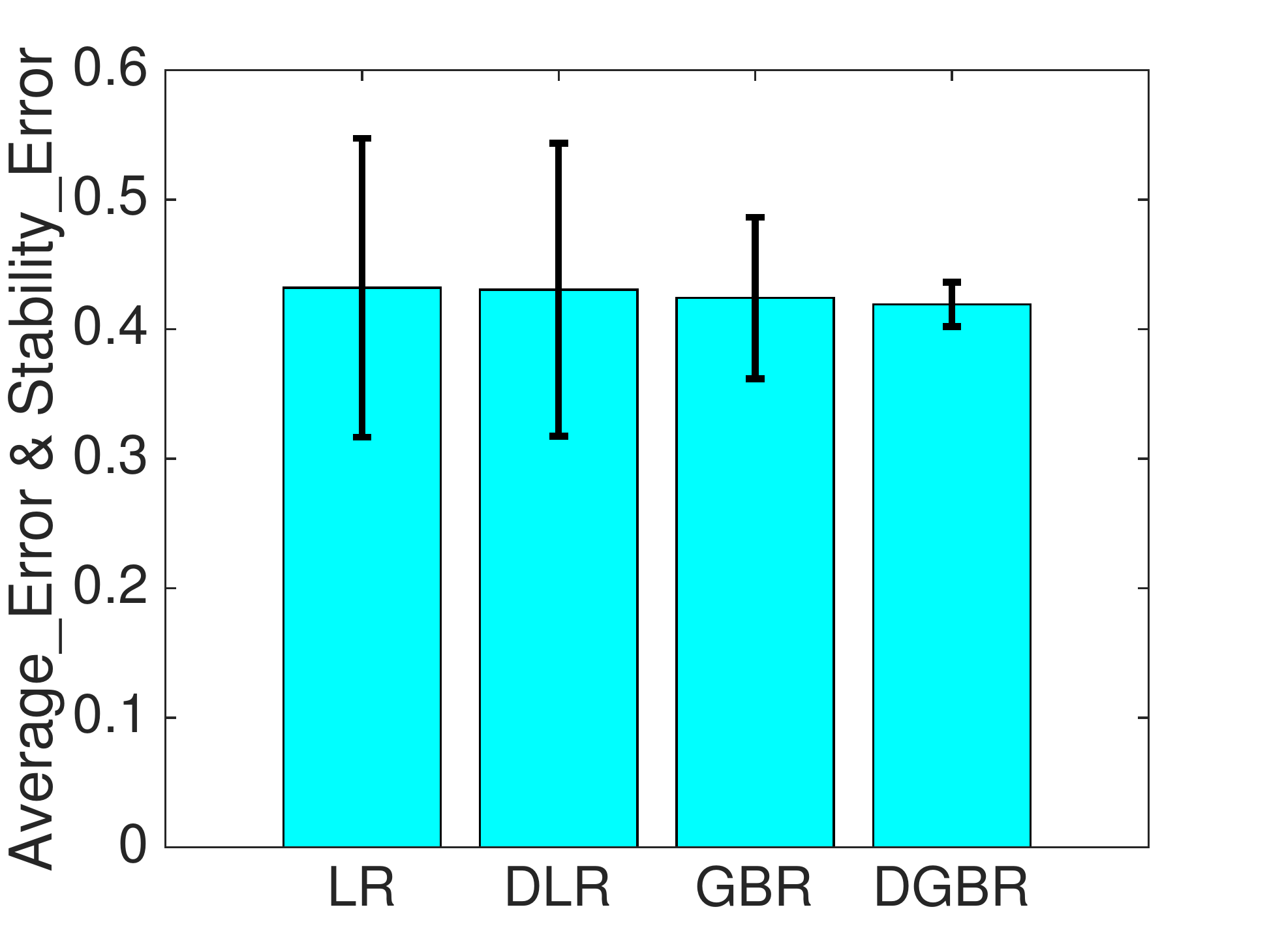}
}
\caption{Our proposed DGBR algorithm makes the most stable prediction on whether user will like or dislike an advertisement.}
\label{fig:prediction_vary_nf}
\end{figure}

\begin{figure}[t]
\centering
\subfloat[Predictor  $mail\ plugin$\label{fig:RMSE_us_14}]{
  \includegraphics[width=1.8in]{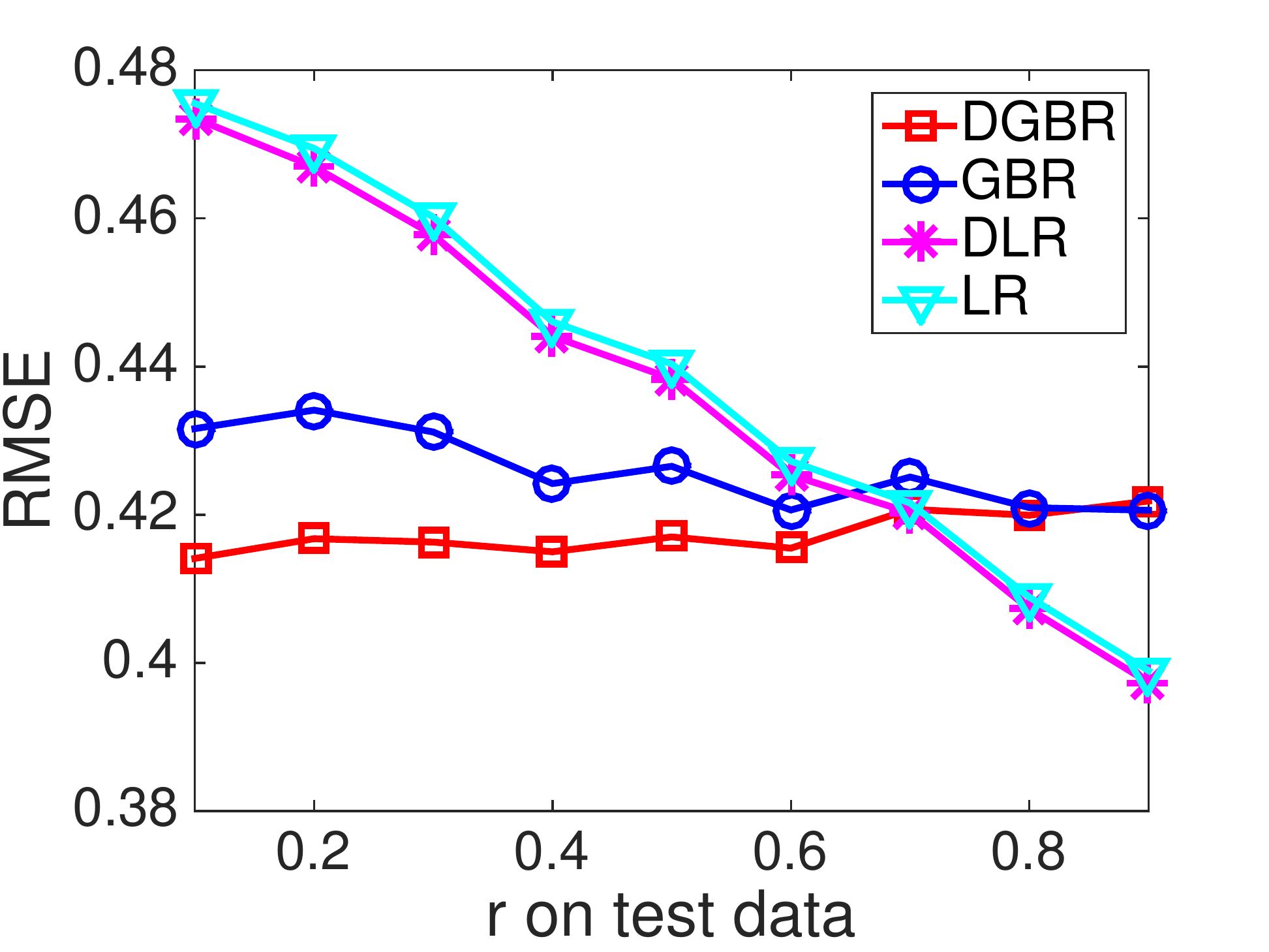}
}
\subfloat[Predictor $bottle\ plugin$\label{fig:RMSE_us_19}]{
  \includegraphics[width=1.8in]{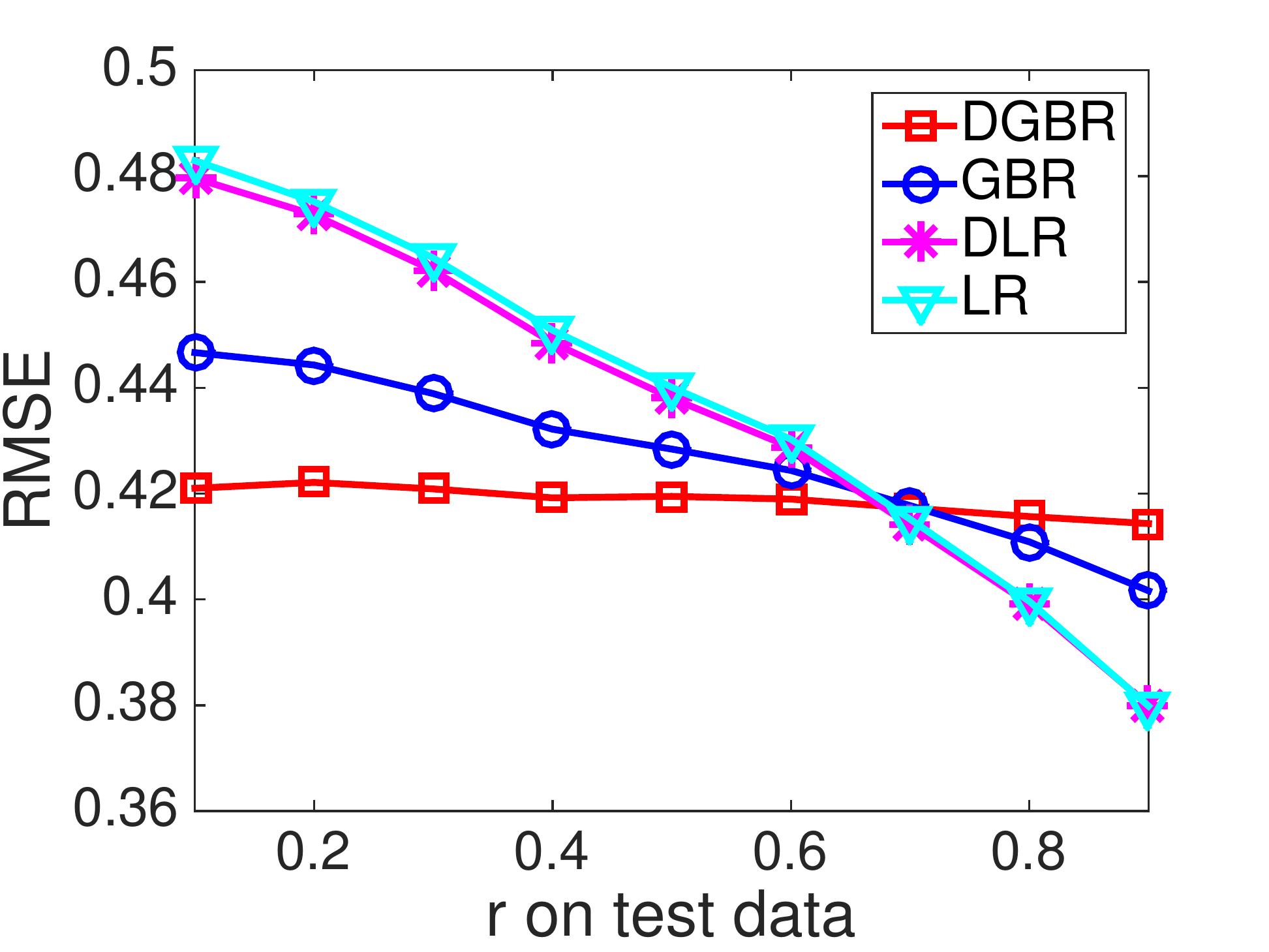}
}
\caption{RMSE of outcome prediction by varying bias rate $r$ between one predictor and outcome.}
\label{fig:prediction_vary_one_nf}
\end{figure}

In order to test the performance of our proposed model, we execute the experiments with two different settings. 
The first experimental setting is similar with the setting on synthetic dataset.
We generate different environments by biased sample selection via bias rate $r$.
In this setting, we choose those features which have no associations with outcome as noisy features for biased sample selection.
In second experimental setting, we generate the various environments by dataset separation with users' feature. Specifically, we separate the whole dataset into 4 parts by users' age, including $Age \in [20,30)$, $Age \in [30,40)$, $Age \in [40,50)$ and $Age \in [50,100)$.

\hide{
\begin{figure}[t]
\centering
\includegraphics[width=2.6in]{figures/RMSE_ad}
\caption{Our proposed DGBR algorithm makes the most stable prediction on whether user will like or dislike an advertisement across environments.}
\label{fig:prediction_vary_nf}
\end{figure}
}

\begin{figure}[t]
\centering
\includegraphics[width=2.4in]{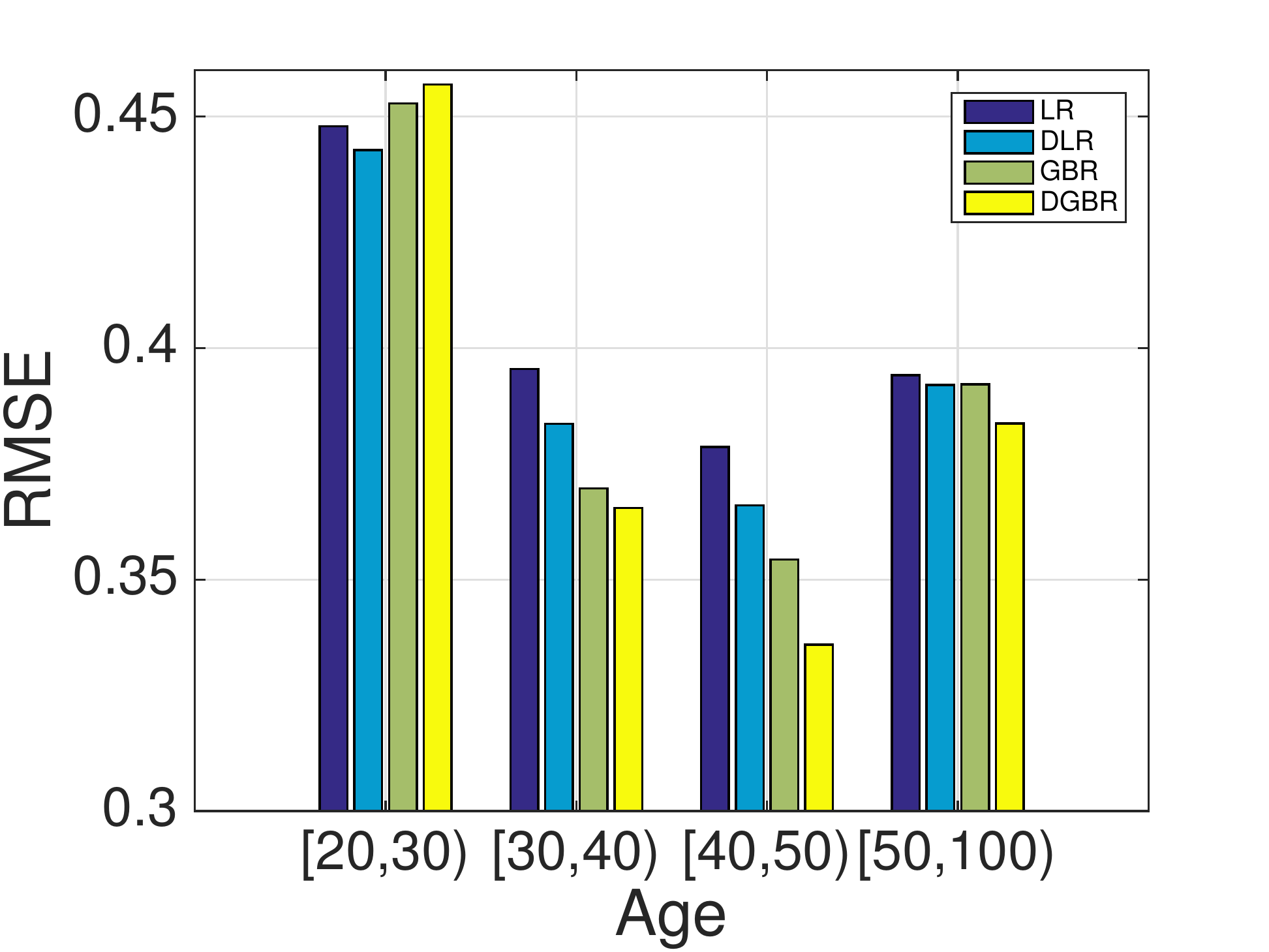}
\vspace{-0.1in}
\caption{Prediction across environments separated by age. The models are trained on dataset where uses' $Age\in [20,30)$, but tested on various datasets with different users' age range. }
\label{fig:prediction_via_separation}
\end{figure}

\par \textbf{Results on Setting 1.} 
Based on the first experimental setting, we plot the results in Figure~\ref{fig:prediction_vary_nf} and Figure~\ref{fig:prediction_vary_one_nf}.
Under this setting, we trained all algorithms on a dataset with bias rate $r=0.6$ for four noisy features. 
Then we test the performance of our proposed algorithm and baselines on various test data with different bias rate on these four noisy features, and report the $RMSE$ in Fig. \ref{fig:RMSE_ad}. 
To explicitly demonstrate the advantage of our proposed algorithm, we plot the $Average\_Error$ and $Stability\_Error$ as defined in Eq. (\ref{metrics:acc}) and (\ref{metrics:stb}) in Fig. \ref{fig:Accuracy_Stability}.
We further generate additional test data by varying bias rate $r$ on other features, with results in Fig.~\ref{fig:prediction_vary_one_nf}. 
Fig.~\ref{fig:RMSE_us_14} and \ref{fig:RMSE_us_19} show that DGBR makes the most stable prediction across test data.
Overall, the results and their interpretation are very similar to the simulation experiments. 

\begin{figure}[t]
\centering
\includegraphics[width=2.4in]{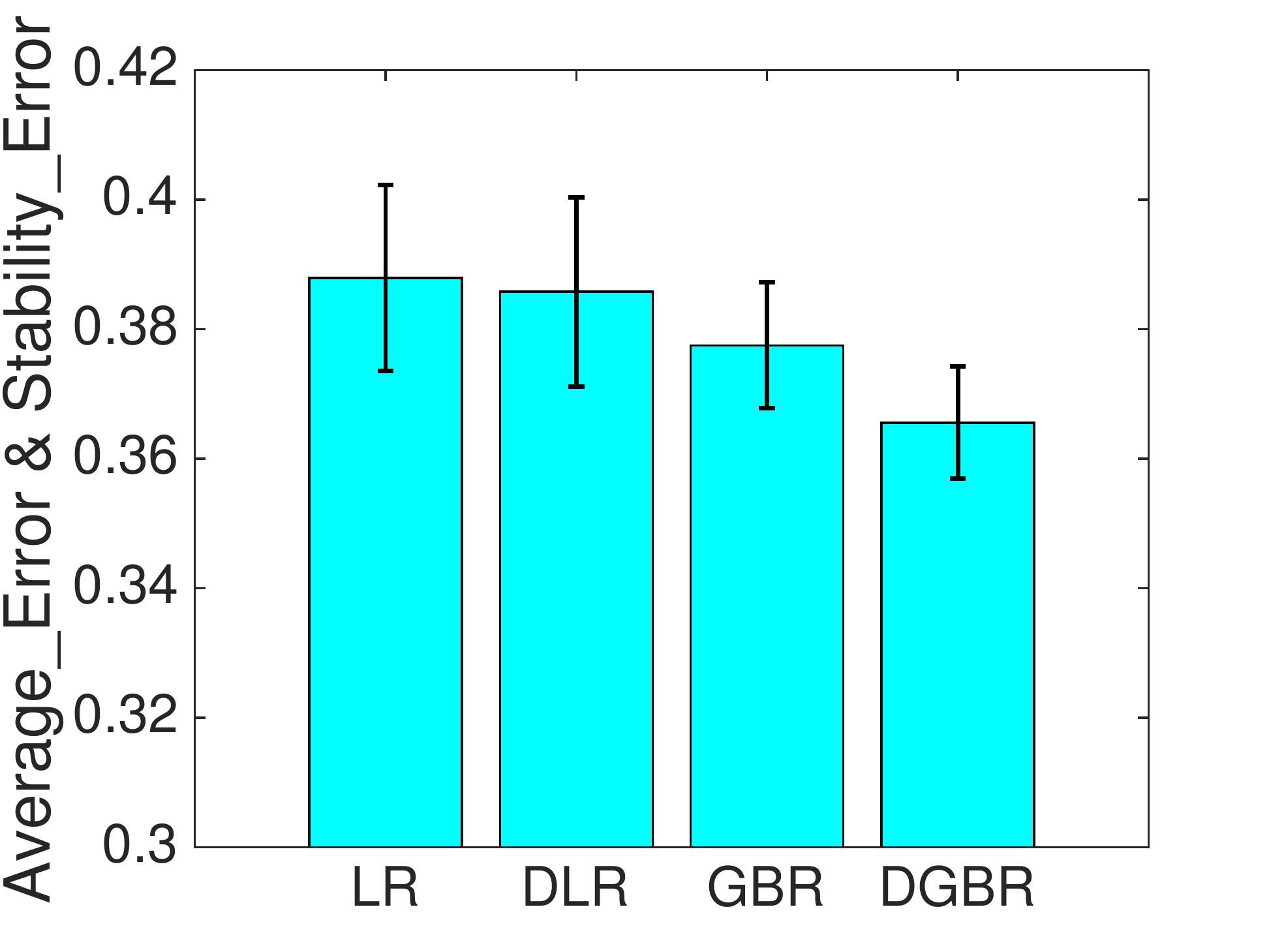}
\caption{$Average\_Error$ and $Stability\_Error$ of all algorithms across environments after fixing $P(Y)$ as the same with its value on global dataset.}
\label{fig:stability_via_separation}
\end{figure}

\par \textbf{Results on Setting 2.} 
Based on the second experimental setting, we plot the results in Figure~\ref{fig:prediction_via_separation}, where we separate the dataset into four environments by users' age, including $Age \in [20,30)$, $Age \in [30,40)$, $Age \in [40,50)$ and $Age \in [50,100)$. 
We trained all algorithms on dataset where users' $Age \in [20,30)$, then we test them on all the four environments.
From Figure \ref{fig:prediction_via_separation}, we could find that our DGBR algorithm achieves comparable result to the baselines on test data with users' $Age \in[20,30)$, where the distributions of variables are similar with the one on the training data.
On such test data, the spurious correlation between noisy features and outcome could help baselines make a more precise prediction.
While on the other three parts of test dataset, whose distributions are different with training dataset, our DGBR algorithm obtains the best prediction performance. 
The main reason is that our algorithm can reduce or even remove the spurious effect of noisy features on outcome and find out the stable features for stable prediction.

We can infer that the stability of DGBR algorithm is not as good as baselines in Fig.~\ref{fig:prediction_via_separation}; this occurs because the distribution of outcome $P(Y)$ varied across these four environments.  After we fixed $P(Y)$  by data sampling on the outcome with $P(Y=1) = \frac{14,891}{14,891+93,108}$ on the global dataset, we report the $Average\_Error$ and $Stability\_Error$ of all algorithms across four environments in Figure \ref{fig:stability_via_separation}.
And we find that as the $P(Y)$ is stable, DGBR outperforms baselines.

\section{Conclusion}
In this paper, we focus on how to make a stable prediction across unknown environments, where the data distribution of unknown environments might be very different with the distribution of training data.
We argued that most previous methods for addressing stable prediction are deficient because either they need the distribution of test data as prior knowledge or rely on diversity of training datasets from different environments.
Therefore, we propose a Deep Global Balancing Regression algorithm for stable prediction across environments by jointly optimizing the deep auto-encoder model and global balancing model.
The global balancing model can identify the causal relationship between predictor variables and response variable, while the deep auto-encoder model is designed for capturing the non-linear structure among variables and making global balancing easier and less noisy.
We prove that our algorithm can make a stable prediction from both theoretical analysis and empirical experiments.
The experimental results on both synthetic and real world datasets show that our DGBR algorithm outperforms the baselines for stable prediction across unknown environments.

\hide{
In this paper, we did not do sensitivity analysis on how the validation dataset are constructed. We leave a further exploration of this issue for future work.
}


\acks{This work is supported by the National Program on Key Basic
Research Project (No. 2015CB352300), and the National
Natural Science Foundation of China (No. 61772304, No.
61521002, No. 61531006, and No. U1611461). Thanks for
the research fund of Tsinghua-Tencent Joint Laboratory for
Internet Innovation Technology, and the Young Elite Scientist
Sponsorship Program by CAST.
Ruoxuan Xiong's research was supported by Charles and Katharine Lin Graduate Fellowship.
Bo Li's research was supported by the
Tsinghua University Initiative Scientific Research Grant, No.
20165080091; National Natural Science Foundation of China,
No. 71490723 and No. 71432004; Science Foundation of
Ministry of Education of China, No. 16JJD630006. 
Susan Athey's research was supported by the Office of Naval Research under grant N00014-17-1-2131 and the Sloan foundation. }



\appendix

\section{Proof of Lemma \ref{pro:population_overlap}}
\label{Appendix:A}

\begin{proof}
Assume treatment variable is $T=\mathbf{X}_{i,j}$ and $\mathbf{X}_{i,-j}$ are covariates. From the propensity score is bounded away from zero and one, and
$\exists (x_1^0, \cdots, x_{j-1}^0, x_{j+1}^0, \cdots, x_{p}^0)$, $P(\mathbf{X}_{i,-j} = (x_1^0, \cdots, x_{j-1}^0, x_{j+1}^0, \cdots, x_{p}^0)) > 0 $, from 

\begin{eqnarray*}
&& P(\mathbf{X}_i = (x_1^0, \cdots, x_{j-1}^0, x_j, x_{j+1}^0, \cdots, x_{p}^0))  \\
&=& P(\mathbf{X}_{i,-j} =(x_1^0, \cdots, x_{j-1}^0, x_{j+1}^0, \cdots, x_{p}^0)) \cdot \\
&&P(\mathbf{X}_{i,j} = x_j| \mathbf{X}_{i,-j} = (x_1^0, \cdots, x_{j-1}^0, x_{j+1}^0, \cdots, x_{p}^0))
\end{eqnarray*}
we have 
\begin{eqnarray}
0 < P(\mathbf{X}_i = (x_1^0, \cdots, x_{j-1}^0, x_j, x_{j+1}^0, \cdots, x_{p}^0) < 1 \label{tmp1} 
\end{eqnarray}
for $x_j = 0$ or $x_j = 1$.

Next is to proof $\forall x$ ($x$ is binary), $$0 < P(\mathbf{X}_i=x) < 1$$ from inequality (\ref{tmp1}). Let $k \neq j$, from
\begin{eqnarray*}
&& P(\mathbf{X}_i = (x_1^0, \cdots, x_{j-1}^0, x_j, x_{j+1}^0, \cdots, x_{p}^0))  \\
&=& P(\mathbf{X}_{i,-k} = (x_1^0, \cdots, x_{k-1}^0, x_{k+1}^0, \cdots, x_{p}^0)) \cdot \\
&&P(\mathbf{X}_{i,k} = x_k^0| X_{i,-k} = (x_1^0, \cdots, x_{k-1}^0, x_{k+1}^0, \cdots, x_{p}^0))
\end{eqnarray*}
and $0< P(\mathbf{X}_i = (x_1^0, \cdots, x_{j-1}^0, x_j, x_{j+1}^0, \cdots, x_{p}^0)) < 1$, we have $$P(\mathbf{X}_{i,-k} = (x_1^0, \cdots, x_{k-1}^0, x_{k+1}^0, \cdots, x_{p}^0))  > 0$$ Furthermore, $\mathbf{X}_{i,k}$ can also be viewed as the treatment variable, so $$0 < P(\mathbf{X}_{i,k} = x_k^0| \mathbf{X}_{i,-k} = (x_1^0, \cdots, x_{k-1}^0, x_{k+1}^0, \cdots, x_{p}^0)) < 1$$, and therefore, 
$$\Scale[0.9]{0 < P(\mathbf{X}_{i,k} = 1 -x_k^0| \mathbf{X}_{i,-k} = (x_1^0, \cdots, x_{k-1}^0, x_{k+1}^0, \cdots, x_{p}^0)) < 1}$$
We have  (without loss of generality, we assume $k < j$), $\forall x_k, x_j$
$$
\Scale[0.85]{0 < P(\mathbf{X}_{i} = (x_1^0, \cdots, x_{k-1}^0, x_k, x_{k+1}^0, \cdots,  x_{j-1}^0, x_j, x_{j+1}^0, \cdots, x_{p}^0) < 1}.
$$
We repeat the above for all other variables one by one, we have $\forall x$, 
$$
0 < P(\mathbf{X}_{i} = x) < 1
$$
\end{proof}

\section{Proof of Proposition \ref{prop:svindep}}
\label{Appendix:B}
\begin{proof}
If $0 < \hat{P}(\mathbf{X}_i = x) < 1$,  from Theorem \ref{thm1}, $W^*_i = \frac{1}{\hat{P}(\mathbf{X}_i = x)}$ satisfies equation (\ref{eq:L_balancing}) equals 0. Next is to show all variables in $\mathbf{X}$ are independent after balancing by this $W^*$. Note that 
\begin{eqnarray*}
\Scale[1.0]\sum_{i=1}^n W^*_i &=& n  \Scale[1.0]\sum_{x} \frac{1}{n}  \Scale[1.0]\sum_{i: X_{i}=x} W^*_i \\
&=&  \Scale[1.0] n \sum_{x} \hat{P}(\mathbf{X}_i= x) \cdot \frac{1}{\hat{P}(\mathbf{X}_i= x)}= n \cdot 2^{p}
\end{eqnarray*}
Similarly, $\sum_{i: \mathbf{X}_{i,j}=1} W^*_i  =  n \cdot 2^{p-1}$ and $\sum_{i: \mathbf{X}_{i,j}=0} W^*_i = n \cdot 2^{p-1}$. Denote the probability mass function of $\mathbf{X}$ weighted by $W^*$ as $\tilde P$. Thus, for $x = (x_1, \cdots, x_p)$,
$$
\tilde P(\mathbf{X}_{i} = (x_1, \cdots, x_p)) = \frac{\sum_{i: \mathbf{X}_{i,j}=x} W^*_i}{\sum_{i} W^*_i} = \frac{1}{2^p}
$$
and $\forall j$, $\tilde P(\mathbf{X}_{i,j} = x_j) = \frac{\sum_{i: \mathbf{X}_{i,j}=j} W^*_i}{\sum_{i} W^*_i}=\frac{1}{2}$, so we have 
$$
\tilde P(\mathbf{X}_{i} = (x_1, \cdots, x_p)) = \tilde P(\mathbf{X}_{i,1} = x_1) \cdots \tilde P(\mathbf{X}_{i,p} = x_p),
$$
which implies that covariates in $\mathbf{X}$ are independent after balanced by $W^*$. 

\hide{Let $\tilde{\mathbf{X}}$ be an ``extended'' matrix of $\mathbf{X} \in R^{n \times p}$ where each row $\mathbf{X}_i$ is duplicated $W^*_i = \frac{1}{\hat{P}(\mathbf{X}_i = x)}$ times. \footnote{$W^*_i$ does not need to be an integer.} Denote the number of rows in $\tilde{\mathbf{X}}$ to be $\tilde{n}$. When $0 < \hat{P}(\mathbf{X}_i = x) < 1$, 
\begin{eqnarray*}
\Scale[1.0]\sum_{i} W^*_i &=& \tilde{n}  \Scale[1.0]\sum_{x} \frac{1}{\tilde{n}}  \Scale[1.0]\sum_{i: \tilde{X}_{i}=x} W^*_i \\
&=&  \tilde{n} \Scale[1.0]\sum_{x} \hat{P}(\tilde{\mathbf{X}}_i= x) \cdot \frac{1}{\hat{P}(\tilde{\mathbf{X}}_i= x)}= \tilde{n} \cdot 2^{p}
\end{eqnarray*}
Similarly, $\sum_{i: \tilde{\mathbf{X}}_{i,j}=1} W^*_i  =  \tilde{n} \cdot 2^{p-1}$, $\sum_{i: \tilde{\mathbf{X}}_{i,j}=0} W^*_i = \tilde{n} \cdot 2^{p-1}$, and 
$\sum_{i: \tilde{\mathbf{X}}_{i,j}=x} W^*_i=\tilde{n}$. 
}

\end{proof}

\section{Proof of Lemma \ref{lemma:alpha}}
\label{Appendix:C}
\begin{proof}
$\forall k, j$, $k \neq j$, it has $0 \leq \frac{\sum_{i: X_{i,k}=1, X_{i,j}=1} \hat{W}_i }{\sum_{i: X_{i,j}=1} \hat{W}_i } \leq 1$ and $0 \leq \frac{\sum_{i: X_{i,k}=1, X_{i,j}=0} \hat{W}_i }{\sum_{i: X_{i,j}=0} \hat{W}_i } \leq 1$. Thus, $0 \leq \alpha \leq 1$, $\forall m$. Assume for $k, j$ and $k \neq j$, $\sum_{x: x_k = 1, x_j = 1} \mathbbm{1} (\sum_{i=1}^n \mathbbm{1}(X_i = x) = 0) = m_1$, $\sum_{x: x_j = 1} \mathbbm{1} (\sum_{i=1}^n \mathbbm{1}(X_i = x) = 0) = m_2$, $\sum_{x: x_k = 1, x_j = 0} \mathbbm{1} (\sum_{i=1}^n \mathbbm{1}(X_i = x) = 0) = m_3$ and $\sum_{x: x_j = 0} \mathbbm{1} (\sum_{i=1}^n \mathbbm{1}(X_i = x) = 0) = m_4$.

\begin{enumerate}
\item If $m = 0$, $\alpha = 0$ is a direct result from Theorem \ref{thm1}
\item If $0 < m \leq 2^{p-2}$, without loss of generality, assume $m_2 \geq m_4$,  \label{lemma1part2}
\begin{eqnarray*}
\alpha_{jk} =& \left| \frac{\sum_{i: X_{i,k}=1, X_{i,j}=1} \hat{W}_i }{\sum_{i: X_{i,j}=1} \hat{W}_i }  - \frac{\sum_{i: X_{i,k}=1, X_{i,j}=0} \hat{W}_i }{\sum_{i: X_{i,j}=0} \hat{W}_i } \right| \\
=& \left| \frac{2^{p-2} - m_1}{2^{p-1} - m_2} - \frac{2^{p-2} - m_3}{2^{p-1} - m_4} \right| \\
\leq&  \frac{2^{p-2}}{2^{p-1} - m_2} - \frac{2^{p-2} - m_4}{2^{p-1} - m_4}  
\end{eqnarray*}
Given $m_4 = m - m_2$, 
\begin{eqnarray*}
\frac{\partial \alpha_{jk}}{\partial m_2} = 2^{p-2} \left( \frac{1}{(2^{p-1} - m_2)^2} - \frac{1}{(2^{p-1} - m + m_2)^2} \right)
\end{eqnarray*}
which is positive when $m_2 \leq m/2$ (we assume $m_2 \geq m_4$), and therefore 
\begin{eqnarray*}
\alpha_{jk} \leq \frac{2^{p-2}}{2^{p-1} - m} - \frac{1}{2}
\end{eqnarray*}
\item If $2^{p-2} < m < 2^{p-1}$, without loss of generality, assume $m_2 \geq m_4$, when $m_2 \leq 2^{p-2}$, from \ref{lemma1part2}, we have 
\begin{eqnarray*}
\left| \frac{2^{p-2} - m_1}{2^{p-1} - m_2} - \frac{2^{p-2} - m_3}{2^{p-1} - m_4} \right| \leq&  \frac{2^{p-2}}{2^{p-2} - m_2} - \frac{2^{p-2} - m_4}{2^{p-1} - m_4}  \\
\leq& 1 - \frac{2^{p-2} - m + 2^{p-2}}{2^{p-1} - m + 2^{p-2}}
\end{eqnarray*}
when $m_2 > 2^{p-1}$, 
\begin{eqnarray*}
\left| \frac{2^{p-2} - m_1}{2^{p-1} - m_2} - \frac{2^{p-2} - m_3}{2^{p-1} - m_4} \right| \leq&  1 - \frac{2^{p-2} - m_4}{2^{p-1} - m_4} \\
 <& 1 - \frac{2^{p-2} - m + 2^{p-2}}{2^{p-1} - m + 2^{p-2}}
\end{eqnarray*}
because $\frac{2^{p-2} - m_3}{2^{p-2} - m_4}$ is decreasing in $m_4$. Thus 
\begin{eqnarray*}
\alpha \leq 1 - \frac{2^{p-2} - m + 2^{p-2}}{2^{p-1} - m + 2^{p-2}} = 1 - \frac{2^{p-1} - m}{3 \times 2^{p-2} - m}
\end{eqnarray*}
\item If $2^{p-1} \leq m$, let $m_1 = \floor*{\frac{m}{2}} - 2^{p-2}$, $m_2 = \floor*{\frac{m}{2}}  $, $m_3 = 2^{p-2}$, $m_4 = \ceil{\frac{m}{2}}$, which satisfy $m_2 + m_4 = 1$, $m_1 \leq m_2$, and $m_3 \leq m_4$. Moreover,

$$
\left| \frac{2^{p-2} - m_1}{2^{p-1} - m_2} - \frac{2^{p-2} - m_3}{2^{p-1} - m_4} \right| = 1
$$
together with $\alpha \leq 1$, we have $\alpha = 1$

\end{enumerate}
\end{proof}

\section{Proof of Theorem \ref{thm:mean_alpha}}
\label{Appendix:D}
\begin{proof}
The probability that $m$ different values in $\mathcal{X}$ do not appear in $\mathbf{X}$ equals the ratio of the number of solutions to 
\begin{eqnarray} \label{boxprob}
y_1 + y_2 + \cdots y_{2^p} = n,
\end{eqnarray}
where $m$ different $i$s have $y_i=0$, to the total number of solution to Eq. (\ref{boxprob}) without any constraint. The denominator is  ${{n + 2^p - 1} \choose {2^p - 1}}$. The numerator is the number of methods to select $m$ different $i$s, such that $y_i = 0$ multiplied by the number of solutions to $y_1 + y_2 + \cdots y_{2^p-m} = n$ without any constraint, which is ${2^p \choose m} {{n-1} \choose {2^p-1-m}} $. Thus the probability that $m$ different $x$s do not appear in $\mathbf{X}$  is
\begin{eqnarray*}
\frac{1}{{{n + 2^p - 1} \choose {2^p - 1}}} {2^p \choose m} {{n-1} \choose {2^p-1-m}} 
\end{eqnarray*}
With lemma \ref{lemma:alpha}, 
$$\Scale[1.0]{E\left[ \alpha \right] = \frac{1}{{{n + 2^p - 1} \choose {2^p - 1}}} \left\lbrace \sum_{m=0}^{2^{p} - 1} {2^p \choose m} {{n-1} \choose {2^p-1-m}} g(p,m)  \right\rbrace}, $$
where $g(p, m)$ is defined in (\ref{eqn:def_g_p_m}).
\end{proof}

\section{Proof of Theorem \ref{thm:upper_bound_risk}}
\label{Appendix:E}
\begin{proof}
Define $L_{\tilde P}(f) = E_{\tilde P}(l(f(\mathbf{X}), Y))$, where the probability mass function  $\tilde P(\mathbf{X}_i, Y_i) = \tilde P(\mathbf{X}_i) P(Y_i | \mathbf{X}_i)$ has $ P(Y_i = y|\mathbf{X}_i = x) = P(Y_i = y|\mathbf{S}_i=s, \mathbf{V}_i=v) =P(y|s)$ to be the same as that in Assumption \ref{asmp:stable} and  $\tilde P(\mathbf{X}_i = x) = \tilde p_x$, where $\tilde p_x$ is defined in Eq. (\ref{eqn:epsilon_x-def}) and equals $\frac{1}{\tilde{n}} \sum_{i=1}^n W_i^* \mathbbm{1}(\mathbf{X}_i = x)$. 

Let $\tilde{f}^* = \arg \min_{f} L_{\tilde P}(f)$. For all $f$, 
 \begin{eqnarray}
\Scale[0.9]{ |L_{\tilde P}(f) - L_{P}(f)| \leq \max_{x} \mathbb{E}[l(f(x), y)|x] \sum_{x} |\epsilon_x|,}\label{eq:risk-inequal}
 \end{eqnarray}
followed by
\begin{eqnarray*}
 \!\!\! \!\!\!\!\!\!\!\! &&\Scale[0.85]{ |L_{\tilde P}(f) - L_{P}(f)| = |\sum_{x} \tilde{p}_x \mathbb{E}[l(f^*(x), y)|x] - \sum_{x} p_x \mathbb{E}[l(f(x), y)|x]|} \\
 \!\!\! \!\!\!\!\!\!\!\! && \Scale[0.85]{= |\sum_{x} \tilde{p}_x \mathbb{E}[l(f(x), y)|x] - \sum_{x} (\tilde{p}_x - \epsilon_x) \mathbb{E}[l(f(x), y)|x]|}  \\
 \!\!\! \!\!\!\!\!\!\!\! && \Scale[0.85]{= |\sum_{x} \epsilon_x \mathbb{E}[l(f(x), y)|x]|  \leq \max_{x} \mathbb{E}[l(f(x), y)|x] \sum_{x} |\epsilon_x|}.\\
 \end{eqnarray*}
 $ \max_{x} \mathbb{E}[l(\hat{f}(x), y)|x]$ is bounded because $x$ are $y$  are binary and all weights in $f \in \mathcal{F}$ are bounded, where $ \mathcal{F}$ is the model class defined by the constraints in FWDGBR algorithm. From Eq. (\ref{eq:risk-inequal}), we have 
 \begin{eqnarray}
  \Scale[0.9]{L_{P}(\hat{f}) \leq L_{\tilde P}(\hat{f}) + \max_{x} E[l(\hat{f}(x), y)|x] \sum_{x} |\epsilon_x|}.\label{eqn:risk-tmp1}
 \end{eqnarray}

Next is to upper bound the difference between $L_{\tilde P}(\hat{f}) $ and $ L_{\tilde P}(f^*)$. Let $\mathcal{A} = \{x \mapsto l(f(x),y): f \in \mathcal{F} \}$ to be the loss class, where $l(\cdot)$ is the cross-entropy loss and $y$ is binary. From Lemma 3 in \cite{wan2013regularization}, the generalized bound of a $2$-class classifier with logistic cross-entropy loss function is related empirical Rademacher complexity, with probability at least $1-\delta$,
\begin{eqnarray}\label{eqn:loss-complexity}
\Scale[0.85]{L_{\tilde P}(\hat{f}) \leq L_{\tilde P}(f^*) + 4R_n(\mathcal{A}) + 3 \sqrt{\frac{log(2/\delta)}{2n}}}.
\end{eqnarray}

Note that the auto-encoder has $K$ layers to construct $\phi(\mathbf{X}_i)$ and another $K$ layers to reconstruct $\mathbf{X}_i$ from $\phi(\mathbf{X}_i)$. $Y_i$ is predicted by a logistic regression model on $\phi(\mathbf{X}_i)$. The Rademacher complexity depends on weight constraints  $\|\beta\|_2^2 \leq \lambda_4$, $\|\beta\|_1 \leq \lambda_5$ and $\sum_{k=1}^K (\|A^{(k)}\|_F^2+\|\hat{A}^{(k)}\|_F^2) \leq \lambda_7$ in the FWDGBR algorithm. The decoder from $\phi(\mathbf{X}_i)$ to $\mathbf{X}_i$ is not used to predict $Y_i$, so the decoder does not affect the complexity $R_n(\mathcal{A})$. 

Our goal is to give an upper bound on $R_n(\mathcal{A})$. Constraint $\sum_{k=1}^K (\|A^{(k)}\|_F^2+\|\hat{A}^{(k)}\|_F^2) \leq \lambda_7$ implies that $\sum_{k=1}^K \|A^{(k)}\|_F^2 \leq \lambda_7$, and together with $\|b^{(k)}\|_2 \leq M^{(k)}$, implies that $\|A^{(k)}_j\|_2 \leq \sqrt{\lambda_7+(M^{(k)})^2}$. Let $B_k = \sqrt{\lambda_7+(M^{(k)})^2}$. Constraints $\sum_{k=1}^K (\|A^{(k)}\|_F^2+\|\hat{A}^{(k)}\|_F^2) \leq \lambda_7$ and $\|b^{(k)}\|_2 \leq M^{(k)}$ imply $\|[A^{(k)}_j, b^{k}_j]\|_2 \leq B_k$.

We can employ Theorem 3.1 in  \cite{zhai2018adaptive} to obtain the empirical Rademacher complexity $R_n(\mathcal{A})$. Since $\mathbf{X}$ is binary, $\|\mathbf{X}\|_{\max}  = \max_{i,j} |X_{i,j}|= 1$. Theorem 3.1 in  \cite{zhai2018adaptive} does not have bias term in each layer. We can add constant $1$ to neurons in $k$-th layer $\phi(\mathbf{X}_i)^{(k)}$ to fit in the framework of Theorem 3.1 in \cite{zhai2018adaptive}. Thus, the dimension of the $k$-th layer is $l_k$ for $k = 0, 1, 2, \cdots, K$.  The retain vector is $\theta^k = [1]^{l_k}$ in our case (corresponding to the dropout rate in each layer is 0). If $\|\beta\|_2^2 \leq \lambda_4$ is tighter than $\|\beta\|_1 \leq \lambda_5$, that is, $1/p=1/2$ and $1/q = 1/2$ for all layers, we have 
\[ \Scale[0.85]{R_n(\mathcal{A}) \leq  2^{K+1} \sqrt{\frac{2log(2p)}{n}} \sqrt{\lambda_4 l_K} \prod_{k=1}^K B_k (l_{k-1})^{1/2}}\]
On the other hand, if $\|\beta\|_1 \leq \lambda_5$ is tighter than $\|\beta\|_2^2 \leq \lambda_4$, that is $1/p=1$ and $1/q=0$ for the $K$-th layer, so $l_{K}^{1/q}=1$ and
\[ \Scale[0.85]{R_n(\mathcal{A}) \leq  2^{K+1} \sqrt{\frac{2log(2p)}{n}} \lambda_5 \prod_{k=1}^K B_k (l_{k-1})^{1/2} }\]
We combine these two cases and have 
\begin{eqnarray}\label{eqn:rademacher}
 \Scale[0.85]{R_n(\mathcal{A}) \leq  2^{K+1} \sqrt{\frac{2log(2p)}{n}} \min( \sqrt{\lambda_4 l_K},\lambda_5)\prod_{k=1}^K B_k (l_{k-1})^{1/2}}
\end{eqnarray}
Plug Inequality (\ref{eqn:rademacher}) into Inequality (\ref{eqn:loss-complexity}), we have 
\begin{eqnarray}\label{eqn:risk-tmp2}
\nonumber &&\!\!\! \!\!\! \!\!\! \!\!\!\!\!\!\!\! \Scale[0.85]{ L_{\tilde P}(\hat{f}) \leq L_{\tilde P}(f^*)} \\
\nonumber  \!\!\! \!\!\!\!\!\!\!\! && \Scale[0.85]{ + 2^{K+3} \sqrt{\frac{2log(2p)}{n}} \min( \sqrt{\lambda_4 l_K},\lambda_5)\prod_{k=1}^K B_k (l_{k-1})^{1/2}} \\
\!\!\! \!\!\!\!\!\!\!\! && \Scale[0.85]{+ 3 \sqrt{\frac{log(2/\delta)}{2n}}.}
\end{eqnarray}
 \hide{
 Using Chernoff bound and results in Empirical Risk Minimization (using Chernoff Bound), since $|\mathcal{F}|$ is finite and $l(f(x), y) \in [a,b]$, with probability $\geq 1 - \delta$, for all $f \in \mathcal{F}$
$$
\Scale[0.9]{|\hat{L}(f) - L_{\tilde P}(f)| \leq (b-a) \sqrt{\frac{log2|\mathcal{L}|+log(1/\delta)}{2n}}}.
$$
Thus, 
 \begin{eqnarray}
 \Scale[0.9]{ L_{\tilde P}(\hat{f})  \leq L_{\tilde P}(\tilde{f}^*)  + 2(b-a) \sqrt{\frac{log2|\mathcal{F}|+log(1/\delta)}{2n}}}.\label{eqn:risk-tmp2}
 \end{eqnarray}
 }
 The last step is to bound the difference between $L_{\tilde P}(\tilde{f}^*) $ and $ L_{P}(f^*)$. When $L_{\tilde P}(\tilde{f}^*) \geq L_{P}(f^*)$, 
\begin{eqnarray}
\nonumber \!\!\! \!\!\!\!\!\!\!\! &&\Scale[0.85]{L_{\tilde P}(\tilde{f}^*) - L_{P}(f^*) = \sum_{x} \tilde{p}_x \mathbb{E}[l(\tilde{f}^*(x), y)|x] - \sum_{x} p_x \mathbb{E}[l(f^*(x), y)|x]} \\
\nonumber \!\!\! \!\!\!\!\!\!\!\! &&\Scale[0.8]{= \sum_{x} \tilde{p}_x \mathbb{E}[l(\tilde{f}^*(x), y)|x] - \sum_{x} \tilde{p}_x \mathbb{E}[l(f^*(x), y)|x] + \sum_{x} \epsilon_x \mathbb{E}[l(f^*(x), y)|x]} \\ 
   \!\!\! \!\!\!\!\!\!\!\! &&\Scale[0.85]{\leq \sum_{x} \epsilon_x \mathbb{E}[l(f^*(x), y)|x]} \label{ineq:risk} \\ 
\nonumber \!\!\! \!\!\!\!\!\!\!\! &&\Scale[0.85]{\leq \max_{x} \mathbb{E}[l(f^*(x), y)|x] \sum_{x: \epsilon_x > 0} \epsilon_x }
\end{eqnarray}
Eq. (\ref{ineq:risk}) holds followed by $\tilde{f}^*(x)$ minimizes $L_{\tilde P}(\tilde{f})$, so $L_{\tilde P}(\tilde{f}^*) \leq L_{\tilde P}(f^*) $, and then $\sum_{x} \tilde{p}_x \mathbb{E}[l(\tilde{f}^*(x), y)|x] \leq \sum_{x} \tilde{p}_x E[l(f^*(x), y)|x]$. Thus, 
 \begin{eqnarray}
 \Scale[0.9]{L_{\tilde P}(\tilde{f}^*)  \leq L_{P}(f^*) + \max_{x} \mathbb{E}[l(f^*(x), y)|x] \sum_{x: \epsilon_x > 0} \epsilon_x}\label{eqn:risk-tmp3}
 \end{eqnarray}
always holds.
From Eq. (\ref{eqn:risk-tmp1}), (\ref{eqn:risk-tmp2}), (\ref{eqn:risk-tmp3}), we have
\begin{eqnarray*}
 \!\!\! \!\!\!\!\!\!\!\! &&\Scale[0.85]{ L_{P}(\hat{f}) \leq L_{P}(f^*) + \max_{x} \mathbb{E}[l(\hat{f}(x), y)|x] \sum_{x} |\epsilon_x|  } \\
 \!\!\! \!\!\!\!\!\!\!\! && \Scale[0.85]{+ 2^{K+3} \sqrt{\frac{2log(2p)}{n}} \min( \sqrt{\lambda_4 l_K},\lambda_5)\prod_{k=1}^K B_k (l_{k-1})^{1/2} }  \\
  \!\!\! \!\!\!\!\!\!\!\! && \Scale[0.85]{+  3 \sqrt{\frac{log(2/\delta)}{2n}} + \max_{x} \mathbb{E}[l(f^*(x), y)|x] \sum_{x: \epsilon_x > 0}  \epsilon_x }  \\
 \!\!\! \!\!\!\!\!\!\!\! &\leq&\Scale[0.85]{ L_{P}(f^*)+ 2^{K+3} \sqrt{\frac{2log(2p)}{n}} \min( \sqrt{\lambda_4 l_K},\lambda_5)\prod_{k=1}^K B_k (l_{k-1})^{1/2}  },\\
  \!\!\! \!\!\!\!\!\!\!\! && \Scale[0.85]{+ 3 \sqrt{\frac{log(2/\delta)}{2n}} + 2 \max_{x, f} \mathbb{E}[l(f(x), y)|x] \sum_{x} |\epsilon_x|  }  
\end{eqnarray*}
with probability $\geq 1 - \delta$.
\end{proof}

\bibliography{sample}

\end{document}